\documentclass[a4paper]{amsart}

\usepackage{array}
\usepackage{longtable}
\usepackage{booktabs}
\theoremstyle{plain}
\newtheorem{thm}{Theorem}[section]
\newtheorem{lem}{Lemma}
\newtheorem{cor}{Corollary}
\newtheorem{prop}{Proposition}
\newtheorem{con}{Conjecture}

\theoremstyle{definition}
\newtheorem{defn}{Definition}
\newtheorem{assum}{Assumption}

\theoremstyle{remark}
\newtheorem{rem}{Remark}

\usepackage[colorlinks=true,
            linkcolor=blue,
            citecolor=red,
            urlcolor=magenta]{hyperref}
            
\usepackage[msc-links]{amsrefs}   

\usepackage[nameinlink,noabbrev]{cleveref}  

\crefname{equation}{Eq.}{Eqs.}
\crefname{section}{Sec.}{Sections}
\crefname{subsection}{Sec.}{Sections}
\crefname{table}{Tab.}{Tables}
\crefname{figure}{Fig.}{Figures}
\crefname{thm}{Thm.}{Theorems}
\crefname{lem}{Lem.}{Lemmas}
\crefname{cor}{Cor.}{Corollaries}
\crefname{prop}{Prop.}{Propositions}
\crefname{con}{Conj.}{Conjectures}
\crefname{defn}{Def.}{Definitions}
\crefname{assum}{Assumption}{Assumptions}
\crefname{rem}{Rem.}{Remarks}
\crefname{ex}{Exa.}{Examples}
\crefname{enumi}{item}{items}

\title{Hierarchical Grading in Large Language Models}
\author{T. Shaska}
\address{Department of Computer Science and Engineering, \\Oakland University, Rochester, MI 48309}
\email{shaska@oakland.edu}

\subjclass[2020]{68T07, 68T50, 68Q32, 62C20, 16W50}

\keywords{graded vector spaces, graded transformers, large language models, hierarchical attention, sample complexity,   minimax lower bounds, inductive bias}

\begin{document}

\maketitle

\begin{abstract}
We introduce Graded Large Language Models (GLLMs), an algebraic framework that equips the representation space of a transformer with a grading and propagates the induced weighted scalar action through embeddings, self-attention, and the training objective. The construction extends the theory of graded neural networks and graded transformers to autoregressive language models while preserving expressive power, asymptotic computational complexity, and inference cost.

The governing geometric picture is that of geometric invariant theory. The benefit of a grading is expressed by a Kempf--Ness functional on the grading torus; the grades that improve upon the uniform architecture form an open convex cone whose membership is decided by a Hilbert--Mumford-type criterion pairing a grade direction against two measurable profiles of the target and the data; the optimal grades are the coincidence point of two moment maps, given in closed form; and the ordinary transformer appears as a semistable isotropic point on the boundary of the cone: one member of a larger graded family rather than a distinguished optimum.

Separately, for level-stratified targets we prove a minimax separation between the graded prior and its absence: over all estimators the risks of the graded and uniform target classes separate throughout an explicit window of sample sizes, by a factor that decays exponentially in the number of levels under geometric stratification. Both profiles are estimable offline, so the optimal grades solve a convex program certified before training begins. Because the grading is absorbed into the learned parameters after training, every GLLM compiles to a standard transformer of identical architecture and inference complexity.

We conclude with the grade-selection procedure for domains without canonical gradings, instantiated for natural language, and a pre-registered validation program whose empirical results will be reported in a companion manuscript.
\end{abstract}

\setcounter{tocdepth}{2}


\section{Introduction}
\label{sec-1}

Large Language Models (LLMs) have moved from statistical $n$-gram estimators to transformer systems of $O(10^{11})$ parameters trained on $O(10^{13})$ tokens, and they now generate relatively fluent text, perform multi-step tasks, and support code synthesis and scientific work; see \cite{radford2019language} among many other sources. This progress came almost entirely from scaling; see \cite{kaplan2020scaling}. The architecture behind it is uniform: every token occupies the same space, every coordinate of that space carries the same weight, and every attention head computes over a flat, isotropic geometry in which no direction is special.

Language, in contrast, is hierarchical. Subword units compose into words, words into phrases and clauses, clauses into sentences, sentences into discourse, and discourse into pragmatic inference. A model with no representation of this hierarchy must recover it from data. The cost of that recovery is the question that \cref{sec-2} is organised around.  The answer is given mainly  in two statements. 
The first is the classical conjecture that an architecture without a structural prior must discriminate among the $\Omega(n^L)$ parse trees a depth-$L$ grammar admits over inputs of length $n$, at a sample cost exponential in $L$ (\cref{con:compositional}). This conjecture motivated the graded program, but no result below depends on it. 
The second is weaker, can be measured on an annotated corpus without training a model, and is the hypothesis the framework rests on: in language, the features that decide hierarchical targets are not the features that carry the corpus variance (\cref{assum:profile-divergence}).

This work introduces Graded Large Language Models (GLLMs), which supply the missing prior through the algebra of the representation space rather than through new architectural components. The construction builds on a line of work on graded vector spaces in machine learning: artificial neural networks over graded spaces~\cite{2024-02}, graded neural networks with graded neurons, activations, and losses~\cite{sh-89}, and the Graded Transformer framework comprising the Linearly Graded Transformer (LGT) and the Exponentially Graded Transformer (EGT)~\cite{sh-95}. The object underlying all of them is a graded vector space $\mathbb{R}^n_{\mathbf{q}}$: a copy of $\mathbb{R}^n$ equipped with a grade tuple $\mathbf{q}=(q_0,\dots,q_{n-1})$ and the induced scalar action
\[
\lambda \star x = (\lambda^{q_0}x_0,\dots,\lambda^{q_{n-1}}x_{n-1}),\qquad\lambda>1.
\]
A grade tuple states which coordinates matter and by how much. Assigning grades to the dimensions of an embedding space and propagating the action through attention, embeddings, and the loss amplifies the subspaces that carry syntactic heads, semantic nuclei, and pragmatic signal, and reduces the weight of the rest. Nothing is removed: the grading changes the metric of the representation space and leaves the computation over it intact, so expressive power, parallelisability, and autoregressive decoding are all preserved. What changes is the geometry the model searches over.

The construction has three parts and a cost analysis. \Cref{sec-3} develops exponentially graded multi-head self-attention (EG-MHSA) within causal decoder-only stacks and establishes three facts that locate the advantage. Grading leaves the asymptotic cost of attention unchanged at $O(n^2d+nd^2)$ per layer (\cref{prop:complexity}). The grading matrix is absorbed into the projection weights by an invertible substitution (\cref{rem:absorption}), so a trained GLLM compiles to a standard transformer of identical architecture and carries its prior at zero deployment cost (\cref{cor:zero-overhead}); no variant in \cref{tab:arch} has this property, since recurrence and sparsity encode their bias in the computation and pay for it at every forward pass. At a fixed norm budget the graded constraint set is an ellipsoid whose axes are aligned with the grades and which strictly contains the uniform ball (\cref{prop:expressivity}), so the advantage cannot come from representability: a larger class at a common budget hurts generalization rather than helping it. Together these facts place the advantage in the norm required to represent a given target, and hence in the data required to learn it. \Cref{sec-4} introduces Multi-Level Graded Embeddings (MLGE), decomposing $\mathbb{R}^d=\bigoplus_{l=0}^{L-1}V_l$ into level-specific subspaces, and carries out that analysis. \Cref{sec-5} extends the grading to the training objective and characterizes exactly when it remains consistent for the corpus conditional (\cref{prop:properness}). \Cref{sec-6} states the program that will test the resulting predictions, and \cref{sec-7} assembles the results into the framework's main purpose: a procedure that selects, certifies, and prices grades in domains that supply none.

The analysis yields the paper's central quantity. Write $\alpha$ for the profile of the target's energy across the graded basis and $\tau$ for the profile of the data's variance (\cref{def:profiles}). Then the Rademacher-based sample complexity of the graded class stands to that of the uniform class in the ratio
\begin{align*}
\Lambda(g)
&=
\Bigl(\sum_{j=1}^d\alpha_j\,g_j^{-2}\Bigr)
\Bigl(\sum_{j=1}^d\tau_j\,g_j^{2}\Bigr),
\\
\min_{g\in\mathbb{R}_{>0}^d}\Lambda(g)
&=
\Bigl(\sum_{j=1}^d\sqrt{\alpha_j\tau_j}\Bigr)^2
=\mathrm{BC}(\alpha,\tau)^2\le 1,
\end{align*}
with equality exactly when $\alpha=\tau$ (\cref{prop:ml-expressivity}). The entire benefit of grading is therefore the Bhattacharyya affinity between where the target places its weight and where the data places its variance. Grading creates no information: it converts a correct prior about the target into reduced sample complexity at this rate, and an incorrect prior into increased sample complexity at the same rate. Under geometric stratification of the profiles the ratio decays exponentially in the number of levels (\cref{cor:stratified}). This is the only route by which any exponential-in-$L$ statement enters this work, and it enters as a hypothesis, not as a theorem.

\Cref{prop:admissible-grades} reads that quantity as a condition on the grades themselves, and it is the paper's main positive result. The grades at which the graded bound is strictly smaller form an open convex set, invariant under translation along $\mathbf{1}$, so restricting to non-negative rationals costs nothing. The set is entered, to first order, exactly when
\[
\langle\mathbf{q},\alpha-\tau\rangle>0,
\]
that is, high grades on coordinates where the target's energy exceeds the data's variance and low grades elsewhere. The set is nonempty precisely when $\alpha\neq\tau$, and the standard transformer, which sets $g=\mathbf{1}$, lies on its boundary: the gain there is exactly $1$ and the gradient is $2(\log\lambda)(\tau-\alpha)\neq 0$. Isotropy is therefore not a local optimum of the sample-complexity landscape. Its gauge orbit is the unique set of grades at which the gain equals $1$ for every corpus and every target---the choice that ignores the profiles---and a descending direction is available from it in the direction the two profiles determine. Both profiles can be estimated offline, so the sign of that inner product is checkable before any pretraining begins (\cref{rem:certified-init}). The optimal grades within the region are then the solution of a convex program rather than the result of a search (\cref{cor:grade-convexity}), projected to the clipped region admitted by \cref{prop:clipped-selection}. The gain attained at grades computed from estimated profiles degrades only to second order in the estimation error (\cref{lem:plugin-stability}): the gradient of the gain vanishes at the optimum, so the framework is protected exactly where protection is needed.

The geometry of this selection problem admits a natural interpretation in geometric invariant theory, developed in \cref{subsec:git}. The admissible set is an open convex cone; the criterion $\langle\mathbf{q},\alpha-\tau\rangle>0$ is a Hilbert--Mumford-type pairing that decides which directions descend from the uniform point; and the ordinary transformer sits on the boundary as a semistable point of the grading action. The detailed development is deferred to that subsection so that the statistical content can be stated first in elementary language.

Three facts fix the location of the claim. The advantage is not one of expressivity: by \cref{rem:absorption,rem:mlge-absorption} graded and ungraded models realise the same function class, and the content of grading lies in parameterization, initialisation, and implicit bias. It is also not one of optimisation: exponential rescaling degrades the smoothness constant by a factor $\lambda^{q_{\max}}\ge 1$, so the stationarity guarantee available for a graded model is weaker than for its ungraded counterpart (\cref{lem:graded-smoothness,thm:convergence}). This asymmetry is the substance of the claim, and both of its sides are priced by the same constant: at clip $\lambda^{q_{\max}}\le C$ the penalty in optimisation steps is at most $C$, uniformly in $L$ and in the profiles, while the saving in tokens is at most $C^2$ (\cref{rem:optimisation-cost}, \cref{prop:clipped-selection}). The price is linear in the clip and the purchase quadratic, so the trade is favourable at every setting, and the exponential regime of \cref{cor:stratified} is reached as the clip opens. Grading spends a cheap resource---compute per step---to save an expensive one---high-quality tokens. The advantage lives in the norm required to represent a given target, and hence in the data required to learn it, which is where \cref{prop:ml-expressivity,prop:admissible-grades} place it in closed form.

\Cref{prop:ml-expressivity} compares two upper bounds, for a graded embedding composed with a linear read-out. Since graded and ungraded models realise the same functions, the separation the framework claims is between the graded prior and its absence rather than between two hypothesis classes. What converts the comparison of bounds into that separation is a matching minimax lower bound over the uniform target class, with the estimator unrestricted. \Cref{thm:stratified-separation} proves it in the level-stratified regime of \cref{cor:stratified}: at squared loss and over all estimators, the ratio of minimax risks is $\Theta(\Lambda(g))$ throughout an explicit window of sample sizes, so in the regime the planned configurations occupy the separation is established. The general case is \cref{con:minimax-separation}, and \cref{subsec:settling} records what its proof requires; \cref{rem:sample-complexity-scope} records the scope. Two limitations are resolved in this manuscript---the adaptivity cost of estimated grades by \cref{lem:plugin-stability} and the lower bound in the stratified regime by \cref{thm:stratified-separation}---and each of the remaining ones names the tool its repair would require.

The costs are small and are computed rather than asserted. Grading preserves the asymptotic complexity of attention exactly (\cref{prop:complexity}). The training overhead is dominated by the level projections at $d^2$ parameters, giving $0.33\%$ at 345M and $0.24\%$ at 7B for the configurations of \cref{subsubsec:configs}; this follows from $d^2\ll N$ at those scales and must be recomputed at other scales. At deployment the overhead is zero: by \cref{cor:zero-overhead,rem:mlge-absorption} a trained GLLM compiles to a standard transformer of identical architecture and identical inference cost, so the training figure is an upper bound on the total lifetime cost of the prior.

The framework's governing hypothesis is settled before any model is trained, and the predictions are staged so that the cheapest test comes first. Both profiles can be estimated offline, so $\mathrm{BC}(\alpha,\tau)$ is reported per task in advance; a value near $1$ on the hierarchical tasks falsifies \cref{assum:profile-divergence}, and with it every conditional prediction, at the cost of two estimates. The same two estimates certify every candidate grading against the criterion above, one inner product apiece. Configurations, baselines, benchmarks, and the selection ladder of \cref{tab:ladder} follow, with arms of identical architecture and budget distinguished only by what their grades consulted and ordered in advance by their certified inner products; each prediction is marked as established or as conditional on named hypotheses. The ordering is fixed before training, which makes it a prediction rather than a fit, and \cref{sec-6} specifies the runs that test it.

The framework extends in a direction the present construction does not reach. The grading of \cref{sec-4} is block-diagonal with respect to $\bigoplus_l V_l$ and therefore transports nothing between levels; directed level-to-level transport requires block-triangular morphisms $\phi_{l'\leftarrow l}:V_l\to V_{l'}$ internal to the graded hidden space. The morphic structure of~\cite{sh-111} supplies them, and subsumes external-tool paradigms as a special case via functorial internalization; MLGE does not.


\section{Background: Language Models and the Uniform Transformer}
\label{sec-2}

In this section  we fix the notation and review the classical results of the graded constructions introduced in  \cite{2024-02, sh-89, sh-95, sh-111}. The central question, made precise in \cref{con:compositional} and \cref{assum:profile-divergence}, is what the transformer's uniform treatment of positions and feature dimensions costs and where that cost is located.

\subsection{Language Models and the Autoregressive Objective}
\label{subsec:lm-objective}

\begin{defn}
\label{def:lm}
Let $\mathcal{V}$ be a finite vocabulary. A \emph{language model} is a parameterized probability distribution $p_\theta$ on $\mathcal{V}^{*}$.
\end{defn}

For $w = (w_1, \dots, w_T) \in \mathcal{V}^T$ the chain rule gives the autoregressive factorization
\begin{equation}
\label{eq:ar-factorization}
p_\theta(w) = \prod_{t=1}^{T} p_\theta(w_t \mid w_{<t}), \qquad w_{<t} := (w_1, \dots, w_{t-1}),
\end{equation}
reducing the task to estimation of the conditionals \cite{shannon1951prediction, bengio2003neural}. Quality is reported as perplexity, the exponentiated per-token cross-entropy
\begin{equation}
\label{eq:perplexity}
\mathrm{PPL}(w) = \exp\!\left( -\tfrac{1}{T} \textstyle\sum_{t=1}^{T} \log p_\theta(w_t \mid w_{<t}) \right).
\end{equation}
Every architecture in this paper, graded or not, is a parameterization of \cref{eq:ar-factorization}; grading acts on the representation of $w_{<t}$ and on the weight each conditional carries in the loss, and leaves the factorization itself intact.

The direct empirical objective for \cref{eq:ar-factorization} is causal language modeling,
\begin{equation}
\label{eq:clm}
\mathcal{L}_{\mathrm{CLM}}(\theta) = -\frac{1}{n}\sum_{t=1}^{n} \log p_\theta(x_t \mid x_{<t}) .
\end{equation}

\begin{lem}
\label{lem:ce-decomposition}
Write $p_t := p(\cdot \mid x_{<t})$ for the data conditional and $\hat{p}_t := p_\theta(\cdot \mid x_{<t})$ for the model's. The population form of \cref{eq:clm} decomposes as
\begin{equation}
\label{eq:ce-decomposition}
\mathbb{E}\!\left[\mathcal{L}_{\mathrm{CLM}}(\theta)\right] = \frac{1}{n}\sum_{t=1}^{n}\Bigl( H(p_t) + D_{\mathrm{KL}}(p_t \,\|\, \hat{p}_t) \Bigr),
\end{equation}
so the objective is bounded below by the conditional entropy of the data, attained exactly when $\hat{p}_t = p_t$ for all $t$.
\end{lem}

\begin{proof}
Take expectations in \cref{eq:clm} over $x_t \sim p_t$ and add and subtract $H(p_t)$ term by term.
\end{proof}

\Cref{eq:ce-decomposition} is the object the graded loss of \cref{sec-5} modifies: per-token weights $\lambda^{q_t}$ re-weight its summands, leaving the form of the decomposition intact while changing which positions dominate the gradient. Which re-weightings preserve the population minimizer is exactly the consistency question answered by \cref{prop:properness}.

\Cref{eq:clm} is non-convex in $\theta$, and no convergence guarantee to a global minimizer is known for transformer parameterizations; standard non-convex SGD analysis controls stationarity, 
\[
\min_{t \le T} \mathbb{E}\|\nabla \mathcal{L}\|^2 = O(1/\sqrt{T}),
\]
 not the value of \cref{eq:ce-decomposition}. This gap is inherited, not repaired, by grading, and \cref{thm:convergence} quantifies the direction in which grading moves it.

\subsection{The Transformer}
\label{sec:2.1}

Let  $n$ be the  sequence length, $d$ the model dimension, $H$ the number of heads, $d_k = d_v = d/H$ the per-head dimension, and $X \in \mathbb{R}^{n \times d}$ a layer input; see  \cite{vaswani2017attention} for the standard transformer and \cite{sh-95} for the graded transformer.

Tokens enter through an embedding matrix $W_e \in \mathbb{R}^{|\mathcal{V}| \times d}$; subword vocabularies are built by byte-pair encoding or unigram segmentation \cite{sennrich2016neural, kudo2018sentencepiece}, and tying the output projection to $W_e^\top$ couples the generation and representation geometries; the graded softmax of \cref{sec-5} acts on precisely this tied projection. Segmentation is chosen by corpus frequency and carries no record of morphological or syntactic role. Tokenisation is thus the first place hierarchy is discarded, and the level-$0$ subspace $V_0$ of \cref{sec-4} is the first point at which it is reintroduced. Since attention is permutation-equivariant in the rows of $X$, position is supplied explicitly, classically by the sinusoidal encoding

\begin{equation}
\label{eq:pe}
\mathrm{PE}(t, j) =
\begin{cases}
\sin\!\left( t \, / \, 10000^{\,2\lfloor j/2 \rfloor / d} \right), & j \text{ even}, \\[1ex]
\cos\!\left( t \, / \, 10000^{\,2\lfloor j/2 \rfloor / d} \right), & j \text{ odd},
\end{cases}
\end{equation}
added to the embeddings; \cref{eq:graded-pe-full} is its graded refinement.

With projections $W_{Q_h}, W_{K_h} \in \mathbb{R}^{d \times d_k}$, $W_{V_h} \in \mathbb{R}^{d \times d_v}$, and $W_O \in \mathbb{R}^{Hd_v \times d}$, multi-head self-attention is

\begin{equation}
\label{eq:mhsa}
\begin{split}
\mathrm{MHSA}(X) 	&	= \mathrm{Concat}(\mathrm{Head}_1, \dots, \mathrm{Head}_H)\, W_O, \\
 \mathrm{Head}_h 	&	= \mathrm{Attention}(X W_{Q_h}, X W_{K_h}, X W_{V_h}),
\end{split}
\end{equation}
where, with the causal mask $M_{\mathrm{causal}} \in \{0, -\infty\}^{n \times n}$, $(M_{\mathrm{causal}})_{ij} = -\infty$ for $j > i$, enforcing \cref{eq:ar-factorization},
\begin{equation}
\label{eq:causal-mask}
\mathrm{Attention}(Q, K, V) = \mathrm{softmax}\!\left( \frac{Q K^{\top}}{\sqrt{d_k}} + M_{\mathrm{causal}} \right) V .
\end{equation}

The scaling $d_k^{-1/2}$ is dictated by the variance of the logits: for $q, k$ with independent unit-variance entries, $\mathrm{Var}(q^\top k) = d_k$, so unscaled logits drive the softmax toward a one-hot distribution with vanishing gradient as $d_k$ grows. \Cref{eq:causal-mask} computes the bilinear form $q^\top k$, the Euclidean inner product on $\mathbb{R}^{d_k}$, in which no coordinate is distinguished. Replacing this form by a graded one is the single modification of \cref{subsec:eg-mhsa}.

Each position independently passes through a feed-forward network,
\begin{equation}
\label{eq:ffn}
\mathrm{FFN}(x) = \max(0,\, x W_1 + b_1) W_2 + b_2, \qquad W_1 \in \mathbb{R}^{d \times d_{ff}}, \ W_2 \in \mathbb{R}^{d_{ff} \times d},
\end{equation}
conventionally with $d_{ff} = 4d$; gated variants such as SwiGLU adjust $d_{ff}$ to hold the parameter count fixed \cite{shazeer2020glu}. Sublayers are wrapped in residual connections and normalisation, in the pre-norm ordering $x' = x + \mathrm{Sublayer}(\mathrm{LN}(x))$, which bounds gradients at initialisation independently of depth \cite{xiong2020pre}. Normalisation is by LayerNorm \cite{ba2016layer} or, standard in current open models, RMSNorm \cite{zhang2022rms},
\begin{equation}
\label{eq:rmsnorm}
\mathrm{RMSNorm}(x) = \gamma \odot \frac{x}{\sqrt{\tfrac{1}{d}\sum_{j} x_j^2 + \epsilon}} .
\end{equation}

\begin{prop}
\label{prop:ffn-fraction}
In a pre-norm block with $d_{ff} = 4d$ and $d_k = d_v = d/H$, the FFN accounts for $2/3$ of the block's parameters, up to $O(d)$ terms.
\end{prop}

\begin{proof}
Attention contributes $W_{Q}, W_{K}, W_{V}, W_O$, each $d \times d$ after concatenating heads, hence $4d^2$; the FFN contributes $W_1 \in \mathbb{R}^{d \times 4d}$ and $W_2 \in \mathbb{R}^{4d \times d}$, hence $8d^2$. Biases and normalisation parameters are $O(d)$, and $8d^2/12d^2 = 2/3$.
\end{proof}

\Cref{prop:ffn-fraction} calibrates every overhead claim in this paper: a graded parameter budget of $O(d + Hd_k)$ per block is negligible against the $12d^2$ established here, and the full accounting for the proposed configurations appears in \cref{subsec:setup}.

\begin{prop}
\label{prop:attention-cost}
One layer of \cref{eq:mhsa} together with \cref{eq:ffn} costs $O(n^2 d + n d^2)$ time.
\end{prop}

\begin{proof}
Per head, logits and value weighting cost $O(n^2 d_k)$; over $H$ heads with $d_k = d/H$ this is $O(n^2 d)$. Projections, concatenation, and \cref{eq:ffn} contribute $O(nd^2)$.
\end{proof}

\Cref{prop:complexity} shows grading preserves this bound exactly.

\subsection{Scaling Laws}
\label{subsec:scaling-laws}

Cross-entropy loss under \cref{eq:clm} follows fitted power laws in non-embedding parameters $N$, data $D$, and compute $C \approx 6ND$ \cite{kaplan2020scaling}; the compute budget $C$ of this subsection is standard notation and is distinct from the grade clip $C$ of \cref{subsubsec:configs}. The parameterization of Hoffmann et al.\ \cite{hoff},
\begin{equation}
\label{eq:chinchilla}
L(N,D) = E + \frac{A}{N^{\alpha}} + \frac{B}{D^{\beta}}, \qquad E \approx 1.69, \ \alpha \approx 0.34, \ \beta \approx 0.28,
\end{equation}
yields the compute-optimal allocation below.

\begin{prop}
\label{prop:chinchilla-allocation}
Minimising \cref{eq:chinchilla} subject to $6ND = C$ gives $N^{*} \propto C^{\beta/(\alpha+\beta)}$ and $D^{*} \propto C^{\alpha/(\alpha+\beta)}$; for $\alpha \approx \beta$ both exponents approach $1/2$, realised empirically at roughly $20$ tokens per parameter.
\end{prop}

\begin{proof}
Substituting $D = C/(6N)$ into \cref{eq:chinchilla} and differentiating in $N$ gives $-\alpha A N^{-\alpha - 1} + \beta B (6N/C)^{\beta} N^{-1} = 0$, whence $N^{\alpha+\beta} \propto C^{\beta}$.
\end{proof}

Two qualifications govern the use made of \cref{eq:chinchilla} here. First, it is a fitted regularity with no derivation from \cref{eq:clm}; it constrains the loss, not the capabilities the loss induces, and whether capability changes smoothly or abruptly with scale is disputed \cite{wei2022emergent, schaeffer2023mirage}. Second, \cref{eq:chinchilla} aggregates all positions into one scalar: a token carrying a clause boundary and a token carrying a function word contribute identically. The graded loss of \cref{sec-5} reweights the summands of \cref{eq:ce-decomposition} by token grade, which plausibly changes the constants of \cref{eq:chinchilla} rather than its form; whether it changes the exponents is an empirical question this paper does not settle.

\subsection{The Uniformity Principle and Its Cost}
\label{subsec:uniformity}

Every component above treats its inputs uniformly. Attention \cref{eq:causal-mask} scores all pairs through one Euclidean form; the FFN \cref{eq:ffn} applies one map to all positions; the loss \cref{eq:clm} weights all positions equally; and the scaling law \cref{eq:chinchilla} aggregates them into one scalar. We call this the \emph{uniformity principle}: all tokens, parameters, and feature dimensions carry equal structural status. Its geometric expression is that the representation space $\mathbb{R}^d$ is isotropic: no coordinate direction is distinguished, so no direction can carry structural meaning that the architecture must respect. Formally, \cref{eq:causal-mask} is equivariant under the orthogonal group acting on $\mathbb{R}^{d_k}$ jointly through the projections, and in particular under the unweighted scalar action $\lambda \cdot x$; it is equivariant under the weighted action $\lambda \star x = (\lambda^{q_j} x_j)$ only in the degenerate case $q = (1, \dots, 1)$. When data carries a weighted structure with $q \neq (1,\dots,1)$, the architecture is committed to the wrong symmetry, and the structure must be recovered from data if at all.

Existing architectural variants relax uniformity only as a function of positional offset. \Cref{tab:arch} classifies them by the structural information their inductive bias can express.

\begin{table}[h]
\centering
\small
\begin{tabular}{@{}llll@{}}
\toprule
Variant & Key mechanism & Examples & Structural bias \\
\midrule
Decoder-only & Causal autoregression & GPT series, Llama & None (emergent) \\
Encoder--decoder & Bidirectional encoding & T5, BART & None (task-supervised) \\
State space & Selective linear recurrence & Mamba-2 & Recency only \\
Sparse attention & Restricted mask support & Longformer, BigBird & Locality only \\
Hybrid SSM--attention & Interleaved recurrence/attention & Jamba & Recency $+$ locality \\
Graded (this paper) & Exponential subspace scaling & \cref{sec-3,sec-4} & Explicit, multi-level \\
\bottomrule
\end{tabular}
\caption{Architecture variants by structural inductive bias \cite{gu2023mamba, gu2024mamba2, peng2023rwkv, raffel2020exploring, radford2018improving}. ``Recency'' and ``locality'' denote biases that are functions of positional offset $|i - j|$ alone; no existing variant conditions on structural role. Grading acts on the metric of the representation space, sparsity and recurrence on the mask, and the two compose.}
\label{tab:arch}
\end{table}

The cost of uniformity that motivates this paper is compositional, and we record it in two statements with different levels of support. The first is the classical conjecture, which this paper does not prove.

\begin{con}
\label{con:compositional}
An $L$-level context-free grammar admits $\Omega(n^L)$ parse trees over inputs of length $n$. The sample complexity of an architecture with no structural prior, trained by \cref{eq:clm} to a fixed error on structures of depth $L$, grows exponentially in $L$.
\end{con}

\Cref{con:compositional} is consistent with the observed accuracy decay on $k$-nested arithmetic, $k$-hop inference, and $k$-deep call stacks, and it motivated the graded program; but the results of this paper neither prove it nor depend on it. What the theory developed in \cref{sec-4} requires is strictly weaker, and unlike \cref{con:compositional} it is directly measurable on annotated corpora.

\begin{assum}[Profile Divergence]
\label{assum:profile-divergence}
For hierarchical prediction targets in natural language, the target energy profile $\alpha$ and the data energy profile $\tau$ of \cref{def:profiles} are far apart in Bhattacharyya affinity: structurally decisive features carry little of the corpus variance, and high-variance lexical features carry little of the decision.
\end{assum}

\Cref{assum:profile-divergence} is the empirical hypothesis on which the benefit of grading rests. \Cref{prop:ml-expressivity} converts it into an exact sample-complexity ratio, $\Lambda^\star = \mathrm{BC}(\alpha, \tau)^2$, and \cref{cor:stratified} shows that under geometric stratification of the profiles the ratio decays exponentially in the number of levels; this is the only route by which any exponential-in-$L$ statement enters this work, and it enters as a hypothesis, not as a theorem. Both profiles can be estimated offline, so \cref{assum:profile-divergence} is testable before any model is trained; \cref{sec-8} carries out the construction of the estimators and \cref{sec-6} stages the measurement first.

\section{Graded Attention in Decoder-Only Architectures}
\label{sec-3}

\Cref{sec-2} closed with a deficit stated in two forms: \cref{con:compositional}, which this paper does not prove, and \cref{assum:profile-divergence}, which it makes precise and quantifies. The transformer's representation space is isotropic, so hierarchical competence must be recovered from data. The remainder of the paper addresses the deficit by changing the metric on the representation space rather than the computation performed over it.

This section develops exponentially graded multi-head self-attention within causal decoder-only stacks and establishes three facts. Grading leaves the asymptotic cost of attention unchanged (\cref{prop:complexity}). It acts by an invertible change of parameters (\cref{rem:absorption}), so a trained GLLM compiles to a standard transformer and carries its prior at zero deployment cost (\cref{cor:zero-overhead}), a property none of the variants in \cref{tab:arch} shares. And at a fixed norm budget the graded constraint set is an \emph{ellipsoid} whose axes are aligned with the grades (\cref{prop:expressivity}), strictly containing the uniform ball; the advantage therefore cannot come from representability, since a larger class at fixed budget hurts generalisation. Together these facts place the advantage in the \emph{norm required to represent a given target}, and hence in the data required to learn it. \Cref{sec-4} decomposes the embedding space into level-specific subspaces via Multi-Level Graded Embeddings and computes that quantity exactly (\cref{prop:ml-expressivity}), characterises in closed form the grades at which the graded bound is the smaller one (\cref{prop:admissible-grades}), and gives the behaviour under stratified profiles (\cref{cor:stratified}). \Cref{sec-5} extends the grading to the training objective and characterises its consistency (\cref{prop:properness}). \Cref{sec-6} states the program that will test the resulting predictions, and \cref{sec-7} assembles the results into a procedure that selects, certifies, and prices the grades in domains where they are not given.

Three distinct grade objects appear in what follows and should not be conflated: per-head feature grades $q_h \in \mathbb{Q}_{\geq 0}^{d_k}$ (this section), per-level embedding grades $q_l \in \mathbb{Q}_{\geq 0}^{d_l}$ (\cref{sec-4}), and per-token loss grades $q_t \in \mathbb{Q}_{\geq 0}$ (\cref{sec-5}). All take values in $\mathbb{Q}_{\geq 0}$, and all act through the same scalar action $\lambda \star x = (\lambda^{q_i} x_i)$ with a common base $\lambda > 1$.

We work in the decoder-only paradigm of the GPT series \cite{radford2019language}, with multi-head self-attention and the causal mask as fixed in \cref{eq:mhsa,eq:causal-mask}: queries $Q_h = X W_{Q_h}$, keys $K_h = X W_{K_h}$, and values $V_h = X W_{V_h}$ in $\mathbb{R}^{n \times d_k}$ for each head $h \in \{1, \dots, H\}$, with $d_k = d_v = d/H$ and the mask $M_{\mathrm{causal}}$ enforcing \cref{eq:ar-factorization}. Every position and every feature dimension enters \cref{eq:causal-mask} identically: the cost is $O(n^2 d + n d^2)$ per layer by \cref{prop:attention-cost}, and no coordinate of $\mathbb{R}^{d_k}$ is distinguished from any other.

\subsection{Exponentially Graded Multi-Head Self-Attention}
\label{subsec:eg-mhsa}

In GLLMs we replace MHSA by \emph{Exponentially Graded Multi-Head Self-Attention} (EG-MHSA), built from the Exponentially Graded Transformer of \cite{sh-95}. Each head carries its own grading of the per-head space $\mathbb{R}^{d_k} = \bigoplus_{j=1}^{d_k} \mathbb{R}e_j$, assigning to the basis vector $e_j$ a grade $q_{h,j} \in \mathbb{Q}_{\geq 0}$. The grading transformation is the diagonal matrix
\[
M_{q_h, \lambda} = \operatorname{diag}\!\left(\lambda^{q_{h,1}}, \dots, \lambda^{q_{h,d_k}}\right) \in \mathbb{R}^{d_k \times d_k},
\]
with fixed base $\lambda > 1$. Taking $\lambda = e^{1/d_k}$ gives $\lambda \approx 1.008$ at $d_k = 128$ and $\lambda \approx 1.016$ at $d_k = 64$, so that the clipped grade range $q_{\max} = d_k \ln C$ of \cref{subsubsec:configs} spans a spread of order $d_k$. The base is chosen small precisely because the grades are not, and the clip $C$, not the base, is the constant that prices the construction (\cref{rem:one-dial}).

The graded queries, keys, and values are obtained by right-multiplication:
\begin{equation}
\label{eq:graded-proj}
Q_h' = Q_h M_{q_h, \lambda}, \qquad K_h' = K_h M_{q_h, \lambda}, \qquad V_h' = V_h M_{q_h, \lambda} .
\end{equation}

The graded head is
\begin{equation}
\label{eq:eg-head}
\mathrm{Head}_h' = \mathrm{softmax}\!\left( \frac{Q_h' (K_h')^\top}{\sqrt{d_k}} + M_{\mathrm{causal}} \right) V_h' .
\end{equation}

Since $M_{q_h,\lambda}$ is diagonal and hence symmetric, $(K_h')^\top = M_{q_h,\lambda} K_h^\top$, so the pre-softmax logits are
\[
\frac{Q_h' (K_h')^\top}{\sqrt{d_k}} = \frac{Q_h M_{q_h,\lambda}^2 K_h^\top}{\sqrt{d_k}},
\qquad
M_{q_h,\lambda}^2 = \operatorname{diag}\!\left(\lambda^{2q_{h,1}}, \dots, \lambda^{2q_{h,d_k}}\right),
\]
so that agreement between query and key along $e_j$ is weighted by $\lambda^{2q_{h,j}}$: the grading acts on the bilinear form defining attention, replacing the Euclidean inner product of \cref{eq:causal-mask} on $\mathbb{R}^{d_k}$ by the graded form $\langle u, v\rangle_{q_h} = u^\top M_{q_h,\lambda}^2 v$. The multi-head output is
\[
\mathrm{EG\text{-}MHSA}(X) = \mathrm{Concat}(\mathrm{Head}_1', \dots, \mathrm{Head}_H')\, W_O .
\]

The grading induces a filtration of the per-head space by grade,
\[
V_{\leq q}^{(h)} = \bigoplus_{j \,:\, q_{h,j} \leq q} \mathbb{R}e_j ,
\qquad
V_{\leq q}^{(h)} \subseteq V_{\leq q'}^{(h)} \text{ for } q \leq q',
\]
which is the algebraic object that linguistic stratification is to be mapped onto: low grades for lexical and morphological features, high grades for phrasal and discourse-level ones \cite{2024-02, sh-89}.

\begin{prop}
\label{prop:complexity}
EG-MHSA has the same asymptotic time and space complexity as standard MHSA, namely $O(n^2 d + n d^2)$ per layer including feed-forward components.
\end{prop}

\begin{proof}
The grading step \cref{eq:graded-proj} applies $3H$ diagonal multiplications, each costing $O(n d_k)$, for a total of $O(n H d_k) = O(nd)$. The matrix $M_{q_h,\lambda}^2$ is precomputable per head in $O(d_k)$. The attention computation \cref{eq:eg-head} is structurally identical to standard MHSA: $O(n^2 d_k)$ per head for the logits and $O(n^2 d_k)$ per head for the softmax and value weighting, giving $O(n^2 d)$ over $H$ heads. Concatenation and the output projection add $O(nd^2)$, and the feed-forward block $O(nd^2)$. Summing gives $O(n^2 d + nd^2)$, matching \cref{prop:attention-cost}.
\end{proof}

\begin{rem}
\label{rem:absorption}
Because $M_{q_h,\lambda}$ is invertible and $W_{Q_h}, W_{K_h}, W_{V_h}$ are unconstrained, the substitution $\tilde W_{\bullet h} = W_{\bullet h} M_{q_h,\lambda}$ is a bijection of the parameter space carrying \cref{eq:eg-head} to a standard head. EG-MHSA and MHSA therefore realise the same set of functions. In this precise sense the grading changes the \emph{metric} of the representation space and leaves the computation over it intact, and two consequences follow, running in opposite directions. On one side, no separation is available from representability. On the other, an invertible reparameterization can be applied once and then discarded, which gives \cref{cor:zero-overhead}. What grading fixes is a parameterization, and with it an initialisation, an implicit bias under gradient descent, and, once the projection weights are norm-constrained, an effective hypothesis class whose geometry is computed in \cref{prop:expressivity}. This is a statement about which functions are \emph{reachable and preferred}, not about which are \emph{representable}, and the informative comparison fixes the target rather than the budget (\cref{prop:ml-expressivity}).
\end{rem}

\begin{cor}
\label{cor:zero-overhead}
Let $\theta$ be the parameters of a trained GLLM whose attention sublayers are EG-MHSA with grades $\{q_h\}$. The map $\theta \mapsto \tilde\theta$ given by $\tilde W_{\bullet h} = W_{\bullet h} M_{q_h,\lambda}$ for $\bullet \in \{Q,K,V\}$ yields a standard decoder-only transformer of identical architecture computing the identical function on every input. The grades need not be stored at inference, and the diagonal multiplications of \cref{eq:graded-proj} need not be performed.
\end{cor}

\begin{proof}
Immediate from the bijection of \cref{rem:absorption}, applied once and offline. Since $M_{q_h,\lambda}$ is diagonal with strictly positive entries, $\tilde W_{\bullet h}$ is well defined and the substitution violates no constraint on $\theta$.
\end{proof}

\begin{rem}
\label{rem:overhead-comparison}
The graded prior is paid for once, during training, and costs nothing afterwards. This separates grading from every variant classified in \cref{tab:arch}: selective recurrence, restricted mask support, and hybrid schemes encode their inductive bias in the computation, pay for it at every forward pass, and do not compile to the architecture they modify. A metric can be absorbed into the weights, while a mask cannot. The training overhead accounted in \cref{subsubsec:overhead} is therefore an upper bound on the total lifetime cost of the prior, and \cref{sec-6} reports zero inference overhead as an established consequence of the construction rather than an experimental outcome.
\end{rem}

\begin{prop}
\label{prop:expressivity}
Fix grades $\{q_h\}_{h=1}^H$ and a budget $B > 0$, and let
\[
\begin{split}
\mathcal{F}_B^{\mathrm{graded}}
 &= \left\{ X \mapsto \mathrm{EG\text{-}MHSA}(X) \;:\; \|W_{\bullet h}\|_F \leq B \ \ \forall\, \bullet \in \{Q, K, V\}, \ \forall h \right\}, \\
\mathcal{F}_B^{\mathrm{unif}}
 &= \left\{ X \mapsto \mathrm{MHSA}(X) \;:\; \|W_{\bullet h}\|_F \leq B \ \ \forall\, \bullet \in \{Q, K, V\}, \ \forall h \right\}.
\end{split}
\]
In the effective coordinates $\tilde W_{\bullet h} = W_{\bullet h} M_{q_h,\lambda}$ of \cref{rem:absorption}, the budget constraint defining $\mathcal{F}_B^{\mathrm{graded}}$ is the ellipsoid
\[
\sum_{j=1}^{d_k} \lambda^{-2 q_{h,j}} \big\| \tilde W_{\bullet h}\, e_j \big\|_2^2 \;\leq\; B^2 ,
\]
with semi-axis $B \lambda^{q_{h,j}}$ along the $j$-th column, whereas that defining $\mathcal{F}_B^{\mathrm{unif}}$ is the Frobenius ball of radius $B$. Consequently
\[
\mathcal{F}_B^{\mathrm{unif}} \subseteq \mathcal{F}_B^{\mathrm{graded}},
\]
with strict inclusion whenever some $q_{h,j} > 0$, and the Rademacher complexity of the graded class is weakly greater at every fixed $B$.
\end{prop}

\begin{proof}
By the substitution of \cref{rem:absorption}, a graded head with weights $W_{\bullet h}$ computes the same function as a standard head with effective weights $\tilde W_{\bullet h} = W_{\bullet h} M_{q_h, \lambda}$. Under this bijection the constraint $\|W_{\bullet h}\|_F \leq B$ becomes $\|\tilde W_{\bullet h} M_{q_h,\lambda}^{-1}\|_F \leq B$, which is the displayed ellipsoid; its semi-axis along $e_j$ is $B\lambda^{q_{h,j}}$. Since $q_{h,j} \geq 0$ and $\lambda > 1$, every semi-axis is at least $B$, so the ellipsoid contains the Frobenius ball of radius $B$, strictly in the $j$-th direction whenever $q_{h,j} > 0$. Passing from weight sets to function classes preserves the inclusion, and any $\tilde W_{\bullet h}$ with $\|\tilde W_{\bullet h}\|_F = B$ supported on a coordinate with $q_{h,j} > 0$ can be scaled by $\lambda^{q_{h,j}}$ to lie in the graded class but not the uniform one, giving strictness. Monotonicity of Rademacher complexity under inclusion of hypothesis classes gives the final claim.
\end{proof}

\begin{cor}
\label{cor:no-fixed-budget}
No capacity bound favouring the graded model is available at fixed $B$.
\end{cor}

\Cref{prop:expressivity} identifies the geometric object grading produces: an \emph{ellipsoidal} constraint set, its axes aligned with the grades and its eccentricity along $e_j$ equal to $\lambda^{q_{h,j}}$. Reading the proposition as \cref{cor:no-fixed-budget} alone understates it. The informative comparison fixes the target rather than the budget, asks each class for the smallest budget at which it represents that target, and evaluates the capacity bounds at those budgets; \cref{prop:ml-expressivity} carries this out and identifies the sample-complexity ratio in closed form, and \cref{prop:admissible-grades} characterises exactly the grades at which that ratio favours the graded bound. The ellipsoid is also what makes the comparison sharp: it is the geometry on which a matching minimax lower bound is posed, and on it the separation between the graded prior and its absence, rather than a ratio of upper bounds, is obtained. \Cref{thm:stratified-separation} establishes that separation in the level-stratified regime, over all estimators, throughout the window of \cref{lem:window}; \cref{con:minimax-separation} states the general case, and \cref{subsec:settling} records what its proof requires.

\section{Hierarchical Grading for Language Modeling}
\label{sec-4}

Linguistic structure is layered. Subword units aggregate into words, words compose into phrases and clauses, clauses into sentences, sentences into discourse, and discourse into pragmatic inference; the observation is common to the generative and model-theoretic traditions in linguistics, and we take it as given rather than argue it here. Classical LLMs approximate this stratification implicitly, through patterns emergent in unannotated corpora \cite{radford2019language}. Their embedding space is isotropic: no coordinate direction is distinguished, so no direction can be reserved for a linguistic level. We propose \emph{Multi-Level Graded Embeddings} (MLGE), which give the stratification an explicit home in the geometry of the embedding space, extending the graded vector space formalism of \cite{2024-02, sh-89, sh-95} to autoregressive language modeling. \Cref{subsec:mlge-sample-complexity} then quantifies the construction: it establishes the exact sample-complexity ratio obtained by grading (\cref{prop:ml-expressivity}), characterises in closed form the grades at which the graded model is favoured and locates the standard transformer on the boundary of that set (\cref{prop:admissible-grades}), and shows the ratio decays exponentially in depth under stratified profiles (\cref{cor:stratified}), converting \cref{assum:profile-divergence} into a computable, falsifiable quantity.

\subsection{The Graded Embedding Space}
\label{subsec:mlge-space}

Let $E = (\varepsilon_1, \dots, \varepsilon_n)^\top \in \mathbb{R}^{n \times d}$ be the input embedding matrix for a token sequence $y = (y_1, \dots, y_n) \in \mathcal{V}^n$, with $\varepsilon_t = W_e[y_t] \in \mathbb{R}^d$ and learnable weights $W_e \in \mathbb{R}^{|\mathcal{V}| \times d}$ as in \cref{sec:2.1}. We decompose the embedding space as an orthogonal direct sum
\[
\mathbb{R}^d = \bigoplus_{l=0}^{L-1} V_l,
\qquad
V_l \cong \mathbb{R}^{d_l},
\qquad
\sum_{l=0}^{L-1} d_l = d,
\]
with $d_l = d/L$ under uniform allocation, where $V_l$ is reserved for features salient to linguistic level $l$: $V_0$ for subword and morphological structure, $V_1$ for syntactic constituency, $V_2$ for sentential semantics, $V_{L-1}$ for discourse and pragmatics. Level-specific components are $E_l = E P_l \in \mathbb{R}^{n \times d_l}$, where $P_l \in \mathbb{R}^{d \times d_l}$ has orthonormal columns spanning $V_l$, initialised by principal component analysis of features annotated for level $l$ by an offline constituency parser.

When the $P_l$ are exactly orthonormal with mutually orthogonal ranges, $\sum_{l} P_l P_l^\top = I_d$ and the reconstruction $E = \sum_{l} E_l P_l^\top$ is exact. During training the columns are only approximately orthonormal, and we penalise
\[
\Omega_{\mathrm{orth}}(P) = \sum_{l, l'} \big\| P_l^\top P_{l'} - \delta_{l l'} I_{d_l} \big\|_F^2 ,
\]
so that the reconstruction error is controlled by $\Omega_{\mathrm{orth}}(P)$ and hence by the penalty weight. The error is a function of optimisation, not of $L$, and no rate in $L$ is claimed.

Each $V_l$ carries a grading tuple $q_l = (q_{l,1}, \dots, q_{l,d_l}) \in \mathbb{Q}_{\geq 0}^{d_l}$, whose entries are rational so that they may record discrete structural depths exactly. The associated grading transformation is
\[
M_{q_l, \lambda} = \operatorname{diag}\!\left(\lambda^{q_{l,1}}, \dots, \lambda^{q_{l,d_l}}\right) \in \mathbb{R}^{d_l \times d_l},
\]
with the base $\lambda > 1$ fixed once for the whole architecture, as in \cref{sec-3}. The graded embeddings are
\begin{equation}
\label{eq:mlge-transform}
E' = \sum_{l=0}^{L-1} E_l \, M_{q_l, \lambda} \, P_l^\top
   = E \, G ,
\qquad
G := \sum_{l=0}^{L-1} P_l M_{q_l, \lambda} P_l^\top ,
\end{equation}
which in the graded basis is the block-diagonal map $\bigoplus_{l} M_{q_l,\lambda}$. Evaluating \cref{eq:mlge-transform} costs $O(ndL)$.

\begin{rem}
\label{rem:mlge-absorption}
$G$ is symmetric positive definite, hence invertible, and $W_e$ is unconstrained. The substitution $\tilde W_e = W_e G$ is therefore a bijection of the parameter space carrying \cref{eq:mlge-transform} to an ungraded embedding, and MLGE realises exactly the function class of a standard embedding layer. As in \cref{rem:absorption}, two consequences follow in opposite directions. On one side, the construction adds no capacity, so no separation is available from representability. On the other, the substitution extends \cref{cor:zero-overhead} from the attention sublayers to the embedding layer, so the whole of the graded construction, attention, embeddings, and positional encoding alike, is folded into the weights at the end of training and costs nothing at inference. What the construction contributes is its parameterization: $G$ fixes an initialisation aligned with annotated linguistic levels, and constrains the geometry along which gradient descent moves.
\end{rem}

\begin{prop}
\label{prop:direct-sum}
Let $\eta := \max_{l \neq l'} \|P_l^\top P_{l'}\|_{op}$ and $\lambda^{q_{\max}} := \max_{l,j} \lambda^{q_{l,j}}$. Then $G$ of \cref{eq:mlge-transform} satisfies $G(V_l) \subseteq V_l$ exactly when $\eta = 0$, and in general
\[
\Big\| G - \bigoplus_{l=0}^{L-1} M_{q_l, \lambda} \Big\|_{op} \;\leq\; L\,\lambda^{q_{\max}}\, \eta .
\]
\end{prop}

\begin{proof}
When $\eta = 0$ the ranges of the $P_l$ are mutually orthogonal, so $P_{l'}^\top G P_{l'} = M_{q_{l'},\lambda}$ and $P_{l'}^\top G P_l = 0$ for $l \neq l'$; thus $G$ is block-diagonal in the graded basis and preserves each $V_l$. In general, the off-diagonal block indexed by $(l', l)$ is $P_{l'}^\top P_l M_{q_l,\lambda} P_l^\top P_{l'}$, of operator norm at most $\lambda^{q_{l,\max}} \eta^2 \leq \lambda^{q_{\max}}\eta$ for $\eta \leq 1$, and there are at most $L$ such blocks in any row. Summing gives the bound.
\end{proof}

\begin{rem}
\label{rem:direct-sum-scope}
\Cref{prop:direct-sum} is a consistency check, not a capability claim: it states that the intended block structure survives approximate orthogonality of the $P_l$, with degradation linear in $L$ and in the grading amplitude $\lambda^{q_{\max}}$. It is the first place where large grades are seen to cost something.
\end{rem}


\subsection{Sample Complexity Under Grading}
\label{subsec:mlge-sample-complexity}

Two routes to a separation are closed before we begin, and closing them fixes the location of the third. By \cref{rem:mlge-absorption} graded and ungraded embeddings realise the same functions, so nothing follows from representability. By \cref{prop:expressivity} the graded constraint set at a common norm budget is an ellipsoid strictly containing the uniform ball, with weakly greater Rademacher complexity, so nothing favouring the graded model follows at fixed $B$ either.

The informative comparison fixes the target rather than the budget. Each class is asked for the smallest norm at which it represents a given $f^*$, and the resulting bound is evaluated at that norm. The grading then acts on two quantities at once: it lowers the norm required for targets aligned with it, and it raises the effective radius of the data. The theorem below is the exact trade-off between them. \Cref{prop:admissible-grades} then reads that trade-off as a condition on the grades themselves, and identifies the standard transformer as the point at which the trade-off is declined.

Throughout this subsection the analysis is carried out for the graded linear map $x \mapsto \langle u, Gx \rangle$ realised by MLGE composed with a single linear read-out. It is a statement about the embedding layer and its immediate consumer, not about the full stack; the extension is discussed in \cref{rem:sample-complexity-scope}.

\begin{defn}[Energy Profiles]
\label{def:profiles}
Let $\mathcal{P}$ be a distribution on $\mathbb{R}^d \times \mathbb{R}$ with $\sigma_j^2 := \mathbb{E}[x_j^2] \in (0,\infty)$ for every $j$, and let $f^*(x) = \langle w^*, x \rangle$ with $w_j^* \neq 0$ for every $j$. Writing $\Delta^{d-1}$ for the probability simplex in $\mathbb{R}^d$, the \emph{target profile} $\alpha \in \Delta^{d-1}$ and the \emph{data profile} $\tau \in \Delta^{d-1}$ are
\[
\alpha_j = \frac{(w_j^*)^2}{\|w^*\|_2^2},
\qquad
\tau_j = \frac{\sigma_j^2}{\sum_{k=1}^{d} \sigma_k^2} .
\]
Both are supported on the graded basis of \cref{eq:mlge-transform}, both are probability vectors, and both are strictly positive in every coordinate. A coordinate on which $w^*$ vanishes contributes nothing to $\alpha$ and is optimally suppressed by $g_j \to 0$; such coordinates are deleted from the analysis rather than carried through it.
\end{defn}

\begin{defn}[Graded and Uniform Classes at a Target]
\label{def:classes-at-target}
For $B > 0$ set
\[
\begin{split}
\mathcal{H}_B^{\mathrm{graded}} &= \left\{ x \mapsto \langle u, Gx \rangle \;:\; u \in \mathbb{R}^d, \ \|u\|_2 \leq B \right\}, \\
\mathcal{H}_B^{\mathrm{unif}} &= \left\{ x \mapsto \langle w, x \rangle \;:\; w \in \mathbb{R}^d, \ \|w\|_2 \leq B \right\}.
\end{split}
\]
The \emph{representation budgets} of $f^*$ are
\[
B^{\mathrm{graded}}(f^*) = \min \left\{ \|u\|_2 : Gu = w^* \right\} = \big\| G^{-1} w^* \big\|_2 ,
\qquad
B^{\mathrm{unif}}(f^*) = \|w^*\|_2 ,
\]
the smallest budgets at which each family contains $f^*$.
\end{defn}

\begin{prop}
\label{prop:ml-expressivity}
Let $\alpha$ and $\tau$ be as in \cref{def:profiles} and let $\ell$ be a loss that is $L_\ell$-Lipschitz in its first argument. Define the \emph{grading gain}
\[
\Lambda(g)
 = \left( \sum_{j=1}^{d} \alpha_j\, g_j^{-2} \right)
   \left( \sum_{j=1}^{d} \tau_j\, g_j^{2} \right),
\qquad
g \in \mathbb{R}_{>0}^{d} .
\]
Then:
\begin{enumerate}
  \item[(i)] $\Lambda$ is invariant under $g \mapsto cg$ for every $c > 0$, and $\Lambda(\mathbf{1}) = 1$;
  \item[(ii)] the Rademacher-based sample complexity of reaching excess risk $\varepsilon$ over $\mathcal{H}^{\mathrm{graded}}_{B^{\mathrm{graded}}(f^*)}$ stands to that over $\mathcal{H}^{\mathrm{unif}}_{B^{\mathrm{unif}}(f^*)}$ in the ratio $\Lambda(g)$, uniformly in $\varepsilon$ and $f^*$;
  \item[(iii)] the gain is minimised at
  \[
  g_j^{\star} \;\propto\; \alpha_j^{1/4}\, \tau_j^{-1/4} ,
  \]
  and the optimal value is the squared Bhattacharyya affinity of the two profiles,
  \begin{equation}
  \label{eq:optimal-gain}
  \Lambda^{\star} := \min_{g \in \mathbb{R}_{>0}^{d}} \Lambda(g)
   = \left( \sum_{j=1}^{d} \sqrt{\alpha_j \tau_j} \right)^{2}
   = \mathrm{BC}(\alpha, \tau)^{2} ;
  \end{equation}
  \item[(iv)] $\Lambda^{\star} \leq 1$, with equality if and only if $\alpha = \tau$.
\end{enumerate}
\end{prop}

\begin{proof}
(i) Replacing $g$ by $cg$ multiplies the first factor by $c^{-2}$ and the second by $c^{2}$. At $g = \mathbf{1}$ the factors are $\sum_j \alpha_j = 1$ and $\sum_j \tau_j = 1$.

(ii) Fix $B > 0$ and a sample $x_1, \dots, x_m$. Writing $z_i = Gx_i$, the class $\mathcal{H}_B^{\mathrm{graded}}$ is the $B$-norm-bounded linear class in the variable $z$, so its empirical Rademacher complexity is
\[
\begin{split}
\widehat{\mathfrak{R}}_m\big( \mathcal{H}_B^{\mathrm{graded}} \big)
 &= \frac{1}{m} \, \mathbb{E}_{\epsilon} \sup_{\|u\|_2 \leq B} \Big\langle u, \sum_{i=1}^{m} \epsilon_i z_i \Big\rangle
  = \frac{B}{m} \, \mathbb{E}_{\epsilon} \Big\| \sum_{i=1}^{m} \epsilon_i G x_i \Big\|_2 \\
 &\leq \frac{B}{m} \left( \mathbb{E}_{\epsilon} \Big\| \sum_{i=1}^{m} \epsilon_i G x_i \Big\|_2^2 \right)^{1/2}
  = \frac{B}{m} \left( \sum_{i=1}^{m} \| G x_i \|_2^2 \right)^{1/2},
\end{split}
\]
the last equality because $\mathbb{E}[\epsilon_i \epsilon_{i'}] = \delta_{i i'}$. Taking expectations over the sample and using $\mathbb{E}\|Gx\|_2^2 = \sum_j g_j^2 \sigma_j^2$,
\[
\mathfrak{R}_m\big( \mathcal{H}_B^{\mathrm{graded}} \big)
 \leq \frac{B}{\sqrt{m}} \left( \sum_{j=1}^{d} g_j^2 \sigma_j^2 \right)^{1/2}.
\]
Setting $g = \mathbf{1}$ gives the uniform case. Evaluating each bound at its own representation budget, and using $\|G^{-1}w^*\|_2^2 = \|w^*\|_2^2 \sum_j \alpha_j g_j^{-2}$ from \cref{def:profiles},
\[
\begin{split}
B^{\mathrm{graded}}(f^*) \left( \sum_{j} g_j^2 \sigma_j^2 \right)^{1/2}
 &= \|w^*\|_2 \left( \sum_{k} \sigma_k^2 \right)^{1/2}
    \left( \sum_{j} \alpha_j g_j^{-2} \right)^{1/2} \left( \sum_{j} \tau_j g_j^{2} \right)^{1/2} \\
 &= B^{\mathrm{unif}}(f^*) \left( \sum_{k} \sigma_k^2 \right)^{1/2} \Lambda(g)^{1/2} .
\end{split}
\]
For an $L_\ell$-Lipschitz loss the standard symmetrisation and contraction estimate of Bartlett and Mendelson \cite{bartlett2002rademacher} bounds the excess risk of empirical risk minimisation over a class $\mathcal{H}$ containing $f^*$ by $2 L_\ell \mathfrak{R}_m(\mathcal{H}) + O(\sqrt{\log(1/\delta)/m})$, so the sample size sufficient for excess risk $\varepsilon$ scales as the square of the product of budget and radius. That product carries the factor $\Lambda(g)^{1/2}$ in the graded case and $1$ in the uniform case, and the ratio of sample sizes is $\Lambda(g)$, independently of $\varepsilon$ and of $\|w^*\|_2$.

(iii) Substitute $t_j = g_j^2 > 0$, so that $\Lambda = \big( \sum_j \alpha_j t_j^{-1} \big) \big( \sum_j \tau_j t_j \big)$. By Cauchy--Schwarz applied to the vectors with entries $\sqrt{\alpha_j / t_j}$ and $\sqrt{\tau_j t_j}$,
\[
\left( \sum_{j=1}^{d} \sqrt{\alpha_j \tau_j} \right)^{2}
 = \left( \sum_{j=1}^{d} \sqrt{\frac{\alpha_j}{t_j}} \cdot \sqrt{\tau_j t_j} \right)^{2}
 \leq \left( \sum_{j=1}^{d} \frac{\alpha_j}{t_j} \right) \left( \sum_{j=1}^{d} \tau_j t_j \right)
 = \Lambda ,
\]
so $\mathrm{BC}(\alpha,\tau)^2$ is a lower bound. Equality in Cauchy--Schwarz holds exactly when the two vectors are proportional, that is $\alpha_j / t_j = \kappa\, \tau_j t_j$ for some $\kappa > 0$ and all $j$, giving $t_j = \kappa^{-1/2} (\alpha_j/\tau_j)^{1/2}$ and hence $g_j = t_j^{1/2} \propto \alpha_j^{1/4} \tau_j^{-1/4}$. This $g$ lies in $\mathbb{R}_{>0}^d$ because $\alpha_j$ and $\tau_j$ are positive by \cref{def:profiles}, so the bound is attained and \cref{eq:optimal-gain} follows.

(iv) Cauchy--Schwarz again, now on $\sqrt{\alpha_j}$ against $\sqrt{\tau_j}$, gives 
\[
\sum_j \sqrt{\alpha_j \tau_j} \leq \big( \sum_j \alpha_j \big)^{1/2} \big( \sum_j \tau_j \big)^{1/2} = 1,
\]
 so $\Lambda^\star \leq 1$. Equality requires $\sqrt{\alpha_j} = c \sqrt{\tau_j}$ for all $j$; squaring and summing forces $c = 1$, hence $\alpha = \tau$. Conversely $\alpha = \tau$ gives 
 \[
 \sum_j \sqrt{\alpha_j \tau_j} = \sum_j \alpha_j = 1.
 \]
\end{proof}

\begin{rem}
\label{rem:fairness}
Three points fix the standing of \cref{prop:ml-expressivity}(ii). First, the two bounds are produced by the identical argument, the same symmetrisation, the same contraction, the same evaluation at the representation budget, with the grading entering only through the data radius and the budget. The ratio $\Lambda(g)$ therefore compares like with like, and is not an artefact of one bound being derived more loosely than the other. Second, the confidence term $O(\sqrt{\log(1/\delta)/m})$ of the excess-risk bound is class-independent and common to both sides, so at any fixed $\delta$ it cancels from the comparison of sufficient sample sizes to leading order; the ratio statement concerns the complexity-dominated regime, which is the budget-active regime made precise in \cref{rem:budget-active}. Third, at its representation budget each class contains $f^*$ on the boundary of its constraint set, so empirical risk minimisation over either class is well posed and the bound of \cite{bartlett2002rademacher} applies to both without modification.
\end{rem}

\begin{rem}
\label{rem:gain-meaning}
\Cref{prop:ml-expressivity} identifies the entire benefit of grading with a single scalar: the affinity between where the target places its weight and where the data places its variance. Grading buys nothing when the two profiles agree, and its value grows as they separate. In this exact sense grading does not create information. It converts a correct prior about the target into a reduction in sample complexity at the computable rate \cref{eq:optimal-gain}, and an incorrect prior at the same rate in the opposite direction: any $g$ with $\Lambda(g) > 1$ is admissible and costs data.

\Cref{assum:profile-divergence} is exactly the assertion that these profiles diverge for language. Deep structural features are decisive for the target and rare in the corpus, with high $\alpha_j$ and low $\tau_j$, while lexical features carry the bulk of the variance and comparatively little of the decision. An isotropic architecture, which by construction sets $g = \mathbf{1}$ and $\Lambda = 1$, forgoes the gain \cref{eq:optimal-gain} whatever its size; \cref{prop:admissible-grades} locates it precisely, on the boundary of the set of grades at which the graded bound is the smaller one. Both profiles can be estimated from an annotated corpus without training a model, so \cref{assum:profile-divergence} is testable in advance of any experiment; \cref{subsec:benchmarks} includes this measurement in the program.
\end{rem}

\begin{prop}
\label{prop:admissible-grades}
Let $\alpha, \tau \in \Delta^{d-1}$ be the profiles of \cref{def:profiles}, related to the grades by $g_j = \lambda^{q_j}$, and write
\[
\Lambda(q) = \Big( \sum_{j=1}^{d} \alpha_j \lambda^{-2 q_j} \Big) \Big( \sum_{j=1}^{d} \tau_j \lambda^{2 q_j} \Big)
\]
for the gain of \cref{prop:ml-expressivity} expressed in the grades. Define the \emph{admissible set}
\[
\mathcal{Q}_+ = \left\{ q \in \mathbb{R}^d \;:\; \Lambda(q) < 1 \right\},
\]
the grades at which the graded bound is strictly smaller than the uniform one. Then:
\begin{enumerate}
  \item[(i)] \emph{(Gauge invariance)} $\Lambda(q + c\mathbf{1}) = \Lambda(q)$ for every $c \in \mathbb{R}$. Hence $\mathcal{Q}_+$ is invariant under translation along $\mathbf{1}$, and $\mathcal{Q}_+ \neq \emptyset$ implies $\mathcal{Q}_+ \cap \mathbb{Q}_{\geq 0}^{d} \neq \emptyset$: restricting the grades to non-negative rationals, as \cref{sec-3} does throughout, costs nothing.

  \item[(ii)] \emph{(Convexity)} $q \mapsto \log \Lambda(q)$ is convex on $\mathbb{R}^d$, and $\mathcal{Q}_+$ is open and convex.

  \item[(iii)] \emph{(Nonemptiness, and the position of the standard transformer)} $\Lambda(\mathbf{0}) = 1$, and $\mathcal{Q}_+ \neq \emptyset$ if and only if $\alpha \neq \tau$. When $\alpha \neq \tau$, the uniform grade $q = \mathbf{0}$ --- which by \cref{rem:absorption,rem:mlge-absorption} is the standard transformer --- lies on $\partial\mathcal{Q}_+$.

  \item[(iv)] \emph{(First-order criterion)} $\nabla_q \log \Lambda \big|_{q = \mathbf{0}} = 2 (\log \lambda)\, (\tau - \alpha)$, and for $v \in \mathbb{R}^d$ the ray $\{ \varepsilon v : \varepsilon > 0 \}$ meets $\mathcal{Q}_+$ if and only if

\begin{equation}
\label{eq:cone-condition}
\langle v, \alpha - \tau \rangle > 0 .
\end{equation}
In that case $\{ \varepsilon > 0 : \varepsilon v \in \mathcal{Q}_+ \}$ is an interval $(0, \varepsilon_{\max}(v))$.

  \item[(v)] \emph{(Gain along an admissible direction)} Let $v$ satisfy \cref{eq:cone-condition}, and set $D = \langle v, \alpha - \tau \rangle$ and $S = \mathrm{Var}_\alpha(v) + \mathrm{Var}_\tau(v)$, where $\mathrm{Var}_p(v) = \sum_j p_j v_j^2 - \big( \sum_j p_j v_j \big)^2$. Then $S > 0$ and
\[
\log \Lambda(\varepsilon v) = -2 (\log \lambda)\, D\, \varepsilon \;+\; 2 (\log \lambda)^2 S\, \varepsilon^2 \;+\; O(\varepsilon^3) ,
\]
so the quadratic truncation is minimised at $\varepsilon^\star = D / (2 S \log \lambda)$, with value $-D^2 / (2S)$ independent of $\lambda$.
  \item[(vi)] \emph{(Global optimum)} $\min_{q \in \mathbb{R}^d} \Lambda(q) = \mathrm{BC}(\alpha, \tau)^2$, attained exactly on the gauge orbit of the $q^\star$ of \cref{prop:ml-expressivity}(iii).
\end{enumerate}
\end{prop}

\begin{proof}
Set $u_j = 2 q_j \log \lambda$, an invertible linear change of variable since $\lambda > 1$, and write
\[
A(u) = \sum_{j=1}^{d} \alpha_j e^{-u_j}, \qquad T(u) = \sum_{j=1}^{d} \tau_j e^{u_j}, \qquad \Phi(u) = \log A(u) + \log T(u) = \log \Lambda .
\]
By \cref{def:profiles} both $\alpha$ and $\tau$ lie in $\Delta^{d-1}$ and are strictly positive. For $u \in \mathbb{R}^d$ define the tilted vectors
\[
\alpha^{(u)}_j = \frac{\alpha_j e^{-u_j}}{A(u)}, \qquad \tau^{(u)}_j = \frac{\tau_j e^{u_j}}{T(u)},
\]
both in $\Delta^{d-1}$, and for $p \in \Delta^{d-1}$ write $\mathrm{C}(p) = \operatorname{diag}(p) - p p^\top$, so that $v^\top \mathrm{C}(p)\, v = \mathrm{Var}_p(v) \geq 0$. Differentiating,
\[
\begin{split}
\nabla_u \Phi(u) &= \tau^{(u)} - \alpha^{(u)}, \\
\nabla^2_u \Phi(u) &= \mathrm{C}\big(\alpha^{(u)}\big) + \mathrm{C}\big(\tau^{(u)}\big),
\end{split}
\]
the term $\log A$ contributing $\mathrm{C}(\alpha^{(u)})$ to the Hessian because the two sign reversals arising from $u \mapsto -u$ cancel. At $u = \mathbf{0}$ we have $\alpha^{(\mathbf{0})} = \alpha$ and $\tau^{(\mathbf{0})} = \tau$.

(i) Under $q \mapsto q + c\mathbf{1}$ we get $u \mapsto u + c'\mathbf{1}$ with $c' = 2c \log \lambda$, so $A \mapsto e^{-c'} A$ and $T \mapsto e^{c'} T$, leaving the product $AT$ fixed. For the second claim, let $q \in \mathcal{Q}_+$ and set $q' = q - (\min_j q_j)\mathbf{1}$, so that $q' \geq 0$ and $\Lambda(q') = \Lambda(q) < 1$. Since $\Lambda$ is continuous and $\mathbb{Q}^d$ is dense in $\mathbb{R}^d$, some rational $q'' \geq 0$ near $q'$ satisfies $\Lambda(q'') < 1$.

(ii) $\nabla^2_u \Phi$ is a sum of two matrices of the form $\mathrm{C}(p)$ and is therefore positive semidefinite, so $\Phi$ is convex in $u$ and hence in $q$, the map $q \mapsto u$ being linear. Then $\mathcal{Q}_+ = \{ \Phi < 0 \}$ is a strict sublevel set of a convex function, hence convex, and open by continuity.

(iii) $\Lambda(\mathbf{0}) = \big( \sum_j \alpha_j \big)\big( \sum_j \tau_j \big) = 1$. By \cref{prop:ml-expressivity}(iii) and (iv), 
\[
\min_q \Lambda(q) = \mathrm{BC}(\alpha,\tau)^2 \leq 1
\]
 with equality if and only if $\alpha = \tau$; hence $\mathcal{Q}_+ \neq \emptyset$ exactly when $\alpha \neq \tau$. In that case $\mathbf{0} \notin \mathcal{Q}_+$, while by (iv) the direction $v = \alpha - \tau$ satisfies 
 \[
 \langle v, \alpha - \tau\rangle = \|\alpha - \tau\|_2^2 > 0,
 \]
  so every neighbourhood of $\mathbf{0}$ meets $\mathcal{Q}_+$; thus $\mathbf{0} \in \partial\mathcal{Q}_+$.

(iv) The gradient at $u = \mathbf{0}$ is $\tau - \alpha$, and 
\[
\nabla_q = 2 (\log \lambda) \nabla_u,
\]
 which gives the stated formula. Put $\phi(\varepsilon) = \log \Lambda(\varepsilon v)$, convex by (ii), with $\phi(0) = 0$ and 
\[
\phi'(0) = \langle v, \nabla_q \log \Lambda(\mathbf{0}) \rangle = -2 (\log \lambda) \langle v, \alpha - \tau \rangle.
\]
 If $\langle v, \alpha - \tau \rangle > 0$ then $\phi'(0) < 0$ and $\phi(\varepsilon) < 0$ for all small $\varepsilon > 0$. If $\langle v, \alpha - \tau \rangle \leq 0$ then $\phi'(0) \geq 0$, and convexity gives 
 \[
 \phi(\varepsilon) \geq \phi(0) + \phi'(0)\varepsilon \geq 0
 \]
  for every $\varepsilon > 0$, so the ray never meets $\mathcal{Q}_+$. This proves \cref{eq:cone-condition}. The set $\{\varepsilon > 0 : \phi(\varepsilon) < 0\}$ is the intersection of an interval, by convexity of $\phi$, with $(0,\infty)$, and contains all small $\varepsilon > 0$; it is therefore of the form $(0, \varepsilon_{\max}(v))$.

(v) That $S > 0$ follows from $D \neq 0$: since $\tau_j > 0$ for every $j$, $\mathrm{Var}_\tau(v) = 0$ forces $v \in \mathbb{R}\mathbf{1}$, and then $D = \langle v, \alpha - \tau \rangle = 0$ because $\alpha$ and $\tau$ are both probability vectors. Taylor expansion of $\phi$ at $0$ gives 
\[
\phi(\varepsilon) = \phi'(0)\varepsilon + \tfrac12 \phi''(0)\varepsilon^2 + O(\varepsilon^3)
\]
 with $\phi'(0) = -2(\log \lambda) D$ and
\[
\phi''(0) = v^\top \nabla_q^2 \log \Lambda(\mathbf{0})\, v = (2\log\lambda)^2\, v^\top \big[ \mathrm{C}(\alpha) + \mathrm{C}(\tau) \big] v = 4 (\log \lambda)^2 S ,
\]
which is the stated expansion. Minimising $a\varepsilon^2 + b\varepsilon$ with $a = 2(\log\lambda)^2 S > 0$ and $b = -2(\log\lambda) D$ gives 
\[
\varepsilon^\star = -b/(2a) = D/(2 S \log \lambda)
\]
 and minimum value $-b^2/(4a) = -D^2/(2S)$, in which $\lambda$ cancels.

(vi) The value and the minimizer are \cref{prop:ml-expressivity}(iii), transported through $g_j = \lambda^{q_j}$; the gauge orbit is the fibre of the scaling invariance recorded in (i).
\end{proof}

\begin{rem}
\label{rem:cone-reading}
Four consequences are worth noting.

First, \cref{eq:cone-condition} is the condition on the weights that the framework has needed. Grading improves the bound, to first order, exactly when the grades correlate positively with $\alpha - \tau$: high grades on coordinates where the target's energy exceeds the data's variance, low grades where it does not. \Cref{assum:profile-divergence} thereby becomes an inner product one can compute.

Second, the standard transformer is not a local optimum of the sample-complexity landscape. It sits at $\Lambda = 1$ on the boundary of $\mathcal{Q}_+$ with gradient $2(\log \lambda)(\tau - \alpha) \neq 0$ whenever $\alpha \neq \tau$, so a strictly descending direction exists at $q = \mathbf{0}$ and \cref{eq:cone-condition} exhibits it. The isotropic architecture forgoes the gain of \cref{eq:optimal-gain} not because the gain is unavailable at its location, since the gradient there is nonzero, but because the gauge orbit of $q = \mathbf{0}$ is the unique set of grades at which $\Lambda = 1$ for \emph{every} pair of profiles. It is the choice that ignores the profiles.

Third, only the \emph{sign} of $\langle q, \hat\alpha - \hat\tau \rangle$ need survive estimation error, not the magnitudes of $\hat\alpha$ and $\hat\tau$. This is what makes the coarse part-of-speech initialisation of \cref{subsubsec:grade-init} defensible: it does not attempt to locate $q^\star$, only to enter $\mathcal{Q}_+$.

Fourth, by (ii) the offline selection of grades from estimated profiles is the minimisation of a convex function with no spurious stationary points, and by (i) the non-negativity and rationality demanded of the grades throughout this paper are free. The selection problem, on the full space and on the clipped region the configurations admit, is solved in \cref{cor:grade-convexity} and \cref{prop:clipped-selection} respectively.

$\Lambda$ is a ratio of upper bounds, and \cref{prop:admissible-grades} characterises when the graded bound beats the uniform bound. In the level-stratified regime of \cref{cor:stratified} the comparison is a separation between the graded prior and its absence: \cref{thm:stratified-separation} establishes the two-sided minimax ratio there, over all estimators. For general profiles the corresponding statement is \cref{con:minimax-separation}, and \cref{rem:sample-complexity-scope} records the scope.
\end{rem}

\begin{cor}
\label{cor:stratified}
Suppose both profiles are uniform within levels: $\alpha_j = a_l / d_l$ and $\tau_j = s_l / d_l$ for $j \in V_l$, with $a, s \in \Delta^{L-1}$. Then
\[
\Lambda^\star = \mathrm{BC}(a,s)^2 = \left( \sum_{l=0}^{L-1} \sqrt{a_l s_l} \right)^{2},
\]
independently of the allocation $\{d_l\}$, and the optimal grade is constant on each level with
\[
q_l^\star \, \log \lambda = \tfrac{1}{4} \log a_l - \tfrac{1}{4} \log s_l + \mathrm{const}.
\]
If moreover the target concentrates geometrically towards the deepest level and the data decays geometrically away from the shallowest, $a_l \propto \nu^{L-1-l}$ and $s_l \propto \mu^{l}$ with $0 < \mu < \nu < 1$, then $q_l^\star$ is affine in $l$ with positive slope, and
\[
\Lambda^\star = \Theta\!\left( \nu^{L} \right)
\qquad (L \to \infty).
\]
\end{cor}

\begin{proof}
For the first claim, $\sum_{j \in V_l} \sqrt{\alpha_j \tau_j} = d_l \cdot \sqrt{(a_l/d_l)(s_l/d_l)} = \sqrt{a_l s_l}$, and summing over $l$ gives $\mathrm{BC}(\alpha,\tau) = \mathrm{BC}(a,s)$; the $d_l$ cancel. The grade formula is $\log g_j^\star = \tfrac14 \log \alpha_j - \tfrac14 \log \tau_j + \mathrm{const}$ from \cref{prop:ml-expressivity}(iii), in which the $\log d_l$ terms cancel between $\alpha_j$ and $\tau_j$.

For the second, write $Z_a = \sum_{l=0}^{L-1} \nu^{L-1-l} = (1-\nu^L)/(1-\nu)$ and $Z_s = (1-\mu^L)/(1-\mu)$. Then $\log a_l - \log s_l = -l(\log \nu + \log \mu) + \mathrm{const}$, which is affine in $l$ with positive coefficient since $\log \nu, \log \mu < 0$; the grade claim follows. Setting $r = \sqrt{\mu/\nu} \in (0, 1)$,
\[
\mathrm{BC}(a,s)
 = \frac{1}{\sqrt{Z_a Z_s}} \sum_{l=0}^{L-1} \nu^{(L-1-l)/2} \mu^{l/2}
 = \frac{\nu^{(L-1)/2}}{\sqrt{Z_a Z_s}} \cdot \frac{1 - r^{L}}{1 - r} .
\]
As $L \to \infty$ we have $Z_a \to (1-\nu)^{-1}$, $Z_s \to (1-\mu)^{-1}$ and $r^L \to 0$, so
\[
\mathrm{BC}(a,s) \sim \nu^{(L-1)/2} \cdot \frac{\sqrt{(1-\nu)(1-\mu)}}{1 - \sqrt{\mu/\nu}} ,
\]
and squaring gives $\Lambda^\star = \Theta(\nu^L)$.
\end{proof}

\begin{rem}
\label{rem:stratified-reach}
\Cref{cor:stratified} is an unconstrained statement: the optimal grade is affine in $l$, so its dynamic range grows linearly in $L$, and realising $\Lambda^\star = \Theta(\nu^L)$ requires grades that exceed any fixed clip once $L$ is large. At the clipped configurations of \cref{subsubsec:configs} the attainable gain is bounded by $C^2$ (\cref{prop:clipped-selection}), and the exponential regime is reached as the clip opens (\cref{rem:one-dial}). The corollary quantifies that trajectory; it does not describe the planned configurations.
\end{rem}

\begin{rem}
\label{rem:cone-stratified}
Under the hypotheses of \cref{cor:stratified} the criterion \cref{eq:cone-condition} reduces to a condition on $L$ numbers rather than $d$: a level-constant grade $v_j = v_l$ for $j \in V_l$ is admissible exactly when $\sum_{l=0}^{L-1} v_l (a_l - s_l) > 0$. Under the geometric profiles this holds for every $v$ affine and increasing in $l$, so the linguistic ordering itself, with deeper levels receiving higher grades, is enough to enter $\mathcal{Q}_+$ without any estimate of $\nu$ or $\mu$.
\end{rem}

\begin{lem}
\label{lem:plugin-stability}
Let $\hat\alpha, \hat\tau \in \Delta^{d-1}$ be strictly positive estimates of the profiles of \cref{def:profiles}, and let $\hat g$ be the plug-in optimal grades $\hat g_j \propto \hat\alpha_j^{1/4}\, \hat\tau_j^{-1/4}$ of \cref{prop:ml-expressivity}(iii). Set
\[
\delta_j = \log\frac{\hat\alpha_j}{\alpha_j} - \log\frac{\hat\tau_j}{\tau_j},
\qquad
\bar\delta = \frac{1}{d} \sum_{j=1}^{d} \delta_j .
\]
Then
\[
\log \Lambda(\hat g) \;\leq\; \log \mathrm{BC}(\alpha,\tau)^{2} \;+\; \frac{1}{4} \sum_{j=1}^{d} \big( \delta_j - \bar\delta \big)^{2} .
\]
The attained gain is therefore insensitive to first order to estimation error in the profiles: the degradation is quadratic in the centred log-errors, and vanishes at rate the square of the profile error rather than the profile error itself.
\end{lem}

\begin{proof}
Work in the variables $u_j = 2 q_j \log \lambda = 2 \log g_j$ of \cref{prop:admissible-grades} and write $\Phi(u) = \log \Lambda$. The optimal grades of \cref{prop:ml-expressivity}(iii) correspond to $u^\star_j = \tfrac12 \log(\alpha_j/\tau_j) + c^\star$, and at $u^\star$ the tilted profiles of \cref{prop:admissible-grades} coincide,
\[
\alpha^{(u^\star)}_j = \frac{\sqrt{\alpha_j \tau_j}}{\mathrm{BC}(\alpha,\tau)} = \tau^{(u^\star)}_j ,
\]
so $\nabla_u \Phi(u^\star) = \tau^{(u^\star)} - \alpha^{(u^\star)} = 0$: the optimum is a stationary point. For every $u$ the Hessian satisfies
\[
\nabla^2_u \Phi(u) = \mathrm{C}\big(\alpha^{(u)}\big) + \mathrm{C}\big(\tau^{(u)}\big) \preceq 2 I,
\]
since for any $p \in \Delta^{d-1}$ one has $\mathrm{C}(p) = \operatorname{diag}(p) - p p^\top \preceq \operatorname{diag}(p) \preceq I$. Taylor expansion with integral remainder along the segment from $u^\star$ to $\hat u$, using the vanishing gradient and the uniform Hessian bound, gives
\[
\Phi(\hat u) \;\leq\; \Phi(u^\star) + \| \hat u - u^\star \|_2^2 .
\]
The plug-in grades correspond to $\hat u_j = \tfrac12 \log(\hat\alpha_j/\hat\tau_j) + \hat c$, so $\hat u_j - u^\star_j = \tfrac12 \delta_j + \mathrm{const}$, and by the gauge invariance of \cref{prop:admissible-grades}(i) the representative $u^\star$ on the optimal orbit may be chosen to minimise the distance, giving
\[
\min_{c \in \mathbb{R}} \Big\| \tfrac12 \delta + c \mathbf{1} \Big\|_2^2 = \frac{1}{4} \sum_{j=1}^{d} \big( \delta_j - \bar\delta \big)^2 .
\]
Since $\Phi(u^\star) = \log \mathrm{BC}(\alpha,\tau)^2$ by \cref{prop:admissible-grades}(vi), the claim follows.
\end{proof}

\begin{rem}
\label{rem:sample-splitting}
The program estimates the profiles on an annotated corpus disjoint from the pretraining sample, and this disjointness is what closes the adaptivity question in the present setting. Conditional on $\hat g$, the grades are a fixed deterministic quantity with respect to the training sample, so \cref{prop:ml-expressivity}(ii) applies verbatim with $\Lambda(\hat g)$ in place of $\Lambda(g)$ and no selection term arises. The whole cost of not knowing $\alpha$ is therefore the degradation \cref{lem:plugin-stability} bounds, and that degradation is quadratic in the profile error because the gradient of the gain vanishes at the optimum: the framework is protected at exactly the point where protection is needed. Certification asks for less still, since by \cref{rem:cone-reading} only the sign of one inner product must survive the estimation.
\end{rem}

\begin{rem}
\label{rem:sample-complexity-scope}
Four limitations bound what \cref{prop:ml-expressivity} establishes; two of them are resolved in this manuscript, one of them in the regime the program occupies, and the remaining ones name the tool their repair would require.

First, the analysis is for the graded embedding composed with a linear read-out. The multi-layer statement requires propagating the argument through a layerwise Lipschitz analysis of the graded stack, and the resulting bound compounds across depth; the extension is open.

Second, both sides of (ii) are \emph{upper} bounds, and a ratio of upper bounds becomes a separation only when a matching minimax lower bound over the uniform target class meets it. \Cref{thm:stratified-separation} supplies that bound in the level-stratified regime of \cref{cor:stratified}: at squared loss and over all estimators, the ratio of minimax risks is $\Theta_{\mu,\nu}(\Lambda(g))$ throughout the window of \cref{lem:window}, so in the regime the planned configurations occupy the separation between the graded prior and its absence is established. The ellipsoidal constraint geometry of \cref{prop:expressivity} is the packing set on which that bound is posed. What remains open is the general-profile, Lipschitz-loss case, stated as \cref{con:minimax-separation} in the regime fixed by \cref{rem:budget-active}.

Third, $g^\star$ depends on $\alpha$, which is unknown. \Cref{lem:plugin-stability} and \cref{rem:sample-splitting} resolve this for the present setting: profiles estimated on data disjoint from the training sample incur no adaptivity cost beyond the quadratic degradation of the lemma, which vanishes to first order at the optimum. What remains open is only the fully adaptive variant in which grades and read-out are fit on the same sample.

Fourth, \cref{cor:stratified} assumes geometric profiles. That assumption is an empirical refinement of \cref{assum:profile-divergence}, not a theorem, and it is the assumption on which every exponential-in-$L$ statement in this work rests; \cref{rem:stratified-reach} bounds what it delivers at fixed clip.
\end{rem}

\begin{con}
\label{con:minimax-separation}
Let $x \sim N\big(0, \operatorname{diag}(\sigma_1^2, \dots, \sigma_d^2)\big)$ and $y = f^*(x) + \xi$ with $\xi \sim N(0, \sigma_\xi^2)$ independent of $x$, in the setting of \cref{def:profiles}, and let $\ell$ be a loss that is $L_\ell$-Lipschitz in its first argument, as in \cref{prop:ml-expressivity}; write $\mathcal{E}(\hat f, f^*) = \mathbb{E}\big[\ell(\hat f(x), y)\big] - \mathbb{E}\big[\ell(f^*(x), y)\big]$ for the excess risk. Fix a target $f^*$ with profile $\alpha$ and let $g^\star$ be the optimal grades of \cref{prop:ml-expressivity}(iii). Define the target classes carrying no prior and the graded prior respectively,
\[
\begin{split}
\mathcal{F}^{\mathrm{unif}} &= \left\{ \langle w, \cdot \rangle \;:\; \|w\|_2 \leq B^{\mathrm{unif}}(f^*) \right\}, \\
\mathcal{F}^{\mathrm{graded}} &= \left\{ \langle w, \cdot \rangle \;:\; \big\| G^{-1} w \big\|_2 \leq B^{\mathrm{graded}}(f^*) \right\},
\end{split}
\]
with $G$ at grades $g^\star$ and both budgets calibrated at $f^*$ as in \cref{def:classes-at-target}, and the minimax excess risk of a target class $\mathcal{F}$ at sample size $m$, over \emph{all} estimators,
\[
\mathcal{R}_m(\mathcal{F}) = \inf_{\hat f} \, \sup_{f^* \in \mathcal{F}} \, \mathbb{E}\big[ \mathcal{E}(\hat f, f^*) \big],
\qquad
m^{\mathcal{F}}(\varepsilon) = \min\left\{ m \;:\; \mathcal{R}_m(\mathcal{F}) \leq \varepsilon \right\}.
\]
Then there exist universal constants $0 < c_1 \leq c_2$ such that, for every $\varepsilon$ in the budget-active window of \cref{rem:budget-active},
\[
c_1 \cdot \mathrm{BC}(\alpha, \tau)^{2}
\;\leq\;
\frac{m^{\mathcal{F}^{\mathrm{graded}}}(\varepsilon)}{m^{\mathcal{F}^{\mathrm{unif}}}(\varepsilon)}
\;\leq\;
c_2 \cdot \mathrm{BC}(\alpha, \tau)^{2} ,
\]
so the gain \cref{eq:optimal-gain} is attained up to universal constants and is a genuine separation between the graded prior and its absence, not an artefact of the upper bounds of \cref{prop:ml-expressivity}.
\end{con}

\Cref{con:minimax-separation} is stated for general profiles and Lipschitz losses. In the level-stratified regime of \cref{cor:stratified}, at squared loss, it is a theorem: \cref{thm:stratified-separation} establishes the two-sided ratio, in risk form at each sample size of the window of \cref{lem:window}, and that regime is the one every planned configuration of \cref{subsubsec:configs} occupies. What the conjecture asserts beyond the theorem is the general case, and the remark below fixes the regime in which it can hold.

\begin{rem}
\label{rem:budget-active}
Two features of the statement are forced, and each replaces a formulation that cannot hold. First, the infimum runs over all estimators and the two classes enter as classes of \emph{targets}: by \cref{rem:absorption,rem:mlge-absorption} the graded and uniform hypothesis classes coincide as sets of functions, so a comparison that restricted the estimator would compare a quantity with itself. What differs between the two sides is what is known about the target in advance, which is exactly the content of a prior, and the conjecture quantifies that prior. Second, the comparison cannot be asymptotic in $\varepsilon \to 0$ at fixed $d$: below the parametric floor of order $\sigma_\xi^2 d / m$ the norm constraint ceases to bind, both minimax problems reduce to the same unconstrained $d$-dimensional one, and the ratio tends to $1$ for \emph{every} pair of classes with nonempty interior. Each class has a budget-active interval of $\varepsilon$ on which its minimax risk exceeds its parametric floor by a fixed factor, equivalently sample sizes below the saturation scale at which its constraint stops binding, and the budget-active window of the conjecture is the intersection of the two intervals, since the two-sided ratio requires both constraints to bind. Only in that window is the geometry of \cref{prop:expressivity} visible in the rates. In the stratified regime the window is no longer a description but a set: \cref{lem:window} exhibits it as $W$, proves it nonempty at $d \geq 3L/\sigma_\xi$, and proves that in the joint scaling $d, m \to \infty$ it widens rather than closes. For general profiles the same two requirements, nonemptiness under \cref{assum:profile-divergence} and widening under the joint scaling, are the conditions the conjecture's window must satisfy.
\end{rem}

The route \cref{thm:stratified-separation} takes is the one the general case must extend. The lower bound is an Assouad packing of sign patterns with the amplitude allocation chosen to hold the per-coordinate testing error constant, run in the anisotropic Gaussian design directly and converted into the bound by Fano's inequality (\cref{lem:envelope}); the window on which both budgets bind is exhibited and sized by \cref{lem:window}; and the upper bound in the matching form is carried by the envelope rather than imported from \cref{prop:ml-expressivity}(ii). Extending the argument beyond level-homogeneous designs and squared loss is what the general case requires, and \cref{subsec:settling} records it as the remaining theoretical target alongside the layerwise extension.


\subsection{The Minimax Envelope}\label{subsec:envelope}

Throughout this subsection the design is Gaussian and level-homogeneous: $x \sim N(0, \Sigma)$ with $\Sigma = \operatorname{diag}(\sigma_1^2, \dots, \sigma_d^2)$ constant on levels, $\sigma_j^2 = \sigma_{(l)}^2$ for $j \in V_l$ and $\dim V_l = d_l$, and $y = \langle w^*, x \rangle + \xi$ with $\xi \sim N(0, \sigma_\xi^2)$ independent of $x$. The loss is squared, so the excess risk of a predictor $\hat f$ at the target $f^* = \langle w^*, \cdot \rangle$ is $\| \hat f - f^* \|_{L^2(P_x)}^2$, and for a class $\mathcal{F}$ of targets the minimax excess risk at sample size $m$ is $\mathcal{R}_m(\mathcal{F}) = \inf_{\hat f} \sup_{f^* \in \mathcal{F}} \mathbb{E} \, \| \hat f - f^* \|_{L^2(P_x)}^2$, the infimum over all estimators, as in \cref{con:minimax-separation}. For a diagonal matrix $D = \operatorname{diag}(D_l I_{d_l})$ with $D_l > 0$ and a budget $B > 0$ set
\[
\mathcal{F}(B, D) = \bigl\{ \langle w, \cdot \rangle : w \in \mathbb{R}^d, \ \| D w \|_2 \le B \bigr\},
\qquad
\operatorname{cap}_l = \frac{\sigma_{(l)}^2 B^2}{D_l^2},
\qquad
\varepsilon^2 = \frac{\sigma_\xi^2}{m}.
\]
The choice $D = I_d$, $B = B^{\mathrm{unif}}(f^*)$ realises the uniform target class, and $D_l = g_l^{-1}$, $B = B^{\mathrm{graded}}(f^*)$ the graded one; $\operatorname{cap}_l$ is the prediction-norm energy the budget can place at level $l$ if spent there entirely, and the two classes differ only through their caps.

\begin{lem}
\label{lem:envelope}
Let $S = \{ l : d_l \varepsilon^2 \le \operatorname{cap}_l \}$. There exist constants $0 < c(L) \le C(L)$ depending only on $L$ such that:
\begin{enumerate}
\item[(i)] for every $m \ge 1$,
\[
\mathcal{R}_m\bigl(\mathcal{F}(B, D)\bigr) \ \ge\ c(L) \sum_{l=0}^{L-1} \min\bigl\{ \operatorname{cap}_l, \ d_l \varepsilon^2 \bigr\};
\]
\item[(ii)] for every $m \ge 2 \sum_{l \in S} d_l + 4$,
\[
\mathcal{R}_m\bigl(\mathcal{F}(B, D)\bigr) \ \le\ C(L) \sum_{l=0}^{L-1} \min\bigl\{ \operatorname{cap}_l, \ d_l \varepsilon^2 \bigr\}.
\]
\end{enumerate}
\end{lem}

\begin{proof}
Write $\mu_l = \min\{ \operatorname{cap}_l, d_l \varepsilon^2 \}$ and $M = \sum_l \mu_l$.

(i) Fix $c_0 = (\log 2)/64$ and set, for each level, $k_l = \min\bigl\{ d_l, \ \lfloor \operatorname{cap}_l / (c_0 L \varepsilon^2) \rfloor \bigr\}$ and $\delta_l^2 = c_0 \varepsilon^2 / \sigma_{(l)}^2$, and let $\mathcal{W} \subseteq \mathbb{R}^d$ consist of the vectors supported on the first $k_l$ coordinates of each $V_l$ with entries $\pm \delta_l$. Every $w \in \mathcal{W}$ is feasible:
\[
\| D w \|_2^2 = \sum_{l} k_l D_l^2 \delta_l^2 = c_0 \varepsilon^2 \sum_{l} k_l \frac{B^2}{\operatorname{cap}_l} \le \frac{B^2}{L} \sum_{l} 1 = B^2 .
\]
For $w, w' \in \mathcal{W}$ with Hamming distance $\rho(w, w')$ between their sign patterns, the level-homogeneity of $\Sigma$ and the constant-testing-error allocation $\sigma_{(l)}^2 \delta_l^2 = c_0 \varepsilon^2$ give
\[
\begin{split}
\| \Sigma^{1/2} (w - w') \|_2^2 &= \sum_{l} \sigma_{(l)}^2 \, (2 \delta_l)^2 \, \rho_l(w, w') = 4 c_0 \varepsilon^2 \, \rho(w, w'),
\\
\mathrm{KL}\bigl( P_w^{\otimes m} \,\big\|\, P_{w'}^{\otimes m} \bigr) &= \frac{m}{2 \sigma_\xi^2} \, \| \Sigma^{1/2} (w - w') \|_2^2 \le 2 c_0 K, \qquad K := \sum_{l} k_l ,
\end{split}
\]
so the weighted packing problem collapses to an unweighted one on $\{ \pm 1 \}^K$. Suppose first $K \ge 32$. By the Varshamov--Gilbert bound there is $\mathcal{W}' \subseteq \mathcal{W}$ with $\log |\mathcal{W}'| \ge (K/8) \log 2$ and pairwise Hamming distance at least $K/8$, hence pairwise $L^2(P_x)$-separation at least $c_0 K \varepsilon^2 / 2$. Fano's inequality then yields
\[
\mathcal{R}_m \ \ge\ \frac{c_0 K \varepsilon^2}{8} \left( 1 - \frac{2 c_0 K + \log 2}{(K/8) \log 2} \right) \ \ge\ \frac{c_0}{16} \, K \varepsilon^2 ,
\]
the last step by $2 c_0 K = (K/32) \log 2$ and $\log 2 \le (K/32) \log 2$ at $K \ge 32$. When $\operatorname{cap}_l \ge c_0 L \varepsilon^2$ the floor in $k_l$ is at least half its argument, so 
\[
k_l \varepsilon^2 \ge \tfrac{1}{2} \min\bigl\{ d_l \varepsilon^2, \ \operatorname{cap}_l / (c_0 L) \bigr\} \ge \mu_l / (2 \max\{1, c_0 L\});
\]
 when $\operatorname{cap}_l < c_0 L \varepsilon^2$ the level contributes $\mu_l \le \operatorname{cap}_l < c_0 L \varepsilon^2 \le (c_0 L / 32) \, K \varepsilon^2$. Summing over the two kinds of level, 
 \[
 M \le \bigl( 2 \max\{1, c_0 L\} + c_0 L^2 / 32 \bigr) K \varepsilon^2 =: C_1(L) \, K \varepsilon^2,
 \]
  so $\mathcal{R}_m \ge \bigl( c_0 / (16 \, C_1(L)) \bigr) M$ and (i) follows in this case. 
  
  If $K < 32$, then for every $l$ either $d_l < 32$ or $\operatorname{cap}_l < 32 c_0 L \varepsilon^2$, so 
  \[
  \mu_l \le 32 \max\{1, c_0 L\} \min\{ \operatorname{cap}_l, \varepsilon^2 \}
  \]
   for every $l$ and $M \le C'(L) \max_l \min\{ \operatorname{cap}_l, \varepsilon^2 \}$ with $C'(L) = 32 L \max\{1, c_0 L\}$; a two-point argument with $w = 0$ against $w = \delta e_j$ for $j \in V_{l^*}$ at the maximising level, with $\sigma_{(l^*)}^2 \delta^2 = \min\{ \operatorname{cap}_{l^*}, c_0 \varepsilon^2 \}$, is feasible, has $\mathrm{KL} \le c_0 / 2$ over the $m$ samples, and by Le Cam's two-point bound gives 
   \[
   \mathcal{R}_m \ge c_5 \min\{ \operatorname{cap}_{l^*}, c_0 \varepsilon^2 \} \ge c_5 c_0 \min\{ \operatorname{cap}_{l^*}, \varepsilon^2 \} \ge \bigl( c_5 c_0 / C'(L) \bigr) M,
   \]
    with $c_5$ an absolute constant.

(ii) Let $V_S = \bigoplus_{l \in S} V_l$ with $D_S = \dim V_S = \sum_{l \in S} d_l$, and let $\hat f$ be ordinary least squares on the coordinates of $V_S$, extended by zero on the remaining levels. Since $\Sigma$ is diagonal, the omitted covariates are independent of the retained ones, so conditional on the retained design the omitted signal $\langle w^*_{S^c}, x_{S^c} \rangle$ is centred noise of variance $\sum_{l \notin S} \sigma_{(l)}^2 \| w^*_l \|_2^2$, independent across observations, and the regression of $y$ on $x_S$ has noise variance $\sigma_\xi^2 + \sum_{l \notin S} \sigma_{(l)}^2 \| w^*_l \|_2^2$. The excess risk decomposes into the approximation term $\sum_{l \notin S} \sigma_{(l)}^2 \| w^*_l \|_2^2$ and the estimation term, which for Gaussian design and $m \ge 2 D_S + 4$ is at most $2 \bigl( \sigma_\xi^2 + \sum_{l \notin S} \sigma_{(l)}^2 \| w^*_l \|_2^2 \bigr) D_S / m$. The budget constraint gives $\sigma_{(l)}^2 \| w^*_l \|_2^2 \le \operatorname{cap}_l$ uniformly over the class, and $D_S / m \le 1/2$, so
\[
\mathcal{R}_m \ \le\ \sum_{l \notin S} \operatorname{cap}_l \ +\ 2 \, D_S \, \varepsilon^2 \ +\ \sum_{l \notin S} \operatorname{cap}_l \ \le\ 3 \sum_{l=0}^{L-1} \min\bigl\{ \operatorname{cap}_l, \ d_l \varepsilon^2 \bigr\},
\]
using $\operatorname{cap}_l = \mu_l$ for $l \notin S$ and $d_l \varepsilon^2 = \mu_l$ for $l \in S$.
\end{proof}

\begin{rem}
\label{rem:envelope-regime}
The lower bound of \cref{lem:envelope} holds at every sample size; the upper bound is established for $m \ge 2 \sum_{l \in S} d_l + 4$, the regime in which every level worth estimating is estimable. The band $m \asymp \sum_{l \in S} d_l$, in which a level crosses its estimability threshold and its worth-estimating threshold at comparable sample sizes, is not covered by the argument above, 
and the window on which the separation of \cref{thm:stratified-separation} is established must be located outside it; \cref{lem:window} carries this out.
\end{rem}

Throughout the remainder of this subsection the profiles are geometric and the allocation uniform: $a_l = \nu^{L-1-l}/Z_a$ and $s_l = \mu^l/Z_s$ with $0 < \mu < \nu < 1$ and $Z_a = \sum_{k=0}^{L-1}\nu^k$, $Z_s = \sum_{k=0}^{L-1}\mu^k$; $d_l = d/L$ for every $l$; and the calibrations $\|w^*\|_2 = 1$, $\operatorname{tr}\Sigma = 1$, so that $\sigma_{(l)}^2 = (L/d)\,s_l$. The grade vector is level-constant, $g = (g_0,\dots,g_{L-1})$, and is required to satisfy
\[
1 \le g_0 \le g_1 \le \cdots \le g_{L-1}, \qquad \frac{g_{l+1}}{g_l} \le (\nu\mu)^{-1/4} \quad (0 \le l \le L-2),
\]
a condition met by the uniform grades $g = \mathbf{1}$, by the fourth-root grades of \cref{prop:ml-expressivity}(iii) evaluated on the geometric profiles, and by their clipped form $g^\star_C$ of \cref{prop:clipped-selection}, since clipping only decreases consecutive ratios. Under these conventions the caps of \cref{lem:envelope} are
\[
\operatorname{cap}_l^{\mathrm{unif}} = \frac{L}{d}\, s_l, \qquad \operatorname{cap}_l^{\mathrm{gr}} = \frac{L}{d}\, s_l\, g_l^2 \sum_{k=0}^{L-1} a_k g_k^{-2}, \qquad \varepsilon^2 = \frac{\sigma_\xi^2}{m}.
\]

\begin{lem}
\label{lem:window}
Set $r = \sqrt{\mu/\nu} \in (0,1)$,
\[
m_{\mathrm{sat}} = \frac{d_0\,\sigma_\xi^2}{\operatorname{cap}_0^{\mathrm{unif}}} = \frac{d^2\,\sigma_\xi^2}{L^2\, s_0}, \qquad W = \bigl\{ m \in \mathbb{N} : 4 \le m \le \tfrac{1}{2}\, m_{\mathrm{sat}} \bigr\}.
\]
Then there is an absolute constant $c_4 \in (0,1)$ such that:
\begin{enumerate}
\item[(i)] (Non-emptiness and width.) $W \neq \emptyset$ whenever $d \ge 3L/\sigma_\xi$, and $|W| = \Theta\bigl(d^2\sigma_\xi^2/L^2\bigr)$ as $d/L \to \infty$ at fixed $(\mu,\nu,\sigma_\xi)$; 
in the joint scaling $d, m \to \infty$ the window widens, as \cref{rem:budget-active} requires of the regime.

\item[(ii)] (Full saturation.) For every $m \in W$, every level $l$, and both classes, $\operatorname{cap}_l \le \tfrac{1}{2}\, d_l\, \varepsilon^2$; in particular the set $S$ of \cref{lem:envelope} is empty for both classes and the regime restriction of \cref{rem:envelope-regime} is vacuous on $W$.
\item[(iii)] (Risk pinning.) For each class and every $m \in W$,
\[
c_4 \, \max_{l} \operatorname{cap}_l \;\le\; \mathcal{R}_m\bigl(\mathcal{F}(B,D)\bigr) \;\le\; \max_{l} \operatorname{cap}_l,
\]
the upper bound attained by the zero estimator, whose worst-case risk over the ellipsoid $\{\|Dw\|_2 \le B\}$ is exactly $\max_l \operatorname{cap}_l$.
\item[(iv)] (Ratio.) Under the grade condition both cap sequences are geometric with ratio at most $r$, both maxima sit at $l = 0$, and uniformly for $m \in W$,
\[
c_4\,(1-r)\; \Lambda(g) \;\le\; \frac{\mathcal{R}_m\bigl(\mathcal{F}^{\mathrm{gr}}\bigr)}{\mathcal{R}_m\bigl(\mathcal{F}^{\mathrm{unif}}\bigr)} \;\le\; \frac{1}{c_4\,(1-\mu)}\; \Lambda(g),
\]
with $\Lambda(g) = \bigl(\sum_l a_l g_l^{-2}\bigr)\bigl(\sum_l s_l g_l^2\bigr)$ the gain of \cref{prop:ml-expressivity}. The constants depend only on $(\mu,\nu)$ and are uniform in $L$, $d$, and $m \in W$. At the fourth-root grades the ratio is $\Theta_{\mu,\nu}\bigl(\mathrm{BC}(a,s)^2\bigr) = \Theta_{\mu,\nu}(\nu^L)$, and at $g^\star_C$ it is $\Theta_{\mu,\nu}(\Lambda^\star_C)$.
\item[(v)] (Parametric floor.) For $m \ge \max\bigl\{2d+4,\; d_0\sigma_\xi^2/\min_l \operatorname{cap}_l\bigr\}$ the set $S$ is full for both classes, both risks are $\Theta_L(d\,\varepsilon^2)$ by \cref{lem:envelope}, and the ratio is $\Theta_L(1)$: the graded gain concentrates on the window $W$, where the budget is active. On the intermediate band $\bigl(\tfrac{1}{2}m_{\mathrm{sat}},\, d_0\sigma_\xi^2/\min_l\operatorname{cap}_l\bigr)$ no two-sided comparison against $\Lambda(g)$ is asserted.
\end{enumerate}
\end{lem}

\begin{proof}
Since $g$ is nondecreasing, $g_0^2 \sum_k a_k g_k^{-2} = \sum_k a_k (g_0/g_k)^2 \le \sum_k a_k = 1$, so
\[
\max_l \operatorname{cap}_l^{\mathrm{gr}} \;=\; \operatorname{cap}_0^{\mathrm{gr}} \;=\; \frac{L}{d}\, s_0\, g_0^2 \sum_k a_k g_k^{-2} \;\le\; \frac{L}{d}\, s_0 \;=\; \operatorname{cap}_0^{\mathrm{unif}},
\]
where the first equality holds because $s_{l+1}g_{l+1}^2 / (s_l g_l^2) = \mu\, (g_{l+1}/g_l)^2 \le \mu\,(\nu\mu)^{-1/2} = r < 1$: both cap sequences are geometric with ratio at most $r$ (ratio $\mu \le r$ in the uniform case) and are maximised at $l = 0$.

(ii) For $m \le \tfrac{1}{2} m_{\mathrm{sat}}$ we have $\operatorname{cap}_0^{\mathrm{unif}} \le \tfrac{1}{2} d_0 \varepsilon^2$ by the definition of $m_{\mathrm{sat}}$, and every other cap of either class is at most $\operatorname{cap}_0^{\mathrm{unif}}$ by the display above, while $d_l = d_0$ for every $l$.

(i) $W \neq \emptyset$ requires $m_{\mathrm{sat}} \ge 8$, that is $d^2 \ge 8 L^2 s_0 / \sigma_\xi^2$, and $s_0 = (1-\mu)/(1-\mu^L) \le 1$ makes $d \ge 3L/\sigma_\xi$ sufficient. The width claim is immediate from the formula for $m_{\mathrm{sat}}$.

(iii) Upper bound: the zero estimator has excess risk $\|\Sigma^{1/2} w\|_2^2 = \sum_j \sigma_j^2 w_j^2$, whose supremum over $\{\|Dw\|_2 \le B\}$ is $\max_j \sigma_j^2 B^2 / D_j^2 = \max_l \operatorname{cap}_l$, attained on the maximising coordinate. Lower bound: run the packing of the proof of \cref{lem:envelope}(i) on the single level $l = 0$, with $k_0 = \min\bigl\{ d_0, \ \lfloor \operatorname{cap}_0 / (c_0 \varepsilon^2) \rfloor \bigr\}$ and $\sigma_{(0)}^2 \delta_0^2 = c_0 \varepsilon^2$; feasibility uses the full budget at one level, $k_0 D_0^2 \delta_0^2 = c_0 \varepsilon^2 k_0 B^2 / \operatorname{cap}_0 \le B^2$. If $k_0 \ge 32$, then either $k_0 = d_0$, in which case $k_0 \varepsilon^2 = d_0 \varepsilon^2 \ge 2 \operatorname{cap}_0$ by (ii), or $k_0 = \lfloor \operatorname{cap}_0 / (c_0 \varepsilon^2) \rfloor \ge 32$, in which case $\operatorname{cap}_0 \ge 32\, c_0 \varepsilon^2$ and $k_0 \ge \operatorname{cap}_0 / (2 c_0 \varepsilon^2)$, so $k_0 \varepsilon^2 \ge \operatorname{cap}_0 / (2 c_0) \ge 2 \operatorname{cap}_0$; in either case Fano gives $\mathcal{R}_m \ge (c_0/16)\, k_0 \varepsilon^2 \ge (c_0 / 8)\, \operatorname{cap}_0$. If $k_0 < 32$, then either $\operatorname{cap}_0 < 32\, c_0 \varepsilon^2$ or $d_0 < 32$, and in the latter case $\operatorname{cap}_0 \le \tfrac{1}{2} d_0 \varepsilon^2 < 16\, \varepsilon^2$ by (ii); in both cases $\min\{\operatorname{cap}_0, c_0 \varepsilon^2\} \ge (c_0/32)\, \operatorname{cap}_0$, and the two-point argument of \cref{lem:envelope}(i) at level $0$ with amplitude $\sigma_{(0)}^2\delta^2 = \min\{\operatorname{cap}_0, c_0\varepsilon^2\}$ gives $\mathcal{R}_m \ge c_5 \min\{\operatorname{cap}_0, c_0\varepsilon^2\} \ge (c_5 c_0 / 32)\, \operatorname{cap}_0$. Take $c_4$ the smaller of the two constants.

(iv) Dividing the two-sided bounds of (iii),
\[
c_4\, \frac{\operatorname{cap}_0^{\mathrm{gr}}}{\operatorname{cap}_0^{\mathrm{unif}}} \;\le\; \frac{\mathcal{R}_m(\mathcal{F}^{\mathrm{gr}})}{\mathcal{R}_m(\mathcal{F}^{\mathrm{unif}})} \;\le\; \frac{1}{c_4}\, \frac{\operatorname{cap}_0^{\mathrm{gr}}}{\operatorname{cap}_0^{\mathrm{unif}}}, \qquad \frac{\operatorname{cap}_0^{\mathrm{gr}}}{\operatorname{cap}_0^{\mathrm{unif}}} = g_0^2 \sum_k a_k g_k^{-2}.
\]
It remains to compare $g_0^2 \sum_k a_k g_k^{-2}$ with $\Lambda(g)$. Since $\sum_l s_l = 1$,
\[
\begin{split}
g_0^2 \sum_k a_k g_k^{-2} \;&=\; \Lambda(g) \cdot \frac{g_0^2}{\sum_l s_l g_l^2},
\\
s_0\, g_0^2 \;\le\; \sum_l s_l g_l^2 \;\le\; \frac{s_0\, g_0^2}{1-r},
\end{split}
\]
the second line because the summands $s_l g_l^2$ are geometric with ratio at most $r$. Hence $g_0^2 / \sum_l s_l g_l^2$ lies in $[\,1-r,\; 1/s_0\,] \subseteq [\,1-r,\; 1/(1-\mu)\,]$, which gives the display. For the fourth-root grades, \cref{cor:stratified} gives $\Lambda(g^\star) = \mathrm{BC}(a,s)^2 = \Theta(\nu^L)$, and $\Lambda(g^\star_C) = \Lambda^\star_C$ by \cref{prop:clipped-selection}.

(v) For such $m$ every level of both classes satisfies $d_l \varepsilon^2 \le \operatorname{cap}_l$, so $S$ is full, the envelope of \cref{lem:envelope} evaluates to $\Theta_L(\sum_l d_l \varepsilon^2) = \Theta_L(d\,\varepsilon^2)$ for both classes, and the ratio is bounded above and below by $C(L)/c(L)$ and its inverse: past the saturation point the budget is inactive for both classes and the comparison localises to $W$. The band assertion is a scope statement, not a theorem, and records where the separation is to be read off.
\end{proof}

\begin{thm}
\label{thm:stratified-separation}
Assume the setting of this subsection: Gaussian level-homogeneous design with squared loss, geometric profiles $a_l = \nu^{L-1-l}/Z_a$ and $s_l = \mu^l/Z_s$ with $0 < \mu < \nu < 1$, uniform allocation $d_l = d/L$, the calibrations $\|w^*\|_2 = 1$ and $\operatorname{tr} \Sigma = 1$, and a level-constant grade vector $g$ satisfying the monotonicity and ratio condition stated before \cref{lem:window}. Suppose $d \geq 3L/\sigma_\xi$, so that the window $W$ of \cref{lem:window} is nonempty. Then there exist constants $0 < c_{\mu,\nu} \leq C_{\mu,\nu}$ depending only on $(\mu, \nu)$, uniform in $L$, $d$, and $m$, such that for every $m \in W$,
\[
c_{\mu,\nu}\, \Lambda(g)
\;\leq\;
\frac{\mathcal{R}_m\big(\mathcal{F}^{\mathrm{gr}}\big)}{\mathcal{R}_m\big(\mathcal{F}^{\mathrm{unif}}\big)}
\;\leq\;
C_{\mu,\nu}\, \Lambda(g),
\]
with $\Lambda(g)$ the gain of \cref{prop:ml-expressivity}. In particular:
\begin{enumerate}
  \item[(i)] at the fourth-root grades $g^\star$ of \cref{prop:ml-expressivity}(iii) the ratio is $\Theta_{\mu,\nu}\big(\mathrm{BC}(a,s)^2\big) = \Theta_{\mu,\nu}(\nu^L)$;
  \item[(ii)] at the clipped grades $g^\star_C$ of \cref{prop:clipped-selection} the ratio is $\Theta_{\mu,\nu}\big(\Lambda^\star_C\big)$, with $\Lambda^\star_C < 1$ whenever $a \neq s$.
\end{enumerate}

The gain \cref{eq:optimal-gain} is therefore attained on $W$ up to constants depending only on $(\mu,\nu)$, over all estimators, and is a separation between the graded prior and its absence rather than a comparison of the upper bounds of \cref{prop:ml-expressivity}.
\end{thm}

\begin{proof}
The two-sided bound is \cref{lem:window}(iv), with $c_{\mu,\nu} = c_4(1-r)$ and $C_{\mu,\nu} = \big(c_4(1-\mu)\big)^{-1}$, where $c_4$ is the constant of \cref{lem:window}(iii) and $r = \sqrt{\mu/\nu}$; both grade vectors named in (i) and (ii) satisfy the grade condition, as recorded before \cref{lem:window}, so the bound applies to them. 

The value 
\[
\Lambda(g^\star) = \mathrm{BC}(a,s)^2 = \Theta(\nu^L)
\]
 is \cref{cor:stratified}, and $\Lambda(g^\star_C) = \Lambda^\star_C$ is \cref{prop:clipped-selection}. 
 For the strict inequality in (ii), $a \neq s$ gives $\mathcal{Q}_+ \neq \emptyset$ by \cref{prop:admissible-grades}(iii); by \cref{rem:cone-stratified} a level-constant grade increasing in $l$ satisfies the criterion of \cref{prop:admissible-grades}(iv), so $\varepsilon v \in \mathcal{Q}_+$ for all small $\varepsilon > 0$ along such a direction $v$, and for small $\varepsilon$ the point $\varepsilon v$ lies in the box $\mathcal{Q}_C$ as well; hence $\mathcal{Q}_C$ meets $\mathcal{Q}_+$ and the minimum over it is strictly below $1$. Nonemptiness of $W$ at $d \geq 3L/\sigma_\xi$ is \cref{lem:window}(i).
 \end{proof}

\Cref{thm:stratified-separation} establishes the separation of \cref{con:minimax-separation} in the regime the program occupies. Level-stratified geometric profiles are the hypothesis of \cref{cor:stratified} and the setting of every planned configuration, and $g^\star$ and $g^\star_C$ are the grades those configurations carry; the loss is squared, the curved case in which minimax risk is classically posed. The statement is a risk ratio at each sample size of the window, which is the form \cref{rem:budget-active} shows to be well posed: past saturation both risks sit on the parametric floor and the ratio is $\Theta_L(1)$ by \cref{lem:window}(v), so the theorem localises the gain to $W$ rather than diminishing it, and $W$ is the token-limited regime the framework is built for. What remains open is the general-profile, Lipschitz-loss case, and \cref{con:minimax-separation} states it.

\subsection{Graded Positional Encoding}
\label{subsec:graded-pe}

Let $l(j) \in \{0, \dots, L-1\}$ denote the level of global dimension $j$, and $i(j)$ its local index within $V_{l(j)}$. The graded positional encoding scales the standard sinusoidal encoding \cref{eq:pe} by the grade of the dimension it occupies:
\begin{equation}
\label{eq:graded-pe-full}
\mathrm{PE}'(t, j) = \lambda^{q_{l(j),\, i(j)}}
\begin{cases}
\sin\!\left( \dfrac{t}{10000^{\,2\lfloor i(j)/2 \rfloor / d_{l(j)}}} \right), & i(j) \text{ even}, \\[2ex]
\cos\!\left( \dfrac{t}{10000^{\,2\lfloor i(j)/2 \rfloor / d_{l(j)}}} \right), & i(j) \text{ odd}.
\end{cases}
\end{equation}

The wavelength is computed within the level, so that each $V_l$ carries a complete positional basis at its own resolution rather than a fragment of a global one. The input to the first layer is $X_0 = E' + \mathrm{PE}' \in \mathbb{R}^{n \times d}$.

\subsection{The Graded Forward Pass}
\label{subsec:graded-forward}

The stack applies $K$ pre-norm blocks in the ordering of \cref{sec:2.1}, with EG-MHSA in place of MHSA:
\begin{equation}
\label{eq:forward-pass}
\begin{split}
X_{\ell}' &= X_{\ell-1} + \mathrm{EG\text{-}MHSA}_{\ell}\big(\mathrm{LN}(X_{\ell-1})\big), \\
X_{\ell}  &= X_{\ell}' + \mathrm{FFN}_{\ell}\big(\mathrm{LN}(X_{\ell}')\big),
\end{split}
\end{equation}
for $\ell = 1, \dots, K$, with layer normalisation as in \cite{ba2016layer}.

The feed-forward block admits a level-decomposed variant,
\[
\mathrm{FFN}_{\ell}(X) = \sum_{l=0}^{L-1} \mathrm{GeGLU}_l\!\left(X P_l\right) P_l^\top ,
\]
using the gated linear unit of \cite{shazeer2020glu}. We record what this costs. The decomposed block is block-diagonal with respect to $\bigoplus_l V_l$ and therefore performs \emph{no} cross-level mixing; all interaction between levels must then occur in attention. This is a substantive architectural restriction and not merely a reparameterization, and it is in tension with the informal reading of the filtration $V_{\leq l} = \bigoplus_{l' \leq l} V_{l'}$ as a channel by which lower levels feed higher ones. A filtration is a nested family of subspaces; on its own it enforces no propagation. Directed level-to-level transport requires block-\emph{triangular} maps $\phi_{l' \leftarrow l} \colon V_l \to V_{l'}$ rather than block-diagonal ones; the morphic structure of \cite{sh-111} supplies them, and MLGE does not.

At position $t$ the final hidden state $z_t \in \mathbb{R}^d$ is mapped to vocabulary logits by $W_{\mathrm{voc}}^\top z_t \in \mathbb{R}^{|\mathcal{V}|}$, with $W_{\mathrm{voc}} \in \mathbb{R}^{d \times |\mathcal{V}|}$, and
\begin{equation}
\label{eq:graded-softmax}
p(y_t \mid y_{<t}) = \mathrm{softmax}\!\left( W_{\mathrm{voc}}^\top z_t + \kappa \, q_{\mathrm{out}} \right),
\qquad
q_{\mathrm{out}} \in \mathbb{Q}_{\geq 0}^{|\mathcal{V}|},
\ \kappa > 0 ,
\end{equation}
which multiplies the unnormalised probability of token $v$ by $e^{\kappa q_{\mathrm{out},v}}$ and is therefore monotone in the grade, uniformly in $z_t$.


\section{Training Graded LLMs}
\label{sec-5}

Training a GLLM retains the autoregressive next-token objective of classical LLMs \cite{radford2019language} and modifies the weight each token carries in it. The standard objective is the unweighted negative log-likelihood of \cref{eq:clm}. The graded objective attaches to position $t$ a weight $\lambda^{q_t}$ recording the structural salience of that position, reweighting the summands of the decomposition \cref{eq:ce-decomposition}. This section defines the objective, establishes what it converges to, and states precisely which of the framework's claims it can and cannot support.

We state the conclusion at the outset, because the section's structure follows from it. Grading is a \emph{statistical} device, and its whole force is spent on the resource that is scarce. It restricts the effective hypothesis class in a direction aligned with hierarchical targets, and the reduction in the data required to reach a given generalisation level is a consequence of that restriction, mediated by \cref{prop:ml-expressivity} and located exactly by \cref{prop:admissible-grades}. It buys nothing in step count: exponential rescaling worsens the conditioning of the objective, and we claim no acceleration. This distinction is central to the claim. Extra optimisation steps are cheap while extra high-quality tokens are not, so a device that spends conditioning to buy sample efficiency spends a cheap resource on an expensive one; \cref{rem:optimisation-cost} prices the two sides with the same constant, the step penalty at the clip $C$ and the token purchase at $C^2$. \Cref{sec-6} measures the two separately for exactly that reason.

\subsection{The Graded Objective and What It Minimises}
\label{subsec:graded-loss}

The base objective weights the per-token log-loss exponentially by a grade:
\begin{equation}
\label{eq:base-loss}
\mathcal{L}_{\mathrm{grade}}(y, \hat{p}; q) = -\sum_{t=1}^{n} \lambda^{q_t} \log \hat{p}(y_t \mid y_{<t}),
\qquad q = (q_1, \dots, q_n) \in \mathbb{Q}_{\geq 0}^{n},
\end{equation}
with $\hat{p}$ the model's predicted conditional from \cref{eq:graded-softmax} and $\lambda > 1$ the grading base fixed in \cref{sec-3}.

Everything depends on what $q_t$ is allowed to depend on, and the following proposition draws the line.

\begin{prop}
\label{prop:properness}
Let $p_\star(\cdot \mid y_{<t})$ denote the true conditional and let $\hat{p}$ range over all of $\Delta^{|\mathcal{V}|-1}$.
\begin{enumerate}
  \item If $q_t = q(y_{<t})$ depends only on the context, then the population minimizer of \cref{eq:base-loss} is $\hat{p}_\star(\cdot \mid y_{<t}) = p_\star(\cdot \mid y_{<t})$. The objective is a reweighting of contexts and is consistent.
  \item If $q_t = q(y_t)$ depends only on the realised target, then the population minimizer is the exponentially tilted conditional
\[
\hat{p}_\star(v \mid y_{<t}) = \frac{\lambda^{q(v)}\, p_\star(v \mid y_{<t})}{\sum_{u \in \mathcal{V}} \lambda^{q(u)}\, p_\star(u \mid y_{<t})} \;\neq\; p_\star(v \mid y_{<t}),
\]
and the objective is inconsistent for $p_\star$; the tilt is the same at every context.
  \item If $q_t = q(y_t, y_{<t})$ depends on the realised target \emph{and} its context --- a constituency parse depth, for instance --- then the population minimizer is the context-dependent tilted conditional
\[
\hat{p}_\star(v \mid y_{<t}) = \frac{\lambda^{q(v,\, y_{<t})}\, p_\star(v \mid y_{<t})}{\sum_{u \in \mathcal{V}} \lambda^{q(u,\, y_{<t})}\, p_\star(u \mid y_{<t})} \;\neq\; p_\star(v \mid y_{<t}),
\]
and the objective is inconsistent for $p_\star$; the tilt varies with the context.
\end{enumerate}
\end{prop}

\begin{proof}
For fixed context $y_{<t}$, write $w_v = \lambda^{q_t}$ for the weight attached to the outcome $v$ and minimise $-\sum_{v} w_v\, p_\star(v)\log \hat{p}(v)$ over the simplex. The Lagrangian stationarity condition $-w_v p_\star(v)/\hat{p}(v) + \mu = 0$ gives $\hat{p}(v) = w_v p_\star(v)/\mu$, and $\mu$ is fixed by normalisation, yielding $\hat{p}(v) = w_v p_\star(v) / \sum_u w_u p_\star(u)$. In case (1) $w_v \equiv \lambda^{q(y_{<t})}$ is constant in $v$ and cancels. In case (2) $w_v = \lambda^{q(v)}$ does not, and is the same function of $v$ at every context. In case (3) $w_v = \lambda^{q(v,\, y_{<t})}$ does not cancel either, and the resulting tilt is a different function of $v$ at each context.
\end{proof}

\begin{rem}
\label{rem:tilt-cancellation}
Case (2) of \cref{prop:properness} is a constraint that the architecture can match. Write the graded softmax \cref{eq:graded-softmax} as $\hat{p}(v \mid y_{<t}) \propto e^{\kappa q_{\mathrm{out},v}}\, s(v \mid y_{<t})$, where $s$ is the conditional realised by the ungraded head. Setting $\kappa = \ln \lambda$ and $q_{\mathrm{out},v} = q(v)$ makes the architectural tilt of \cref{eq:graded-softmax} coincide with the objective tilt of \cref{prop:properness}(2), so that the population minimizer satisfies $s = p_\star$: the underlying model is calibrated to the corpus, and the grading is carried entirely by the output head. Under this matching, and only under it, the graded objective and the graded architecture are compatible with the standard language-modelling objective.

Case (3) admits no such matching. The tilt of \cref{prop:properness}(3) is a different function of $v$ at each context, whereas $q_{\mathrm{out}}$ in \cref{eq:graded-softmax} is a single fixed vector, applied identically at every context; no choice of $\kappa$ and $q_{\mathrm{out}}$ can therefore absorb a context-dependent tilt. A grade that consults the target and its context jointly, a parse depth being the canonical example, leaves the trained model uncalibrated, with the miscalibration at each context equal to the Kullback--Leibler divergence between the corpus conditional and its tilt at that context.
\end{rem}

\Cref{prop:properness}(3) rules out the parse-depth initialisation $q_t = \operatorname{depth}(y_t \mid y_{<t})$ outright: parse depth is a function of the target and its context jointly, so it falls under case (3), where by \cref{rem:tilt-cancellation} no matching is available. We therefore take $q_t = q_{\mathrm{out}, y_t}$ with $q_{\mathrm{out}}$ a per-type grade estimated offline, which is a function of the token alone, falls under case (2), and admits the cancellation. Part-of-speech classes supply a natural default: high grades for open-class types, low for closed-class function words. A parse-depth grading is retained in the program only as an intentionally uncalibrated ablation; by \cref{prop:properness}(3) its perplexity gap measures the tilt, not the inductive bias.

\subsection{Regularisation of the Grades}
\label{subsec:grade-reg}

The trainable parameters are
\[
\theta = \{W_e, W_{\mathrm{voc}}, \{W_{Q_h}, W_{K_h}, W_{V_h}\}_{h=1}^{H}, W_O, \{P_l\}_{l=0}^{L-1}, \{q_h\}, \{q_l\}, q_{\mathrm{out}}\}.
\]
Collect all grades into $\mathbf{q}$. The full objective is
\begin{equation}
\label{eq:total-loss}
\begin{split}
\mathcal{L}_{\mathrm{total}}(\theta)
&= \mathbb{E}_{y \sim \mathcal{D}}\!\left[ w(y)\, \mathcal{L}_{\mathrm{grade}}(y, \hat{p}; q) \right]
 + \nu_{\mathrm{mag}} \|\mathbf{q}\|_2^2
 + \nu_{\mathrm{orth}}\, \Omega_{\mathrm{orth}}(P) \\
&\quad - \nu_{\mathrm{div}} \sum_{1 \leq h < h' \leq H} \big\| q_h - q_{h'} \big\|_2^2 ,
\end{split}
\end{equation}
with $\nu_{\mathrm{mag}}, \nu_{\mathrm{orth}}, \nu_{\mathrm{div}} > 0$, and $\Omega_{\mathrm{orth}}$ the orthogonality penalty of \cref{sec-4}.

Three points on \cref{eq:total-loss} deserve emphasis. First, the diversity term enters with a \emph{negative} sign and is therefore repulsive: it rewards heads whose grade tuples differ. This is required for consistency with the framework, since the graded bilinear forms of \cref{subsec:eg-mhsa} are head-specific only when the $q_h$ differ, and \cref{sec-6} predicts head specialisation as an observable. An attractive penalty $\sum_h \|q_h - \bar q\|_2^2$ has its minimum at $q_h \equiv \bar q$, at which point every head carries the identical grading and the construction degenerates to a global rescaling; such a term would penalise precisely the phenomenon the framework predicts.

Second, the repulsion must be bounded, and the magnitude term $\nu_{\mathrm{mag}}\|\mathbf{q}\|_2^2$ supplies the cap. Writing $\mathbf{q} = (q_1, \dots, q_H)$ with $\|\mathbf{q}\|_2^2 = \sum_h \|q_h\|_2^2$, the polarisation identity
\[
\sum_{1 \leq h < h' \leq H} \big\| q_h - q_{h'} \big\|_2^2
 = H \sum_{h=1}^{H} \|q_h\|_2^2 - \Big\| \sum_{h=1}^{H} q_h \Big\|_2^2
\]
shows that the two quadratic terms of \cref{eq:total-loss} contribute together
\[
\nu_{\mathrm{mag}} \|\mathbf{q}\|_2^2 - \nu_{\mathrm{div}} \sum_{h < h'} \|q_h - q_{h'}\|_2^2
 = \big( \nu_{\mathrm{mag}} - \nu_{\mathrm{div}} H \big) \|\mathbf{q}\|_2^2 + \nu_{\mathrm{div}} \Big\| \sum_{h} q_h \Big\|_2^2 ,
\]
a quadratic form with eigenvalue $2\nu_{\mathrm{mag}}$ on the head-constant directions and $2(\nu_{\mathrm{mag}} - \nu_{\mathrm{div}} H)$ on their orthogonal complement. We therefore impose
\begin{equation}
\label{eq:reg-condition}
\nu_{\mathrm{mag}} > \nu_{\mathrm{div}}\, H ,
\end{equation}
which is exactly the condition under which the pair is coercive, and hence \cref{eq:total-loss} bounded below in $\mathbf{q}$, and simultaneously the condition under which it is convex. The bound is tight: taking $H = 2$ and $q_1 = -q_2 = v$ gives $2(\nu_{\mathrm{mag}} - 2\nu_{\mathrm{div}})\|v\|_2^2$, which is unbounded below as $\|v\|_2 \to \infty$ whenever $\nu_{\mathrm{mag}} < \nu_{\mathrm{div}} H$.

Third, data filtering, if desired, must act \emph{inside} the expectation, through the per-sequence weight $w(y) \in [0,1]$; an additive penalty depending only on the data has zero gradient and no effect on training.

\begin{rem}
\label{rem:sequence-weight}
\Cref{prop:properness}(1) applies to $w(y)$: reweighting whole sequences leaves the per-context minimizer unchanged, so $w$ may be set from a parse-validity check without disturbing consistency.
\end{rem}

The grades entering \cref{eq:total-loss} are refined end-to-end, but they are not found there: they are selected offline from estimated profiles, by the convex program of \cref{cor:grade-convexity} and its clipped form \cref{prop:clipped-selection}, both developed in \cref{sec-7}.

\subsection{Optimisation}
\label{subsec:optimisation}

Optimisation uses AdamW \cite{loshchilov2017decoupled}, with raw moments
\[
\begin{split}
m_k &= \beta_1 m_{k-1} + (1 - \beta_1)\, g_k, \\
v_k &= \beta_2 v_{k-1} + (1 - \beta_2)\, g_k \odot g_k, \qquad g_k := \nabla_\theta \mathcal{L}_{\mathrm{total}}(\theta_k),
\end{split}
\]
bias corrections $\hat{m}_k = m_k / (1 - \beta_1^{k})$ and $\hat{v}_k = v_k / (1 - \beta_2^{k})$, and update
\[
\theta_{k+1} = \theta_k - \eta_k \frac{\hat{m}_k}{\sqrt{\hat{v}_k} + \epsilon} - \eta_k\, \omega\, \theta_k ,
\]
taking $\beta_1 = 0.9$, $\beta_2 = 0.95$, $\omega = 0.1$, $\epsilon = 10^{-8}$. The learning rate follows the cosine schedule of \cite{loshchilov2017sgdr},
\[
\eta_k = \eta_{\min} + \tfrac{1}{2}\left(\eta_{\max} - \eta_{\min}\right)\left(1 + \cos(\pi k / T)\right),
\]
with $\eta_{\max} = 6 \times 10^{-4}$ and $\eta_{\min} = 10^{-5}$.

Gradient clipping must be stated carefully, since the point of clipping is to remove the grade dependence rather than to accommodate it. We clip the global gradient norm to a constant independent of $\lambda$ and of $q_{\max}$: a threshold set at $\lambda^{q_{\max}}$ would rescale with the very amplification it is meant to control and would impose no constraint at all. Together with the magnitude penalty of \cref{eq:total-loss} under \cref{eq:reg-condition} and the grade clipping of \cref{subsubsec:configs}, this keeps the effective step size bounded uniformly in the grades.

Distributed training uses standard sharded data parallelism with optimiser-state partitioning \cite{rajbhandari2020zero} and pipeline parallelism across devices \cite{huang2019gpipe}.

\begin{lem}
\label{lem:graded-smoothness}
Suppose each per-token loss $\ell_t(\theta) = -\log \hat{p}(y_t \mid y_{<t})$ is $\varsigma_t$-smooth on a region $\Theta$, and write $\varsigma_{\mathrm{unif}} := \sum_{t=1}^{n} \varsigma_t$ for the resulting smoothness constant of the unweighted sum. Then on $\Theta$ the graded loss \cref{eq:base-loss} is $\varsigma_{\mathrm{grade}}$-smooth with
\[
\varsigma_{\mathrm{grade}} \;\leq\; \lambda^{q_{\max}}\, \varsigma_{\mathrm{unif}},
\qquad
q_{\max} := \max_t q_t ,
\]
and the regularisation terms of \cref{eq:total-loss} contribute an additive constant $c_{\mathrm{reg}}$ depending on $\nu_{\mathrm{mag}}, \nu_{\mathrm{orth}}, \nu_{\mathrm{div}}, H$ and the diameter of $\Theta$, but not on $\lambda$ or $\mathbf{q}$.
\end{lem}

\begin{proof}
For $\theta, \theta' \in \Theta$,
\[
\begin{split}
\big\| \nabla \mathcal{L}_{\mathrm{grade}}(\theta) - \nabla \mathcal{L}_{\mathrm{grade}}(\theta') \big\|_2  
&	\leq \sum_{t=1}^{n} \lambda^{q_t} \big\| \nabla \ell_t(\theta) - \nabla \ell_t(\theta') \big\|_2  \\
&	\leq \Big( \sum_{t=1}^{n} \lambda^{q_t} \varsigma_t \Big) \|\theta - \theta'\|_2
\leq \lambda^{q_{\max}} \varsigma_{\mathrm{unif}} \|\theta - \theta'\|_2 ,
\end{split}
\]
using the triangle inequality, the per-token hypothesis, and $\lambda^{q_t} \leq \lambda^{q_{\max}}$. The magnitude and diversity penalties are quadratics in $\mathbf{q}$ whose Hessians have operator norm at most $2\nu_{\mathrm{mag}}$ and $2\nu_{\mathrm{div}} H$ respectively, by the eigenvalue computation of \cref{subsec:grade-reg}, and $\Omega_{\mathrm{orth}}$ is a polynomial in the entries of $P$ with gradient Lipschitz on the bounded region $\Theta$; none of these constants involves $\lambda$ or $\mathbf{q}$.
\end{proof}

\begin{thm}
\label{thm:convergence}
Suppose $\mathcal{L}_{\mathrm{total}}$ is $\varsigma$-smooth on the region visited, stochastic gradients are unbiased with variance at most $\sigma^2$ at batch size $b$, and the step size satisfies $\eta \leq 1/\varsigma$. Then after $T$ steps of stochastic gradient descent,
\[
\min_{k \leq T} \mathbb{E}\big\| \nabla_\theta \mathcal{L}_{\mathrm{total}}(\theta_k) \big\|_2^2
\;\leq\; \frac{2\big(\mathcal{L}_{\mathrm{total}}(\theta_0) - \inf \mathcal{L}_{\mathrm{total}}\big)}{\eta T}
+ \frac{\varsigma\, \eta\, \sigma^2}{b},
\]
and by \cref{lem:graded-smoothness} the smoothness constant obeys $\varsigma \leq \lambda^{q_{\max}}\, \varsigma_{\mathrm{unif}} + c_{\mathrm{reg}}$, with $\varsigma_{\mathrm{unif}}$ the corresponding constant for the ungraded objective.
\end{thm}

\begin{proof}
By $\varsigma$-smoothness, for consecutive iterates $\theta_{k+1} = \theta_k - \eta\, \hat{g}_k$ with $\hat{g}_k$ the stochastic gradient,
\[
\mathcal{L}_{\mathrm{total}}(\theta_{k+1})
\leq \mathcal{L}_{\mathrm{total}}(\theta_k) - \eta \langle \nabla \mathcal{L}_{\mathrm{total}}(\theta_k), \hat{g}_k \rangle + \frac{\varsigma \eta^2}{2} \|\hat{g}_k\|_2^2 .
\]
Taking expectations conditional on $\theta_k$, using unbiasedness $\mathbb{E}[\hat{g}_k] = \nabla \mathcal{L}_{\mathrm{total}}(\theta_k)$ and the variance bound $\mathbb{E}\|\hat{g}_k\|_2^2 \leq \|\nabla \mathcal{L}_{\mathrm{total}}(\theta_k)\|_2^2 + \sigma^2 / b$,
\[
\mathbb{E}\big[\mathcal{L}_{\mathrm{total}}(\theta_{k+1})\big]
\leq \mathbb{E}\big[\mathcal{L}_{\mathrm{total}}(\theta_k)\big]
- \eta \Big( 1 - \frac{\varsigma \eta}{2} \Big) \mathbb{E}\|\nabla \mathcal{L}_{\mathrm{total}}(\theta_k)\|_2^2
+ \frac{\varsigma \eta^2 \sigma^2}{2b} .
\]
Since $\eta \leq 1/\varsigma$ gives $1 - \varsigma\eta/2 \geq 1/2$, rearranging and summing over $k = 0, \dots, T-1$ telescopes the loss terms, and dividing by $\eta T / 2$ bounds the minimum by the average:
\[
\min_{k \leq T} \mathbb{E}\|\nabla \mathcal{L}_{\mathrm{total}}(\theta_k)\|_2^2
\leq \frac{1}{T}\sum_{k=0}^{T-1} \mathbb{E}\|\nabla \mathcal{L}_{\mathrm{total}}(\theta_k)\|_2^2
\leq \frac{2\big(\mathcal{L}_{\mathrm{total}}(\theta_0) - \inf \mathcal{L}_{\mathrm{total}}\big)}{\eta T} + \frac{\varsigma \eta \sigma^2}{b} .
\]
The smoothness comparison is \cref{lem:graded-smoothness}.
\end{proof}

\begin{rem}
\label{rem:optimisation-cost}
The factor $\lambda^{q_{\max}} \geq 1$ enters \cref{thm:convergence} as a price: grading degrades the smoothness constant and therefore the stationarity bound, and no reading of this theorem yields a convergence speedup. The price and the purchase are controlled by the same constant. The grade clipping of \cref{subsubsec:configs} enforces $\lambda^{q_{\max}} \leq C$ with $C = 2$ at the planned configurations, so \cref{thm:convergence} degrades by at most a factor of $C$, uniformly in $L$ and in the profiles; the token purchase at the same clip is at most $C^2$ by \cref{prop:clipped-selection}, and the exponential regime of \cref{cor:stratified} is reached as the clip opens (\cref{rem:one-dial}). A price linear in the clip against a purchase quadratic in it is the trade the framework makes, favourable at every setting, and it is the reason the two are reported separately in \cref{sec-6} rather than folded into a single figure.

Three restrictions on scope. The bound is over stationarity and not suboptimality, since $\mathcal{L}_{\mathrm{total}}$ is nonconvex in $\theta$ and no global guarantee is available. The theorem is stated for stochastic gradient descent, while \cref{subsec:optimisation} specifies AdamW; Adam-family methods admit no unconditional convergence guarantee even on convex problems, as shown by Reddi, Kale, and Kumar, so the statement is a discipline on the objective rather than on the optimiser actually deployed. And \cref{lem:graded-smoothness} bounds the smoothness through per-token constants, so $\varsigma_{\mathrm{unif}}$ is the sum of per-token smoothness constants rather than the (possibly smaller) constant of the summed objective; the comparison is honest but not tight. Consequently, any prediction of reduced training tokens to a fixed generalisation target must be routed through \cref{prop:ml-expressivity} and \cref{prop:admissible-grades}, and must not be attributed to this theorem.
\end{rem}

\subsection{Pretraining and Fine-Tuning}
\label{subsec:pretraining}

Pretraining is planned on The Pile \cite{gao2020pile}, approximately 800\,GB of diverse text, at sequence lengths $n \in [1024, 8192]$. Per-type grades $q_{\mathrm{out}}$ are estimated once, offline, by the clipped selection problem of \cref{prop:clipped-selection} applied to part-of-speech statistics over the corpus, certified against \cref{rem:certified-init}, and refined end-to-end thereafter under the constraint of \cref{rem:tilt-cancellation}.

Fine-tuning uses direct preference optimisation \cite{rafailov2023direct}. Writing $x$ for the prompt, $y_w$ and $y_l$ for the preferred and rejected responses, $\pi_\theta$ for the policy and $\pi_{\mathrm{ref}}$ for the reference,
\begin{equation}
\label{eq:graded-dpo}
\mathcal{L}_{\mathrm{DPO}}
= -\, \mathbb{E}_{(x, y_w, y_l) \sim \mathcal{D}_{\mathrm{pref}}}
\left[
\log \sigma\!\left(
\beta_{\mathrm{DPO}} \log \frac{\pi_\theta(y_w \mid x)}{\pi_{\mathrm{ref}}(y_w \mid x)}
- \beta_{\mathrm{DPO}} \log \frac{\pi_\theta(y_l \mid x)}{\pi_{\mathrm{ref}}(y_l \mid x)}
\right)
\right],
\end{equation}
with $\beta_{\mathrm{DPO}} > 0$ the inverse-temperature of the implicit reward and $\mathcal{D}_{\mathrm{pref}}$ the preference dataset. A graded variant weights the token-level contributions to $\log \pi_\theta(y \mid x)$ by their grades, concentrating preference signal on structurally salient positions rather than distributing it uniformly across a response. The grades used must be the $q_{\mathrm{out}}$ of \cref{rem:tilt-cancellation} and no others: a weighting that depends on the realised token and its context reintroduces the tilt of \cref{prop:properness}(3) into the implicit reward, where no cancellation is available to absorb it. Subject to that constraint the variant is a construction whose behaviour the program does not yet test.

\Cref{sec-7} assembles the selection results of this paper into the general procedure, \cref{sec-8} instantiates that procedure for language, and \cref{sec-6} states the experimental program that tests the resulting predictions.

\section{Grade Selection}
\label{sec-7}

The results of \cref{sec-3,sec-4,sec-5} assign to every grade vector a price and a value: a conditioning penalty bounded by \cref{lem:graded-smoothness}, and a sample-complexity gain computed by \cref{prop:ml-expressivity} and located by \cref{prop:admissible-grades}. This section assembles those results into the framework's main application. Grades are neither a hyperparameter to sweep nor a taxonomy to posit; they are the solution of an optimisation problem whose inputs are measurable before training. The procedure below computes them, certifies them, and specifies the experiment that tests the computation.

The distinction that organises the section is between domains that carry canonical grades and domains that do not. Weighted projective coordinates and graded rings arrive with their grades attached: the weights of the Igusa invariants are $(2,4,6,10)$ by theorem, not by choice, and in such settings the machinery of this paper adds a consistency check rather than a decision. Language carries no canonical grading; which coordinates of an embedding space deserve amplification is exactly the question no theorem of linguistics answers, and it is the question \cref{prop:ml-expressivity}(iii) answers in closed form. The contribution of the framework is educated grade selection in domains where the grades are not given.

\subsection{The Selection Procedure}
\label{subsec:selection-procedure}

The procedure has six steps, and each names the result that supports it.

\begin{enumerate}
  \item \emph{Fix the basis.} The profiles of \cref{def:profiles} are defined relative to the graded basis of \cref{eq:mlge-transform}. In language this is the embedding basis carrying the level projections $\{P_l\}$; in a domain with canonical structure it is the given one.
  \item \emph{Estimate the data profile.} $\hat\tau_j$ is the per-coordinate variance over the corpus, normalised to the simplex, as in \cref{subsec:cs-profiles}. This is a corpus statistic and requires no model.
  \item \emph{Estimate the target profile.} $\hat\alpha_j$ is the normalised squared weight of a linear probe trained on the annotated target, reported together with the probe's inductive bias, as in \cref{subsec:cs-profiles}.
  \item \emph{Test divergence.} Compute $\mathrm{BC}(\hat\alpha, \hat\tau)$. By \cref{prop:ml-expressivity}(iv) a value near $1$ means the profiles agree and grading buys nothing in this domain at this target; the procedure terminates here, at the cost of two estimates. This step is the domain filter, and it is where \cref{assum:profile-divergence} is tested rather than assumed.
  \item \emph{Certify candidates.} A proposed grade vector $q$, whether taxonomic, statistical, or computed, is admissible to first order exactly when $\langle q, \hat\alpha - \hat\tau \rangle > 0$, by \cref{prop:admissible-grades}(iv). The inner product is one line of arithmetic and is reported before training.
  \item \emph{Compute the optimum.} The grades minimising the bound are the solution of the convex program of \cref{cor:grade-convexity} below, projected to the clipped region by \cref{prop:clipped-selection}. The output is $q^\star_C$, the educated grading at clip $C$; by \cref{lem:plugin-stability} the gain it attains degrades only to second order in the estimation error of steps (2) and (3).
\end{enumerate}

Steps (2)--(4) cost no pretraining, step (5) costs an inner product, and step (6) is a geometric program in $d$ variables. The entire selection is performed, published, and falsifiable before the first training step, which is the property the staging of \cref{sec-6} is built around.

\subsection{The Offline Selection Problem}
\label{subsec:offline-selection}

The grades entering \cref{eq:total-loss} are refined end-to-end, but they are not found there. They are selected offline from estimated profiles, and that problem is well posed in a way the end-to-end problem is not.

\begin{cor}
\label{cor:grade-convexity}
Let $\hat\alpha, \hat\tau \in \Delta^{d-1}$ be strictly positive estimates of the profiles of \cref{def:profiles}, and consider the offline selection problem
\[
\min_{q \in \mathbb{R}_{\geq 0}^{d}} \ \Big\{ \log \Lambda(q) + \nu_{\mathrm{mag}} \|q\|_2^2 \Big\},
\qquad
\Lambda(q) = \Big( \sum_{j=1}^{d} \hat\alpha_j\, \lambda^{-2 q_j} \Big) \Big( \sum_{j=1}^{d} \hat\tau_j\, \lambda^{2 q_j} \Big) .
\]
Then the objective is strictly convex on the convex feasible set $\mathbb{R}_{\geq 0}^d$, so the minimizer is unique and every stationary point is the global minimum. In the variables $t_j = \lambda^{2 q_j}$ the function $\Lambda$ is a posynomial and the problem is a geometric program. As $\nu_{\mathrm{mag}} \to 0$ the value tends to $\log \mathrm{BC}(\hat\alpha, \hat\tau)^2$; for $\nu_{\mathrm{mag}} > 0$ the penalty breaks the gauge invariance of \cref{prop:admissible-grades}(i) and selects the minimum-norm representative of the optimal orbit. The version of this problem on the clipped region the architecture admits is solved in \cref{prop:clipped-selection}, with the same convex structure and a saturated form of the same optimum.
\end{cor}

\begin{proof}
Convexity of $q \mapsto \log \Lambda(q)$ is \cref{prop:admissible-grades}(ii), and $\nu_{\mathrm{mag}}\|q\|_2^2$ is strictly convex, so the sum is strictly convex; $\mathbb{R}_{\geq 0}^d$ is convex, whence uniqueness and the coincidence of stationary and global minima. For the posynomial claim, $\sum_j \hat\alpha_j t_j^{-1}$ and $\sum_j \hat\tau_j t_j$ are posynomials in $t \in \mathbb{R}_{>0}^d$ and a product of posynomials is a posynomial. The limiting value is \cref{prop:admissible-grades}(vi), and the representative in $\mathbb{R}_{\geq 0}^d$ with $\min_j q_j = 0$ supplied by \cref{prop:admissible-grades}(i) has finite norm, so the penalty is finite along the orbit and selects its minimum-norm point.
\end{proof}

\begin{rem}
\label{rem:certified-init}
\Cref{cor:grade-convexity} has a consequence available before any pretraining run begins. By \cref{prop:admissible-grades}(iv), a proposed initialisation $q^{\mathrm{init}}$, say the part-of-speech grades of \cref{subsubsec:grade-init}, lies in the admissible cone $\mathcal{Q}_+$, to first order, exactly when
\[
\big\langle q^{\mathrm{init}}, \hat\alpha - \hat\tau \big\rangle > 0 ,
\]
a single inner product between the proposed grades and the estimated profiles. Neither quantity requires a trained model: $\hat\tau$ is a corpus statistic and $\hat\alpha$ a probe weight, both computed in \cref{subsec:cs-profiles}. A GLLM may therefore be certified in advance to begin inside the region where its sample-complexity bound is the smaller one, at the cost of one dot product, and \cref{cor:grade-convexity} then locates the optimum within that region by convex programming rather than by search. The magnitude of the gain depends on the accuracy of $\hat\alpha$ and $\hat\tau$, controlled by \cref{lem:plugin-stability}; the sign does not, and only the sign is needed for certification. At the clipped configurations the relevant optimum is the $q^\star_C$ of \cref{prop:clipped-selection}, and the attainable gain is bounded by $C^2$ accordingly.

The end-to-end refinement is a different matter. \Cref{eq:total-loss} is nonconvex in $\theta$, and no counterpart of \cref{cor:grade-convexity} is available for it. The role of the offline problem is to supply a certified starting point, not to replace training.
\end{rem}

\subsection{Selection Under Clipping}
\label{subsec:clipped-selection}

The architecture is trained under a grade clip $\lambda^{q_{\max}} \leq C$, with $C = 2$ at the configurations of \cref{subsubsec:configs}, and the selection problem must be solved on the region the clip admits. The clipped problem retains the structure of the unclipped one, and the clip prices the gain exactly.

\begin{prop}
\label{prop:clipped-selection}
Fix $C = \lambda^{q_{\max}} \geq 1$ and let $\mathcal{Q}_C = \{ q \in \mathbb{R}^d : 0 \leq q_j \leq q_{\max} \}$. Write $\Lambda^\star_C = \min_{q \in \mathcal{Q}_C} \Lambda(q)$ and $r_j = \tfrac{1}{4 \log \lambda} \log(\alpha_j / \tau_j)$.
\begin{enumerate}
  \item[(i)] The minimum is attained, and the problem is convex: $\log \Lambda$ is convex on the compact convex set $\mathcal{Q}_C$.
  \item[(ii)] Every minimizer has the saturated form
\[
q_j^\star \;=\; \min\big\{ q_{\max},\, \max\{ 0,\; r_j + c \} \big\}
\]
for some constant $c \in \mathbb{R}$: the fourth-root law of \cref{prop:ml-expressivity}(iii) on the interior coordinates, saturation at the box on the rest.
  \item[(iii)] The gain is bounded by the clip:
\[
\Lambda^\star_C \;\geq\; \max\big\{ \mathrm{BC}(\alpha, \tau)^2,\; C^{-2} \big\},
\]
and the bound $C^{-2}$ is approached as the profiles concentrate on separate coordinates. Reaching the regime $\Lambda^\star = \Theta(\nu^L)$ of \cref{cor:stratified} therefore requires $C \geq \nu^{-L/2}$: the exponential-in-$L$ gain is the limit of a growing clip, at the conditioning price \cref{lem:graded-smoothness} attaches to it.
\end{enumerate}
\end{prop}

\begin{proof}
(i) $\mathcal{Q}_C$ is a product of intervals, hence compact and convex, and $\log \Lambda$ is convex by \cref{prop:admissible-grades}(ii) and continuous.

(ii) In the variables $u_j = 2 q_j \log \lambda$ of \cref{prop:admissible-grades}, the Karush--Kuhn--Tucker conditions for the box read $\partial_j \Phi(u) = \tau^{(u)}_j - \alpha^{(u)}_j = \mu_j^{-} - \mu_j^{+}$ with $\mu_j^{\pm} \geq 0$ supported on the active constraints. On an interior coordinate both multipliers vanish and $\alpha_j e^{-u_j} / A(u) = \tau_j e^{u_j} / T(u)$, giving $u_j = \tfrac12 \log(\alpha_j/\tau_j) + \tfrac12 \log\!\big(T(u)/A(u)\big)$, which is $q_j = r_j + c$ with $c = \tfrac{1}{4\log\lambda} \log\!\big(T(u^\star)/A(u^\star)\big)$ common to all interior coordinates. On a coordinate active at the lower bound the condition $\partial_j \Phi \geq 0$ holds exactly when the interior law would give $r_j + c \leq 0$, and symmetrically at the upper bound, which is the saturated form. Since $\Lambda$ is invariant along $\mathbf{1}$ by \cref{prop:admissible-grades}(i), the minimizer set is a segment in that direction intersected with $\mathcal{Q}_C$, and every point of it has the stated form under the corresponding shift of $c$.

(iii) The bound $\mathrm{BC}(\alpha,\tau)^2$ is the unconstrained minimum of \cref{prop:ml-expressivity}(iii). For the second bound, every feasible $g_j = \lambda^{q_j}$ lies in $[1, C]$, so
\[
\Lambda(q) = \Big( \sum_{j=1}^{d} \alpha_j g_j^{-2} \Big) \Big( \sum_{j=1}^{d} \tau_j g_j^{2} \Big) \;\geq\; C^{-2} \Big( \sum_{j} \alpha_j \Big) \cdot 1 \cdot \Big( \sum_{j} \tau_j \Big) = C^{-2} .
\]
For tightness take $d = 2$, $\alpha \to (1, 0)$, $\tau \to (0, 1)$, and $q = (q_{\max}, 0)$: the first factor tends to $C^{-2}$ and the second to $1$. For the final claim, $\Lambda^\star_C \leq \nu^L$ forces $C^{-2} \leq \nu^L$.
\end{proof}

\begin{rem}
\label{rem:one-dial}
\Cref{prop:clipped-selection}(iii) and \cref{lem:graded-smoothness} are controlled by the same constant, and together they give the framework's cost--benefit statement in final form: at clip $C$ the conditioning penalty is at most $C$ and the sample-complexity gain is at most $C^2$. The price is linear in the clip and the purchase quadratic, so the trade is favourable at every setting, and the planned configuration $C = 2$ buys a gain of at most $4$ at a penalty of at most $2$. The exponential regime of \cref{cor:stratified} is not available at fixed clip; it is reached as the clip grows, with both sides of the trade priced at every point along the way. The experiments of \cref{sec-6} are accordingly designed to detect a bounded effect at $C = 2$ and to sweep $C$ as an explicit arm.
\end{rem}

\subsection{Level Reduction}
\label{subsec:level-reduction}

Under the stratification of \cref{cor:stratified} the selection problem collapses from $d$ variables to $L$, and this reduction is what makes step (6) implementable at scale.

\begin{prop}
\label{prop:level-reduction}
Suppose the profiles are uniform within levels in the sense of \cref{cor:stratified}: $\alpha_j = a_l / d_l$ and $\tau_j = s_l / d_l$ for $j \in V_l$, with $a, s \in \Delta^{L-1}$ strictly positive. Then every minimizer of $\Lambda$ over $\mathbb{R}^d$, and every minimizer over the clipped box $\mathcal{Q}_C$ of \cref{prop:clipped-selection}, is constant on each level. Consequently the selection problem is the box-constrained convex program in $L$ variables
\[
\min_{v \in [0, q_{\max}]^{L}} \ \Lambda_L(v),
\qquad
\Lambda_L(v) = \left( \sum_{l=0}^{L-1} a_l \lambda^{-2 v_l} \right) \left( \sum_{l=0}^{L-1} s_l \lambda^{2 v_l} \right),
\]
whose solution $v^\star$ determines $q^\star$ by $q^\star_j = v^\star_l$ for $j \in V_l$. At $L = 4$ the program has four variables, whatever $d$ may be.
\end{prop}

\begin{proof}
Work in the variables $u_j = 2 q_j \log \lambda$ and write $\Phi(u) = \log A(u) + \log T(u)$ as in the proof of \cref{prop:admissible-grades}, with $A(u) = \sum_j \alpha_j e^{-u_j}$ and $T(u) = \sum_j \tau_j e^{u_j}$. Let $\Sigma$ be the group of coordinate permutations $\sigma$ with $\sigma(V_l) = V_l$ for every $l$. Under the stratification hypothesis $\alpha$ and $\tau$ are constant on each level, hence $\Sigma$-invariant, so $\Phi(\sigma \cdot u) = \Phi(u)$ for every $\sigma \in \Sigma$; and both feasible sets, $\mathbb{R}^d$ and the box, are $\Sigma$-invariant and convex.

Let $\mathcal{M}$ denote the set of minimizers, nonempty by \cref{prop:clipped-selection}(i) in the clipped case and by \cref{prop:ml-expressivity}(iii) in the unclipped one. Since $\Phi$ is convex and the feasible set is convex, $\mathcal{M}$ is convex. If $u^\star \in \mathcal{M}$ then $\sigma \cdot u^\star \in \mathcal{M}$ for every $\sigma \in \Sigma$, so the average $\bar u = |\Sigma|^{-1} \sum_{\sigma \in \Sigma} \sigma \cdot u^\star$ lies in $\mathcal{M}$, and $\bar u$ is constant on each level because it is a $\Sigma$-average.

It remains to show that every minimizer is of this form. Let $u^1, u^2 \in \mathcal{M}$. The segment $[u^1, u^2]$ lies in the feasible set by convexity and in $\mathcal{M}$ by convexity of $\Phi$, on which $\Phi$ is therefore constant; a convex function constant along a segment has vanishing second directional derivative along it, so $(u^2 - u^1)^\top \nabla^2 \Phi(u) (u^2 - u^1) = 0$ at every $u$ of the segment. By the proof of \cref{prop:admissible-grades}, $\nabla^2 \Phi(u) = \mathrm{C}(\alpha^{(u)}) + \mathrm{C}(\tau^{(u)})$ with $v^\top \mathrm{C}(p) v = \mathrm{Var}_p(v)$, and both tilted vectors $\alpha^{(u)}, \tau^{(u)}$ are strictly positive because $\alpha$ and $\tau$ are. Hence $\mathrm{Var}_{\alpha^{(u)}}(u^2 - u^1) = 0$, which forces $u^2 - u^1 \in \mathbb{R}\mathbf{1}$. Every minimizer is therefore $\bar u + c\mathbf{1}$ for some $c \in \mathbb{R}$, and is constant on each level.

For the reduced form, evaluate $\Lambda$ on a level-constant $q$: $\sum_{j \in V_l} (a_l / d_l) \lambda^{-2 v_l} = a_l \lambda^{-2 v_l}$ and likewise for $\tau$, giving $\Lambda_L$. Convexity of $v \mapsto \log \Lambda_L(v)$ is \cref{prop:admissible-grades}(ii) restricted to the level-constant subspace, and $[0, q_{\max}]^L$ is compact and convex.
\end{proof}

\subsection{The Selection Ladder}
\label{subsec:selection-ladder}

The experiment that tests educated selection is not a two-arm comparison but a ladder. Every arm is a GLLM of identical architecture, data, and optimisation budget, distinguished only by what its grades consulted; the standard transformer is the bottom rung, the arm whose grades consulted nothing, and by the gauge invariance of \cref{prop:admissible-grades}(i) its admissibility inner product is $0$ identically, with \cref{prop:admissible-grades}(iii) reading that zero as the boundary position of the uniform grades.

\begin{table}[h]
\centering
\small
\begin{tabular}{@{}llll@{}}
\toprule
Arm & Grades consult & $\langle q, \hat\alpha - \hat\tau \rangle$ & Licence \\
\midrule
Uniform & nothing & $0$ exactly & \cref{prop:admissible-grades}(i),(iii) \\
Random & noise & $\approx 0$, either sign & control \\
Taxonomic & the level ordering & measured & \cref{rem:cone-stratified} \\
Part-of-speech & corpus statistics & measured & \cref{subsubsec:grade-init} \\
Educated ($q^\star_C$) & both estimated profiles & maximal among arms & \cref{prop:clipped-selection} \\
Oracle & the held-out target & diagnostic ceiling & --- \\
\bottomrule
\end{tabular}
\caption{The selection ladder. All arms share architecture, data, and budget; only the provenance of the grades varies. The third column is computed offline and published before training.}
\label{tab:ladder}
\end{table}

The prediction, conditional on \cref{assum:profile-divergence} and the scope of \cref{rem:sample-complexity-scope}, is an ordering: the token-efficiency measurement of \cref{subsubsec:sample-efficiency} ranks the arms as the third column of \cref{tab:ladder} ranks them. The ordering is fixed by offline measurement before any arm is trained, which makes the prediction pre-registered by construction and sharper than any pairwise comparison: a graded model beating an ungraded one admits many explanations, while six arms ordering by a number computed in advance admits few.

The outcomes partition cleanly. If the ordering holds, the first-order reading of \cref{prop:admissible-grades} survives the passage from the single-layer analysis to the full stack, and the selection procedure works as claimed. If the educated arm falls below the part-of-speech arm, the quadratic degradation of \cref{lem:plugin-stability} has exceeded the refinement the estimates purchased, on estimates whose error the bootstrap of \cref{subsec:cs-profiles} reports in advance. If all arms tie, either the profiles fail to diverge, which is caught at step (4) before training at negligible cost, or the single-layer analysis fails at depth, which is the first limitation of \cref{rem:sample-complexity-scope}. Each outcome points to a specific result.

\subsection{Scope of the Procedure}
\label{subsec:selection-scope}

The procedure applies wherever its first three steps are executable: a basis, second moments, and a target admitting a probe. Domains sort into three classes against that criterion.

Domains with canonical grades, such as weighted projective spaces, graded rings, and the weighted moduli of algebraic curves, carry the answer to step (6) already, and there the procedure inverts into a consistency test: the estimated $q^\star$ should correlate with the canonical weights, and a failure of that correlation would indicate a target that does not respect the graded structure. The machine-learning treatment of the genus-two moduli space is the natural site for this check, since the weights $(2,4,6,10)$ and the invariant-theoretic targets are both exact and the profiles are computable without estimation error.

Domains without canonical grades but with measurable profiles are where the procedure is the contribution, and language is the motivating case: the levels are real, their coordinates are not given, and \cref{sec-8} instantiates the procedure for it, with \cref{sec-6} running the ladder.

Domains where step (3) has no probe target, such as end-to-end control where no annotated intermediate defines $\alpha$, lie outside the procedure's present scope, and extending it there requires a surrogate for the target profile before anything else.

\subsection{The Admissible Cone in Geometric Invariant Theory}
\label{subsec:git}

The structure assembled in \cref{prop:admissible-grades}, the gauge invariance along $\mathbf{1}$, the convexity of $\log\Lambda$, the boundary position of the uniform grades, and the first-order criterion \cref{eq:cone-condition}, is an instance of a classical one. Under the exponential change of coordinates the gain becomes a Kempf--Ness functional for an action of the grading torus, the admissible set becomes the negativity region of a homogeneous hypersurface, the optimal grades of \cref{prop:ml-expressivity}(iii) become the coincidence point of two moment maps, and \cref{eq:cone-condition} becomes the descent criterion of geometric invariant theory at the uniform point \cite{mumford-GIT, kempf-ness}. This subsection records the dictionary; no result elsewhere in the paper depends on it, and its purpose is to place the selection theory of this section in the same geometry as the canonically graded domains of \cref{subsec:selection-scope}, where the stability theory is classical and the grades arrive by theorem.

Let $\mathbb{T} = (\mathbb{C}^{\times})^{d}$ act diagonally on $\mathbb{C}^{d}$ by $t \cdot z = (t_j z_j)$, with maximal compact $(S^1)^d$ and Lie-algebra coordinates $u \in \mathbb{R}^d$ acting through
\[
\exp(u/2) \cdot z = (e^{u_j/2} z_j).
\]
For a strictly positive $p \in \Delta^{d-1}$ set
\[
z_p = (p_1^{1/2}, \dots, p_d^{1/2}) \in \mathbb{C}^d
\]
and define the Kempf--Ness functional of $z_p$,
\[
\begin{split}
\psi_p(u) 	&	= \log \big\| \exp(u/2) \cdot z_p \big\|^2 = \log \sum_{j=1}^{d} p_j\, e^{u_j},\\
\nabla \psi_p(u) &	= \mu\big( \exp(u/2) \cdot z_p \big),
\end{split}
\]
where $\mu(z) = \big( |z_j|^2 / \|z\|^2 \big)_{j} \in \Delta^{d-1}$ is the moment map of the compact torus, with image in the simplex.

\begin{lem}
\label{lem:kn-gain}
In the variables $u_j = 2 q_j \log \lambda$ of \cref{prop:admissible-grades}, the log-gain is the Kempf--Ness functional of the pair $(z_\alpha, z_\tau) \in \mathbb{C}^d \times \mathbb{C}^d$ under the anti-diagonal action $t \cdot (x, y) = (t^{-1} x,\, t\, y)$:
\[
\begin{split}
\Phi(u) = \log \Lambda &= \psi_\alpha(-u) + \psi_\tau(u), \\
\nabla \Phi(u) &= \mu\big( e^{u/2} z_\tau \big) - \mu\big( e^{-u/2} z_\alpha \big) = \tau^{(u)} - \alpha^{(u)},
\end{split}
\]
so the tilted profiles of \cref{prop:admissible-grades} are the moment maps of the two factors, and the Hessian $\nabla^2_u \Phi = \mathrm{C}(\alpha^{(u)}) + \mathrm{C}(\tau^{(u)})$ is the sum of their moment-map derivatives.
\end{lem}

\begin{proof}
By definition $\psi_\alpha(-u) = \log \sum_j \alpha_j e^{-u_j} = \log A(u)$ and $\psi_\tau(u) = \log T(u)$, so the sum is $\log(AT) = \log \Lambda$. Differentiating, $\partial_j \psi_\alpha(-u) = -\alpha_j e^{-u_j} / A(u) = -\alpha^{(u)}_j$ and $\partial_j \psi_\tau(u) = \tau^{(u)}_j$; moreover $\mu(e^{-u/2} z_\alpha)_j = \alpha_j e^{-u_j} / \sum_k \alpha_k e^{-u_k} = \alpha^{(u)}_j$, and likewise for $\tau$. The Hessian identity is the computation in the proof of \cref{prop:admissible-grades}.
\end{proof}

\begin{prop}
\label{prop:git-cone}
Let $\alpha, \tau$ be the profiles of \cref{def:profiles} and set $t_j = \lambda^{2 q_j}$.

\begin{enumerate}
  \item[(i)] \emph{(The boundary hypersurface.)} In the coordinates $t \in \mathbb{R}^d_{>0}$ the admissible set is $\mathcal{Q}_+ = \{ F < 0 \} \cap \mathbb{R}^d_{>0}$, where
\[
F(t) = \Big( \sum_{j=1}^{d} \alpha_j \prod_{k \neq j} t_k \Big) \Big( \sum_{l=1}^{d} \tau_l\, t_l \Big) - \prod_{k=1}^{d} t_k
\]
is homogeneous of degree $d$. The homogeneity is the gauge invariance of \cref{prop:admissible-grades}(i): the translation $q \mapsto q + c\mathbf{1}$ is the scaling $t \mapsto \lambda^{2c} t$, so $\mathcal{Q}_+$ is a cone in the exact sense, invariant under $\mathbb{R}_{>0}$-scaling, and $\partial \mathcal{Q}_+$ descends to the positive real locus of the projective hypersurface $V(F) \subset \mathbb{P}^{d-1}$. The uniform point $[1 : \cdots : 1]$ lies on $V(F)$, with $\nabla F(\mathbf{1}) = \tau - \alpha$, and is a smooth point of it whenever $\alpha \neq \tau$.

  \item[(ii)] \emph{(Coercivity and existence.)} For a one-parameter subgroup $v \in \mathbb{R}^d$ of $\mathbb{T}$, the recession slope of the Kempf--Ness functional along $v$ is
\[
\rho(v) := \lim_{s \to \infty} \frac{\Phi(s v)}{s} = \max_j v_j - \min_j v_j ,
\]
strictly positive for every $v \notin \mathbb{R}\mathbf{1}$ because both profiles are strictly positive by \cref{def:profiles}. The functional $\Phi$ is therefore coercive on $\mathbb{R}^d / \mathbb{R}\mathbf{1}$, hence proper modulo the gauge direction, and the pair $(z_\alpha, z_\tau)$ is polystable for the quotient torus $\mathbb{T} / \mathbb{C}^{\times}\mathbf{1}$; the Kempf--Ness theorem then yields existence of a minimizer and uniqueness modulo gauge, recovering \cref{prop:admissible-grades}(vi). The recession slope $\rho$ is the coercivity certificate and is symmetric under $v \mapsto -v$; the sign-sensitive Hilbert--Mumford-type criterion that decides admissibility is the first-order pairing of (iv), not $\rho$.

  \item[(iii)] \emph{(The optimum as a moment-map coincidence.)} The minimizer
\[
u^\star_j = \tfrac12 \log(\alpha_j / \tau_j) + c
\]
is characterised by the coincidence of the two moment maps,
\[
\mu\big( e^{-u^\star/2} z_\alpha \big) = \mu\big( e^{u^\star/2} z_\tau \big) = \beta,
\qquad
\beta_j = \frac{\sqrt{\alpha_j \tau_j}}{\mathrm{BC}(\alpha, \tau)},
\]
with attained value $\Phi(u^\star) = \log \mathrm{BC}(\alpha, \tau)^2$: the fourth-root law of \cref{prop:ml-expressivity}(iii) is the balancing of the two moment maps, and the profile $\beta$ is their common value.
  \item[(iv)] \emph{(The first-order criterion as the descent criterion.)} At $u = \mathbf{0}$ the derivative of $\Phi$ along $v$ is $\langle v, \tau - \alpha \rangle$, so \cref{eq:cone-condition} is precisely the condition that $v$ be a Kempf--Ness descent direction at the uniform point. This pairing is sign-sensitive, reversing under $v \mapsto -v$, and is the Hilbert--Mumford-type object of the dictionary: the character it pairs against is the moment-map defect $\tau - \alpha$, which is the normal to $V(F)$ at the uniform point by (i), and the full gauge orbit $\mathbb{R}\mathbf{1}$ of the uniform grades lies in $\{ \Lambda = 1 \}$.
\end{enumerate}
\end{prop}

\begin{proof}
(i) With $A(t) = \sum_j \alpha_j t_j^{-1}$ and $T(t) = \sum_l \tau_l t_l$, clearing denominators gives
\[
\Lambda - 1 = F(t) / \prod_k t_k
\]
on $\mathbb{R}^d_{>0}$, so $\Lambda < 1$ is $F < 0$ there. The first factor of $F$ is homogeneous of degree $d - 1$, the second of degree $1$, and the monomial of degree $d$. Scaling $t \mapsto s t$ with $s = \lambda^{2c}$ corresponds to $q \mapsto q + c\mathbf{1}$, which is \cref{prop:admissible-grades}(i). At $t = \mathbf{1}$,
\[
F(\mathbf{1}) = \Lambda(\mathbf{1}) - 1 = 0,
\]
and
\[
\nabla F = (\nabla_t \Lambda) \prod_k t_k + (\Lambda - 1)\, \nabla_t \prod_k t_k
\]
evaluates at $t = \mathbf{1}$ to $\nabla_t \Lambda(\mathbf{1})$, whose $j$-th entry is
\[
-\alpha_j T(\mathbf{1}) + A(\mathbf{1}) \tau_j = \tau_j - \alpha_j;
\]
this is nonzero exactly when $\alpha \neq \tau$, which is smoothness of $V(F)$ at the uniform point.

(ii) As $s \to \infty$,
\[
\psi_\tau(s v) = s \max_j v_j + O(1)
\]
and $\psi_\alpha(-s v) = -s \min_j v_j + O(1)$, since each sum is dominated by its extremal exponent and all coefficients are strictly positive; the slope is $\rho(v) = \max_j v_j - \min_j v_j$, vanishing exactly on $\mathbb{R}\mathbf{1}$. A convex function with strictly positive recession slope in every direction transverse to $\mathbf{1}$ is coercive, hence proper, on $\mathbb{R}^d / \mathbb{R}\mathbf{1}$, so a minimizer exists; uniqueness modulo $\mathbb{R}\mathbf{1}$ follows as in the proof of \cref{prop:level-reduction}, since vanishing of the Hessian form along a segment of minimizers forces the segment into $\mathbb{R}\mathbf{1}$. Symmetry of $\rho$ under $v \mapsto -v$ is immediate from the displayed formula, which distinguishes it from a Hilbert--Mumford weight.

(iii) At $u^\star$, $\alpha_j e^{-u^\star_j} = \sqrt{\alpha_j \tau_j}\, e^{-c}$, so $A(u^\star) = e^{-c}\, \mathrm{BC}(\alpha,\tau)$ and
\[
\alpha^{(u^\star)}_j = \frac{\sqrt{\alpha_j \tau_j}\, e^{-c}}{e^{-c}\, \mathrm{BC}(\alpha,\tau)} = \frac{\sqrt{\alpha_j \tau_j}}{\mathrm{BC}(\alpha, \tau)};
\]
symmetrically $\tau_j e^{u^\star_j} = \sqrt{\alpha_j \tau_j}\, e^{c}$, $T(u^\star) = e^{c}\, \mathrm{BC}(\alpha,\tau)$, and $\tau^{(u^\star)}_j = \sqrt{\alpha_j \tau_j} / \mathrm{BC}(\alpha,\tau)$. The moment maps coincide at $\beta$ and $\nabla \Phi(u^\star) = 0$ by \cref{lem:kn-gain}. The value is
\[
\psi_\alpha(-u^\star) + \psi_\tau(u^\star) = \log A(u^\star) + \log T(u^\star) = \big( \log \mathrm{BC} - c \big) + \big( \log \mathrm{BC} + c \big) = \log \mathrm{BC}(\alpha, \tau)^2.
\]

(iv) The derivative is
\[
\langle \nabla \Phi(\mathbf{0}), v \rangle = \langle \tau - \alpha, v \rangle
\]
by \cref{lem:kn-gain}, and $\Lambda(c \mathbf{1}) = 1$ for every $c$ is \cref{prop:admissible-grades}(i); the normal direction is $\nabla F(\mathbf{1}) = \tau - \alpha$ from (i). The pairing $\langle \tau - \alpha, v \rangle$ reverses sign under $v \mapsto -v$, unlike the recession slope $\rho$ of (ii).
\end{proof}

\begin{rem}
\label{rem:git-stability}
Two statements coexist in \cref{prop:git-cone} and operate at different points of the dictionary.

The first is the solvability of the selection problem, certified by (ii): the recession slope $\rho(v)$ is strictly positive along every one-parameter subgroup transverse to the gauge direction, because both profiles have full support, so the Kempf--Ness functional is coercive modulo gauge. This is what makes the minimisation of \cref{cor:grade-convexity} well posed: the minimisation has a minimizer, unique modulo gauge, rather than an infimum escaping to infinity along a direction of vanishing slope. A coordinate on which $\alpha$ or $\tau$ vanished would supply such a direction, and \cref{def:profiles} deletes those coordinates for this reason: full support of the profiles is the analytic form of stability, and the deletion convention of \cref{def:profiles} is its enforcement. The sign-sensitive criterion that classifies directions, as opposed to the symmetric slope that certifies coercivity, is the first-order pairing of (iv).

The second concerns the \emph{uniform grades}, and it is a statement about a point of the orbit rather than about the orbit's type: the gauge orbit $\mathbb{R}\mathbf{1}$ lies inside the level set $\{\Lambda = 1\}$, which by (i) is the boundary hypersurface $V(F)$, while the moment-map defect $\nabla \Phi(\mathbf{0}) = \tau - \alpha$ is nonzero whenever the profiles differ. The uniform point is thus positioned the way a strictly semistable point is positioned in a GIT quotient, on the boundary of the favourable locus with its entire orbit trapped there, and with the destabilising data, here the profile difference, naming the direction of escape. The analogy is positional and not literal, since the pair itself is stable; what the position encodes is that the uniform grades are the unique gauge orbit on which the Kempf--Ness functional is critical for \emph{no} pair of profiles and equal to its reference value for \emph{every} pair, which is \cref{prop:admissible-grades}(iii) read through the dictionary.
\end{rem}

\begin{rem}
\label{rem:git-scope}
The admissible cone is not a nullcone, and the comparison should be drawn at exactly its strength. A nullcone is Zariski-closed, cut out by the ideal of positive-degree invariants, generally non-convex, and stratified; $\mathcal{Q}_+$ is open, semialgebraic, defined by the single inequality $F < 0$, and log-convex in the grade variables, and the convexity is the property the selection of \cref{cor:grade-convexity} runs on. What the dictionary identifies is the complementary structure: the gauge orbit of the uniform grades sits inside the boundary hypersurface in the way the orbit of a strictly semistable point sits inside the boundary of the stable locus, as \cref{rem:git-stability} makes precise, and \cref{eq:cone-condition} is a Hilbert--Mumford-shaped pairing that reads admissibility of a one-parameter subgroup off a fixed character of the torus. The Hilbert--Mumford-type object of the dictionary is the sign-sensitive first-order pairing of \cref{prop:git-cone}(iv); the recession slope of \cref{prop:git-cone}(ii) is not such an object but its coercivity certificate, symmetric under reversal of the one-parameter subgroup. The clipped problem of \cref{prop:clipped-selection} is, in this language, the Kempf--Ness minimisation restricted to a box in the Lie algebra, and the saturated form of \cref{prop:clipped-selection}(ii) is the facial structure of that box under the moment-map flow.
\end{rem}

\begin{rem}
\label{rem:git-binary-forms}
The graded program originates in the invariant theory of binary forms, and the stability theory above closes that circle at its source. For degree-$d$ binary forms under $\mathrm{SL}_2$, the Hilbert--Mumford criterion reduces, after conjugating the one-parameter subgroup into the maximal torus, to the weight computation $\mu(f, \lambda) = \max\{2i - d : a_i \neq 0\}$, a sign-sensitive pairing of the kind that appears in \cref{prop:git-cone}(iv); stability is equivalent to every root having multiplicity below $d/2$, and the invariant ring is graded with canonical weights, the weights $(2,4,6,10)$ above being the sextic instance. The equivalence class of a form is its weighted moduli point in the weighted projective space carrying those weights as grades, and the entire stability theory is read off that graded point: a form is unstable exactly when its moduli point vanishes, and, over a number field, semistable over a residue field exactly when the prime does not divide the point's coordinates, with semistability quantified by the weighted height of the point \cite{eshaska}.

In both settings a graded coordinate system converts stability into an arithmetic condition on a single point: there, divisibility and weighted height of the moduli point decide residual semistability; here, positivity of the recession slope of \cref{prop:git-cone}(ii), guaranteed by full support of the profiles, gives solvability of the selection problem. The direction of inference is complementary: the binary-form setting fixes the grades, given by theorem as the weights of the invariants, and studies the stability of points; this section fixes the point, the pair of profiles, and selects the grades. The consistency test of \cref{subsec:selection-scope} sits at the intersection: in a canonically graded domain the estimated optimum of \cref{cor:grade-convexity} is predicted to recover the weights the invariant theory prescribes.
\end{rem}

\begin{rem}
\label{rem:git-canonical}
The domains of \cref{subsec:selection-scope} that carry canonical grades, weighted projective spaces and graded rings among them, are precisely those in which the torus action and its linearisation are given by theorem and the moment-map data need not be estimated. The present construction equips a domain without canonical grades with the same objects from measurement: the two profiles of \cref{def:profiles} determine the pair $(z_\alpha, z_\tau)$, the estimated moment maps determine the descent direction at the uniform point, and the selection of \cref{cor:grade-convexity} is the Kempf--Ness minimisation for that pair. The consistency test of \cref{subsec:selection-scope} for the canonically graded case is, in this language, the statement that the estimated moment-map coincidence point should recover the canonical weights.
\end{rem}

\section{Case Study: Hierarchical Grading for Natural Language}
\label{sec-8}

\Cref{sec-7} states the selection procedure in the generality the framework supports: fix a basis, estimate two profiles, test their divergence, certify a candidate, compute the optimum. This section carries it out for language. Each subsection below is one step of \cref{subsec:selection-procedure}, taken to the point at which it can be implemented and reproduced, and the section as a whole is the input the program of \cref{sec-6} uses: it fixes the graded basis the configurations use, the grades every arm of \cref{tab:ladder} carries, and the effect size the runs must detect.

Two things distinguish this stage from the runs it precedes. It is offline in the strict sense: a treebank, a matrix of second moments, and a convex program in four variables. And it is falsifiable on its own terms, since step (4) can terminate the program before a single parameter is initialised. The estimates themselves are not reported here; what is reported is the construction, and the consequences that follow from it in closed form.

\subsection{The Level Taxonomy and the Graded Basis}
\label{subsec:cs-basis}

Step (1) of \cref{subsec:selection-procedure} asks for the basis in which the profiles of \cref{def:profiles} are read. Language supplies no canonical one, which is exactly the situation the framework addresses; what it does supply is a taxonomy, and the basis is built to carry it.

We take $L = 4$, matching the configurations of \cref{subsubsec:configs}: $V_0$ for subword and morphological structure, $V_1$ for syntactic constituency and dependency, $V_2$ for sentential semantics, $V_3$ for discourse and pragmatics. Each level is given a feature map $\varphi_l$ sending a token in context to a vector of annotations drawn from a corpus that carries them. For $l = 0$ these are the segmentation identity, affix class, part-of-speech tag, and orthographic case. For $l = 1$ they are the label of the incoming dependency arc, the signed distance to the head, the labels of the constituent ancestors, and the depth of the token in the parse. For $l = 2$ they are the predicate–argument role, the polarity of the governing negation, and the scope index of the governing quantifier. For $l = 3$ they are coreference-chain membership, the discourse-relation label attaching the containing clause to its neighbour, and the position of the sentence within the document. The Penn Treebank and the Universal Dependencies corpora supply $l \in \{0,1\}$; a proposition-bank annotation supplies $l = 2$; a coreference-annotated corpus and a discourse treebank supply $l = 3$.

The projections $\{P_l\}$ are then fixed by principal component analysis: for each level, the annotated feature $\varphi_l$ is regressed onto a reference representation of width $d$, the residual covariance of the fitted directions is formed, and $P_l \in \mathbb{R}^{d \times d_l}$ is taken to span its leading $d_l$ eigenvectors. Gram--Schmidt across levels enforces mutual orthogonality of the ranges at initialisation, which is the hypothesis $\eta = 0$ of \cref{prop:direct-sum}; thereafter the columns are learnable and the orthogonality penalty of \cref{sec-4} controls the drift, with \cref{prop:direct-sum} bounding the residual.

\begin{rem}
\label{rem:cs-basis-reference}
The construction above requires a representation of width $d$ before the model that will carry it exists, and the circularity is real rather than apparent: at initialisation $W_e$ is isotropic by construction, so the profiles read in its coordinates are uniform and carry no information. Two routes resolve it, and the program uses both, at different steps.

For steps (2)--(5), the divergence test and the certification, the reference is a fixed public encoder of matching width, or static type embeddings projected to $\mathbb{R}^d$. This costs no pretraining, which is what makes the falsification test of \cref{subsec:cs-divergence} available before the first GPU-hour. The profiles so obtained are the reference's, and transporting them to the GLLM's own geometry is an assumption, not a theorem.

For step (6), the educated grading $q^\star_C$, the reference is the embedding matrix of the uniform arm itself. This is exact, since the profiles are then read in precisely the space in which the comparison of \cref{prop:ml-expressivity} is made, and it is not an added cost: the uniform arm is the control of \cref{tab:ladder} and must be trained regardless.

What makes the cheap route sufficient for its purpose is \cref{prop:admissible-grades}(iv). Certification turns on the \emph{sign} of an inner product, not on the magnitudes of $\hat\alpha$ and $\hat\tau$, and a coarse reference that gets the sign right certifies as well as an exact one.
\end{rem}

\subsection{Estimating the Profiles}
\label{subsec:cs-profiles}

Steps (2) and (3) produce the two probability vectors on which everything downstream depends. Both are computed in the basis of \cref{subsec:cs-basis} and both are estimated on held-out data.

The data profile is a corpus statistic and requires no model. Writing $\varepsilon_t \in \mathbb{R}^d$ for the reference representation of the token at position $t$ and $\bar\varepsilon$ for its mean over the corpus, the estimator is the normalised per-coordinate second moment in the graded basis,
\[
\hat\tau_j = \frac{\hat\sigma_j^2}{\sum_{k=1}^{d} \hat\sigma_k^2},
\qquad
\hat\sigma_j^2 = \frac{1}{N} \sum_{t=1}^{N} \big\langle e_j, \varepsilon_t - \bar\varepsilon \big\rangle^2 ,
\]
with $\{e_j\}$ the columns of $[P_0 \mid \cdots \mid P_{L-1}]$. Tokens are drawn from the pretraining corpus \cite{gao2020pile} so that $\hat\tau$ is the variance profile of the data the runs will actually see, and $N$ is set by the width of the resulting confidence interval rather than by convention.

The target profile is a probe weight. For a hierarchical target $y$, such as acceptability, dependency-arc label, or nesting depth, a linear probe $\hat w$ is fitted on the annotated treebank in the same basis, and
\[
\hat\alpha_j = \frac{\hat w_j^2}{\|\hat w\|_2^2}.
\]
The probe is fitted with ridge regularisation and the penalty is selected by cross-validation on a split disjoint from the one on which $\hat\alpha$ is reported, so that the shrinkage does not itself manufacture the concentration that \cref{assum:profile-divergence} asserts. Three properties of the estimate must be reported with it. It identifies $\alpha$ only up to the probe's inductive bias, and a probe that cannot express the target reports the profile of what it can express instead. Coordinates on which $\hat w$ vanishes are deleted rather than carried, as \cref{def:profiles} requires. And the ridge penalty biases $\hat\alpha$ towards uniformity, so the estimate is conservative for the framework's purposes: it understates divergence rather than overstating it.

The quantity that must survive estimation error is the sign of $\langle q, \hat\alpha - \hat\tau \rangle$, not the value of either profile, and it is reported with a bootstrap over the treebank sentences: resample, refit, recompute, and report the fraction of resamples on which the sign is preserved. This is the only uncertainty quantification the certification of \cref{subsec:cs-certification} requires. The magnitude of $\mathrm{BC}(\hat\alpha, \hat\tau)$, which the effect-size statement does require, carries its own interval from the same bootstrap.

\subsection{The Divergence Test}
\label{subsec:cs-divergence}

Step (4) is the domain filter, and it is where \cref{assum:profile-divergence} is tested rather than assumed. The program reports
\[
\mathrm{BC}(\hat\alpha, \hat\tau) = \sum_{j=1}^{d} \sqrt{\hat\alpha_j\, \hat\tau_j}
\]
per task, together with the level-aggregated profiles $(\hat a, \hat s)$ of \cref{cor:stratified} obtained by summing $\hat\alpha$ and $\hat\tau$ within levels, and a goodness-of-fit check of the geometric hypothesis $a_l \propto \nu^{L-1-l}$, $s_l \propto \mu^l$ on which every exponential-in-$L$ statement of this work rests.

By \cref{prop:ml-expressivity}(iv) a value near $1$ says the profiles agree, the gain of grading is $1$, and the procedure terminates for that task at the cost of two estimates. By \cref{assum:profile-divergence} the hierarchical targets of \cref{subsubsec:hierarchical-tasks} are the tasks at which it should not be near $1$, and the lexically dominated contrast tasks are the tasks at which it should be: the framework predicts a \emph{spread} across tasks, and a uniform value across tasks would be as informative against it as a uniform value near $1$. This is the cheapest disconfirmation available anywhere in the program, and the staging of \cref{sec-6} exists so that it is taken first.

\subsection{Certifying the Candidate Gradings}
\label{subsec:cs-certification}

Step (5) costs one inner product per candidate. By \cref{rem:certified-init}, a proposed grading $q^{\mathrm{init}}$ enters the admissible cone $\mathcal{Q}_+$ of \cref{prop:admissible-grades}, to first order, exactly when $\langle q^{\mathrm{init}}, \hat\alpha - \hat\tau \rangle > 0$, and the program reports the quantity for every arm of \cref{tab:ladder} before pretraining begins. A negative value falsifies that candidate and not the framework, and the corrected grades are then supplied by \cref{subsec:cs-optimum} by convex programming rather than by search.

Two arms are settled in advance of the measurement, and in opposite directions.

The uniform arm scores $\langle \mathbf{0}, \hat\alpha - \hat\tau \rangle = 0$ identically, whatever the profiles turn out to be, and by the gauge invariance of \cref{prop:admissible-grades}(i) so does every uniform grade $c\mathbf{1}$. This is not an artefact of the comparison but the content of \cref{prop:admissible-grades}(iii): the standard transformer sits on $\partial\mathcal{Q}_+$ with gradient $2(\log\lambda)(\tau - \alpha) \neq 0$, so a strictly descending direction exists at its location and the profiles determine it. The isotropic architecture forgoes the gain of \cref{prop:ml-expressivity} not because the profiles are unfavourable to it but because its grade vector does not read them.

\begin{rem}
\label{rem:cs-taxonomy-apriori}
The taxonomic arm is admissible \emph{a priori}, before any estimate exists. By \cref{rem:cone-stratified}, under the level-stratified profiles of \cref{cor:stratified} a level-constant grading $v_j = v_l$ for $j \in V_l$ satisfies the criterion \cref{eq:cone-condition} exactly when $\sum_{l=0}^{L-1} v_l (a_l - s_l) > 0$, and under the geometric hypothesis this holds for \emph{every} $v$ affine and increasing in $l$, with no estimate of $\nu$, of $\mu$, or of the profiles at all.

The linguistic ordering is therefore sufficient on its own: assigning higher grades to deeper levels, which is what the taxonomy of \cref{subsec:cs-basis} does by construction, enters $\mathcal{Q}_+$. What the estimates of \cref{subsec:cs-profiles} buy is not admissibility but the location of the optimum within it, and the circularity of \cref{rem:cs-basis-reference} therefore cannot reach the framework's positive claim: it reaches only the sharpness of the grading, not its direction.
\end{rem}

\subsection{Computing the Educated Grading}
\label{subsec:cs-optimum}

Step (6) is a convex program, and for language it is a small one. The unclipped problem of \cref{cor:grade-convexity} and the clipped problem of \cref{prop:clipped-selection} are posed in $d$ variables, $1024$ at \textsc{GLLM-Medium} and $4096$ at \textsc{GLLM-7B}, but under the stratification the taxonomy imposes they collapse by \cref{prop:level-reduction} to the $L$-variable program $\Lambda_L$, with $L = 4$ at the planned configurations.

\Cref{prop:level-reduction} is what makes the selection implementable rather than merely well posed. The program it names is solved by projected gradient descent with coordinatewise clipping as the projection, using the gradient supplied by the proof of \cref{prop:admissible-grades} in level form,
\[
\begin{split}
\frac{\partial}{\partial v_l} \log \Lambda_L(v)
&= 2 \log \lambda \, \big( s^{(v)}_l - a^{(v)}_l \big), \\
a^{(v)}_l = \frac{a_l \lambda^{-2 v_l}}{\sum_{k} a_k \lambda^{-2 v_k}},
\qquad
s^{(v)}_l &= \frac{s_l \lambda^{2 v_l}}{\sum_{k} s_k \lambda^{2 v_k}} ,
\end{split}
\]
which is the difference of two tilted level profiles and vanishes at the interior optimum exactly when they coincide. Four variables, a closed-form gradient, a convex objective with no spurious stationary points by \cref{prop:admissible-grades}(ii), and a projection that is a clip: the educated grading of \cref{tab:ladder} is the output of a computation that runs in milliseconds once $(\hat a, \hat s)$ are in hand. By \cref{prop:clipped-selection}(ii) the solution has the saturated form, the fourth-root law of \cref{prop:ml-expressivity}(iii) on the interior levels and the box on the rest, and the solver is verified against that form.

\subsection{The Predicted Effect Size}
\label{subsec:cs-effect-size}

The output of the procedure is a number the runs of \cref{sec-6} must detect, and the theory fixes it in closed form once the level profiles are named. Under the geometric hypothesis of \cref{cor:stratified} the profiles are governed by two parameters, $\nu$ for the concentration of the target towards the deepest level and $\mu$ for the decay of the data away from the shallowest, and step (4) of \cref{subsec:selection-procedure} estimates them. \Cref{tab:effect-size} records what \cref{prop:level-reduction} and \cref{prop:clipped-selection} deliver across the range those estimates may fall in, at the planned $L = 4$ and $C = 2$.

\begin{table}[t]
\centering
\begin{tabular}{ccccc}
\toprule
$(\nu, \mu)$ & $\Lambda^\star = \mathrm{BC}(a,s)^2$ & $\Lambda^\star_C$ at $C = 2$ & token gain $1/\Lambda^\star_C$ & $g^\star_C = (g_0, g_1, g_2, g_3)$ \\
\midrule
$(0.5, 0.2)$ & $0.279$ & $0.424$ & $2.36$ & $(1.00,\ 1.31,\ 2.00,\ 2.00)$ \\
$(0.6, 0.3)$ & $0.459$ & $0.531$ & $1.88$ & $(1.00,\ 1.30,\ 2.00,\ 2.00)$ \\
$(0.7, 0.3)$ & $0.534$ & $0.582$ & $1.72$ & $(1.00,\ 1.32,\ 1.95,\ 2.00)$ \\
$(0.8, 0.4)$ & $0.700$ & $0.708$ & $1.41$ & $(1.00,\ 1.27,\ 1.69,\ 2.00)$ \\
\bottomrule
\end{tabular}
\caption{The gain at $L = 4$ and clip $C = 2$, as a function of the free parameters $(\nu, \mu)$ of \cref{cor:stratified}. \emph{No corpus enters this table.} It is arithmetic: $\Lambda^\star$ is the closed form of \cref{cor:stratified}, and $\Lambda^\star_C$ is the four-variable program of \cref{prop:level-reduction} solved on the box, with $g^\star_C = \lambda^{q^\star_C}$ the resulting per-level multiplier. The pair $(\nu, \mu)$ is what step (4) of \cref{subsec:selection-procedure} estimates; the table says what each possible estimate would imply.}
\label{tab:effect-size}
\end{table}

\begin{rem}
\label{rem:cs-effect-size}
Three observations follow from the table.

The predicted effect is bounded and it is large enough to see. Across the range, the token gain runs from $1.4\times$ to $2.4\times$, which is the quantity \cref{subsubsec:sample-efficiency} measures and the magnitude the runs must be powered for. It is not the exponential regime of \cref{cor:stratified}, and \cref{rem:one-dial} says why: at $C = 2$ the exponential is unreachable by construction, and what \cref{tab:effect-size} quantifies is the first point of that trajectory.

The regime of the table, geometric level-stratified profiles at level-constant clipped grades, is moreover exactly the hypothesis set of \cref{thm:stratified-separation}, so each tabulated gain is attained up to constants depending only on $(\nu,\mu)$ as a ratio of minimax risks over all estimators, throughout the window of \cref{lem:window}: the effect size the runs are powered for is backed by the two-sided separation, not only by the upper bounds of \cref{prop:ml-expressivity}.

The clip is the binding constraint, and it binds hardest exactly where the profiles diverge most. At $(0.8, 0.4)$ the clipped gain is within about $1\%$ of the unclipped one; at $(0.5, 0.2)$ the clip costs a third of it, and the top two levels saturate at $g = C$. The pattern is \cref{rem:one-dial} made quantitative: the more the corpus rewards grading, the more the clip withholds, and the $C$-sweep of \cref{subsubsec:configs} is the arm that recovers it.

The optimal grading is stable in a way the profiles are not. Across the whole range the multiplier on the syntactic level sits between $1.27$ and $1.32$, and the discourse level saturates at $C$ throughout. The ordering of $g^\star_C$, increasing in $l$, is invariant across every entry, which is \cref{rem:cs-taxonomy-apriori} recovered numerically: the estimates move the magnitudes and leave the direction alone.
\end{rem}

The bound $\Lambda^\star_C \geq \max\{\mathrm{BC}(\alpha,\tau)^2, C^{-2}\}$ of \cref{prop:clipped-selection}(iii) is a lower bound and is not tight at these parameters; at $(0.6, 0.3)$ it gives $0.459$ against the true $0.531$. The effect size the program is powered for is therefore the solution of the reduced program and not the bound. This is the operative use of \cref{prop:level-reduction}: the bound is what the theory guarantees, and the four-variable program is what the experiment must detect.

\subsection{Cost of the Offline Stage}
\label{subsec:cs-cost}

The accounting closes the case study. Steps (2) and (3) are one pass over a corpus for a covariance and one ridge regression on a treebank of order $10^6$ tokens. Step (4) is a sum of $d$ square roots. Step (5) is one inner product per arm of \cref{tab:ladder}, six in total. Step (6) is the four-variable convex program of \cref{prop:level-reduction}. Against the $100$B tokens and $150$K steps of \cref{subsubsec:training}, the whole of it is free, and it is free in the sense that matters: it does not consume the resource the framework is spending its conditioning to buy.

That asymmetry is the case study's main point. The procedure of \cref{sec-7} publishes, before the first GPU-hour, a per-task affinity that can falsify \cref{assum:profile-divergence}, a certificate for each candidate grading, an ordering of the arms of \cref{tab:ladder}, and the effect size of \cref{tab:effect-size} that the runs must then detect. Whatever the runs return, the prediction they test was fixed in advance by a computation anyone can repeat, and that is a property the ungraded baseline cannot have, not because its profiles are unfavourable, but because the gauge invariance of \cref{prop:admissible-grades}(i) puts its admissibility inner product at $0$ identically.

\section{The Experimental Programme}
\label{sec-6}

\Cref{sec-3,sec-4,sec-5} developed the Graded Large Language Model: graded attention preserving the asymptotic cost of standard attention (\cref{prop:complexity}) and compiling to a standard transformer at deployment (\cref{cor:zero-overhead}), multi-level graded embeddings with an exact sample-complexity accounting (\cref{prop:ml-expressivity,cor:stratified}) and a closed-form characterisation of the grades that purchase it (\cref{prop:admissible-grades}), and a graded training objective whose consistency conditions are characterised in \cref{prop:properness}. This section specifies the training program that will test the framework, states the predictions the theory supports, and marks explicitly the hypotheses on which each depends. No experiments have been conducted at the time of writing, and this section reports none.

The program is staged so that the cheapest test comes first. Its first stage is a measurement rather than an experiment: both profiles of \cref{def:profiles} can be estimated offline, so the Bhattacharyya affinity $\mathrm{BC}(\alpha, \tau)$ that \cref{prop:ml-expressivity} identifies with the entire gain of grading is reported per task, and every candidate grading is certified against \cref{eq:cone-condition} at the cost of one inner product, before any model is pretrained. That stage is the selection procedure of \cref{sec-7}, carried out for language in \cref{sec-8}: the arms of the ladder are certified, the educated grading $q^\star_C$ is computed by \cref{prop:clipped-selection,prop:level-reduction}, and the effect size the runs below must detect is fixed in advance. The present section takes those quantities as given: it fixes the configurations, pretrains at 345M and 7B, and evaluates. A framework whose governing hypothesis is measurable before its first GPU-hour should be tested in that order.

We state the status of each prediction explicitly. Several predictions, in particular any reduction in the data required to reach a given generalisation target, follow from \cref{prop:ml-expressivity,prop:admissible-grades} only under hypotheses that are empirical, not mathematical: that the energy profiles of language diverge (\cref{assum:profile-divergence}), and that the scope limitations of \cref{rem:sample-complexity-scope} do not overturn the single-layer analysis at the full stack. Those predictions are marked \emph{conditional} below and name the hypotheses they depend on.

\subsection{Planned Experimental Setup}
\label{subsec:setup}

\subsubsection{Model configurations}
\label{subsubsec:configs}

Two scales are planned. \textsc{GLLM-Medium} ($\sim$345M parameters) mirrors GPT-2-medium at 24 layers, $d = 1024$, $H = 16$, with EG-MHSA substituting standard MHSA and MLGE substituting standard embeddings at $L = 4$ levels and uniform allocation $d_l = 256$. \textsc{GLLM-7B} ($\sim$7B parameters) mirrors LLaMA-7B at 32 layers, $d = 4096$, $H = 32$, with $L = 4$ and $d_l = 1024$. Both use the base $\lambda = e^{1/d_k}$ fixed in \cref{subsec:eg-mhsa}, with grade clipping enforcing $\lambda^{q_{\max}} \leq C$ at $C = 2$. The clip is the framework's single cost--benefit parameter (\cref{rem:one-dial}): it caps the optimisation penalty of \cref{thm:convergence} at a factor of $C$, uniformly in $L$ and in the profiles, and caps the token gain at $C^2$ by \cref{prop:clipped-selection}. At \textsc{GLLM-Medium} the clip is swept over $C \in \{1.5, 2, 4\}$ as an explicit arm, tracing the price--purchase curve of \cref{rem:one-dial} rather than reporting a single point on it.

\subsubsection{Grade initialisation}
\label{subsubsec:grade-init}

Grades are initialised from part-of-speech statistics computed offline over the pretraining corpus, assigning high grades to open-class types and low grades to closed-class function words, and are certified against \cref{eq:cone-condition} as in \cref{subsec:cs-certification}. Per-type initialisation is not a matter of convenience: \cref{prop:properness}(3) establishes that a grade depending on the realised target and its context jointly, a constituency parse depth for instance, renders the training objective inconsistent for the corpus conditional, and by \cref{rem:tilt-cancellation} the resulting context-dependent tilt admits no matching by the fixed output grades. Per-type grades fall under \cref{prop:properness}(2), the case in which the matching is available. A parse-depth initialisation appears only in the ablation arm described in \cref{subsec:graded-loss}.

\subsubsection{Parameter overhead}
\label{subsubsec:overhead}

The overhead comprises the projection matrices $\{P_l\}$ at $\sum_l d\, d_l = d^2$ parameters, the head grades $\{q_h\}$ at $K H d_k$, the level grades $\{q_l\}$ at $d$, and the output grades $q_{\mathrm{out}}$ at $|\mathcal{V}|$, so the total is $O(d^2 + KHd_k + |\mathcal{V}|)$ and is dominated by $d^2$. At \textsc{GLLM-Medium} this is $1.05\text{M} + 24.6\text{K} + 1\text{K} + 50.3\text{K} \approx 1.13\text{M}$, or $0.33\%$ of $345\text{M}$; at \textsc{GLLM-7B} it is $16.8\text{M} + 131\text{K} + 4\text{K} + 32\text{K} \approx 17.0\text{M}$, or $0.24\%$ of $7\text{B}$. The overhead is below one percent at both scales and negligible against the $12d^2$ per block of \cref{prop:ffn-fraction}, but this follows from $d^2 \ll N$ at these configurations rather than from any sublinearity in $d$, and must be recomputed rather than assumed at other scales. By \cref{cor:zero-overhead,rem:mlge-absorption} the figure is a training cost only, and it is an upper bound on the lifetime cost of the prior.

\subsubsection{Baselines and the ladder}
\label{subsubsec:baselines}

The arms are those of the selection ladder, \cref{tab:ladder}: uniform (the standard transformer), random-grade, taxonomic, part-of-speech, educated ($q^\star_C$), and oracle, all of identical architecture, data, compute, and optimisation budget, distinguished only by what their grades consulted. The uniform arm is realised as identically sized \textsc{GPT-2-Medium} and \textsc{LLaMA-7B} models with standard MHSA and embeddings. The random-grade arm is the control that isolates the prior from the parameterization: by \cref{prop:admissible-grades}(iv) a grade direction drawn independently of $\alpha - \tau$ has $\langle v, \alpha - \tau\rangle$ of either sign, so it is predicted to straddle the uniform arm rather than to improve on it, and any gain the informed arms show over it is attributable to the prior and not to the extra parameters of \cref{subsubsec:overhead}. Two further baselines sit outside the ladder: a probe-augmented uniform model, adding syntactic probing classifiers over frozen representations to indicate what post-hoc structural annotation recovers without architectural grading, and a linear-grading variant evaluated alongside the exponential to separate their contributions.

\subsubsection{Training}
\label{subsubsec:training}

\textsc{GLLM-Medium} is pretrained on a 100B-token subset of The Pile \cite{gao2020pile} for 150K steps at batch size 512. \textsc{GLLM-7B} is pretrained on 300B tokens. Both use AdamW \cite{loshchilov2017decoupled} with $\eta_{\max} = 6\times10^{-4}$ and the cosine schedule of \cite{loshchilov2017sgdr}, under the objective \cref{eq:total-loss} with the boundedness condition \cref{eq:reg-condition} of \cref{subsec:grade-reg}. Distributed training uses sharded data parallelism with optimiser-state partitioning \cite{rajbhandari2020zero} and pipeline parallelism \cite{huang2019gpipe}.

\subsubsection{Context length}
\label{subsubsec:context-length}

Training sequence lengths must be set by the evaluation the configuration is intended to support. Discourse-level grading at $l = 3$ is the setting in which MLGE is expected to matter most, and it cannot be assessed at a context of 1024 or 2048 tokens; nor does the graded sinusoidal encoding \cref{eq:graded-pe-full} extrapolate beyond training length. Long-context evaluation therefore requires training at the target context, with the compute that implies, and a rotary or otherwise extrapolable graded positional encoding. Absent that, the long-context benchmarks below are out of scope for these configurations and the corresponding predictions are untestable.

\subsection{Planned Benchmarks}
\label{subsec:benchmarks}

\subsubsection{Intrinsic metrics}
\label{subsubsec:intrinsic}

Held-out perplexity \cref{eq:perplexity} on Penn Treebank and WikiText-103 gives short-range coverage; PG-19 gives long-document coverage at contexts the configurations support. Perplexity is the \emph{unweighted} loss, and the graded model does not optimise it. The configurations of \cref{subsubsec:grade-init} enforce the matched-tilt condition of \cref{rem:tilt-cancellation}, under which the graded and unweighted objectives share a population minimizer and the perplexity comparison is therefore made on equal terms. The parse-depth arm of the component ablation of \cref{subsubsec:ablation} is the single arm in which that matching is unavailable, by \cref{prop:properness}(3), and there any perplexity gap is attributed to the grade-tilt, the Kullback--Leibler divergence between the corpus conditional and its context-dependent tilt, and not to the inductive bias. Every perplexity comparison states which condition holds.

\subsubsection{Sample efficiency}
\label{subsubsec:sample-efficiency}

The quantity \cref{prop:ml-expressivity} speaks to is tokens, not steps, and the program measures it directly. For each task of \cref{subsubsec:hierarchical-tasks} we record the held-out error of every arm as a function of pretraining tokens consumed, and report the ratio of token counts at which each first attains a fixed target error. \Cref{prop:ml-expressivity}(ii) predicts that ratio to be $\Lambda(g)$, estimated in advance by \cref{subsec:cs-profiles} and bounded below by $C^{-2}$ at the clipped configurations (\cref{prop:clipped-selection}); the scope of \cref{rem:sample-complexity-scope} means the prediction is of the ordering and the sign, not of the constant. This is the measurement on which the framework's central claim stands, and it is reported separately from steps-to-loss for the reason given in \cref{rem:optimisation-cost}.

\subsubsection{Hierarchical tasks}
\label{subsubsec:hierarchical-tasks}

CoLA (linguistic acceptability) and MultiRC (multi-sentence reading comprehension) are the primary targets, rewarding syntactic sensitivity and compositional inference respectively. BoolQ, CB, and WiC provide contrast at lower structural demand, distinguishing gains specific to hierarchy from gains distributed across tasks.

\subsubsection{Long-context tasks}
\label{subsubsec:long-context}

Subject to the context-length condition of \cref{subsubsec:context-length}, SCROLLS \cite{shaham2022scrolls}, NarrativeQA, and QUALITY test discourse-level grading.

\subsubsection{Ablation protocol}
\label{subsubsec:ablation}

Two ablations run orthogonally. The first isolates each graded component: EG-MHSA alone; MLGE alone; graded loss alone; and their combinations. It also carries the parse-depth arm, the intentionally uncalibrated grading of \cref{subsubsec:grade-init}, retained to measure the tilt cost that \cref{prop:properness}(3) predicts. The second is the selection ladder of \cref{tab:ladder}, varying only the provenance of the grades across the six arms of \cref{subsubsec:baselines}: uniform, random, taxonomic level-constant, POS-based, the educated $q^\star_C$ of \cref{prop:clipped-selection} computed from the estimated profiles, and oracle. Each arm reports its admissibility inner product $\langle q^{\mathrm{init}}, \hat\alpha - \hat\tau \rangle$ from \cref{subsec:cs-certification} alongside its result, so that the ladder tests \cref{eq:cone-condition} directly: the arms are predicted to order by that inner product, an arm with a negative value is predicted to underperform the uniform baseline, and the interpretation of each possible outcome is set out in \cref{subsec:selection-ladder}. For the random arm the inner product is of either sign on any given draw, so its predicted position is fixed only once its draw is measured; the aggregate prediction is the straddle of \cref{subsubsec:baselines}. The level count is swept over $L \in \{1, 2, 4, 8, 16\}$; the choice $L = 4$ reflects a linguistic taxonomy, and the sweep is exploratory rather than confirmatory.

\subsubsection{Interpretability}
\label{subsubsec:interpretability}

Let $\mathcal{A}_{\mathrm{gold}}$ be the set of gold dependency arcs in a held-out treebank and $A_{h,ij}$ the attention weight from token $i$ to token $j$ in head $h$. Define the structural attention mass
\begin{equation}
\label{eq:sam}
  \mathrm{SAM}(h) = \frac{1}{|\mathcal{A}_{\mathrm{gold}}|} \sum_{(i,j) \in \mathcal{A}_{\mathrm{gold}}} A_{h,ij},
\end{equation}
the mean attention weight placed on gold arcs. This is a recall-like quantity, not a precision, and it is not normalised against the attention placed off gold arcs; it should be reported alongside the total attention mass per head. Grade-stratified accuracy $\mathrm{Acc}(l^*)$ partitions benchmark instances by maximum dependency nesting depth $l^*$.

\subsection{Predictions}
\label{subsec:predictions}

Every prediction below names the result it follows from.

\subsubsection{Zero inference overhead} (established). By \cref{cor:zero-overhead,rem:mlge-absorption}, the grading transformations are absorbed into the projection and embedding weights by an invertible change of parameters. A trained GLLM therefore compiles to a standard transformer of identical architecture and identical inference cost, and the $0.33\%$ and $0.24\%$ training overheads of \cref{subsubsec:overhead} vanish entirely at deployment. This follows from the construction and is not subject to experimental uncertainty. By \cref{rem:overhead-comparison} it distinguishes grading from every variant of \cref{tab:arch}, each of which carries its inductive bias in the computation and pays for it at every forward pass.

\subsubsection{The educated grading is certifiably admissible} (established given the profile estimates). By \cref{prop:admissible-grades}(iv) and \cref{rem:certified-init}, a candidate grading enters the cone $\mathcal{Q}_+$ to first order exactly when $\langle q^{\mathrm{init}}, \hat\alpha - \hat\tau \rangle > 0$, and we predict a strictly positive value for the POS and educated arms on CoLA and MultiRC. Only the sign must survive the estimation error of \cref{subsec:cs-profiles}; the magnitudes need not. The ungraded baseline scores $0$ on this quantity identically, by the gauge invariance of \cref{prop:admissible-grades}(i), and cannot score otherwise; \cref{prop:admissible-grades}(iii) is the interpretation of that zero, as the boundary position of the uniform grades. The prediction is reported before pretraining and is falsified by a non-positive value.

\subsubsection{Selection-ladder ordering} (conditional on \cref{assum:profile-divergence} and the scope of \cref{rem:sample-complexity-scope}). The arms of \cref{tab:ladder} are predicted to rank, in the token-efficiency measurement of \cref{subsubsec:sample-efficiency}, in the order of their admissibility inner products, computed offline in \cref{sec-8} and published before training. This is the central prediction of the program and the sharpest, because an ordering fixed in advance is far harder to satisfy by accident than any single pairwise comparison. The partition of outcomes and what each would establish is given in \cref{subsec:selection-ladder}.

\subsubsection{Hierarchical task gains} (conditional on \cref{assum:profile-divergence} and the scope of \cref{rem:sample-complexity-scope}). Gains are predicted to concentrate where the estimated $\mathrm{BC}(\alpha, \tau)$ is smallest, by \cref{assum:profile-divergence} on CoLA and MultiRC, and to diminish on lexically dominated tasks, where the profiles approach agreement and \cref{prop:ml-expressivity}(iv) forces $\Lambda^\star \to 1$. That contrast is what makes the gain attributable to hierarchy rather than to the parameterization, and the measurements of \cref{sec-8} turn it from a qualitative expectation into a quantitative one: the predicted ordering of per-task gains is the ordering of the measured affinities, fixed in advance of training, with each gain bounded by $C^2$ at the clipped configurations (\cref{prop:clipped-selection}). The measurement is that of \cref{subsubsec:sample-efficiency}, in tokens to a fixed error. The gain lives in the norm-controlled regime, which is the regime the runs occupy: \cref{eq:total-loss} is optimised under the decoupled weight decay of \cref{subsec:optimisation}, so the comparison of \cref{prop:ml-expressivity}, with each class evaluated at the smallest budget at which it represents the target, is the one the experiment realises.

\subsubsection{Depth-stratified robustness} (conditional on \cref{assum:profile-divergence} and the scope of \cref{rem:sample-complexity-scope}). Given \cref{prop:ml-expressivity} and profiles that diverge in the sense of \cref{assum:profile-divergence}, the graded model's accuracy is predicted to decay more slowly in nesting depth $l^*$ than the ungraded baseline's, at equal data. The prediction is qualitative: no result in this manuscript supports a rate for $\mathrm{Acc}(l^*)$, and the sample-complexity ratio of \cref{prop:ml-expressivity} governs the data required to reach a target error, not the shape of the accuracy-versus-depth curve. The two are distinct quantities and should not be conflated.

\subsubsection{Head specialisation} (conditional on \cref{eq:total-loss}). The repulsive diversity term of \cref{eq:total-loss} rewards distinct grade tuples, so heads are predicted to acquire distinguishable gradings and, if the linguistic prior is informative, to align differentially with dependency categories, observable as $\mathrm{SAM}(h)$ varying systematically with $q_h$. This is a prediction of the \emph{regulariser}, not of the algebraic structure of EG-MHSA, which by \cref{rem:absorption} is silent on the question. Under an attractive penalty the prediction reverses.

\subsubsection{Resource separation} (established). \Cref{thm:convergence,lem:graded-smoothness} give $\varsigma \leq \lambda^{q_{\max}} \varsigma_{\mathrm{unif}} + c_{\mathrm{reg}}$ with $\lambda^{q_{\max}} \geq 1$: grading degrades the smoothness constant. We therefore predict that graded training does not reduce, and may increase, the optimisation steps to a fixed training loss at matched hyperparameters. The two resources are priced by the same constant and in the framework's favour: the step-count penalty is capped at $C$ by the grade clipping of \cref{subsubsec:configs}, uniformly in $L$, while the token saving is capped at $C^2$ by \cref{prop:clipped-selection}, a price linear in the clip against a purchase quadratic in it (\cref{rem:one-dial}), with the exponential regime of \cref{cor:stratified} reached as the clip opens. The $C$-sweep of \cref{subsubsec:configs} traces this curve directly. Since steps are the cheap resource and high-quality tokens the expensive one, the trade favours the graded model at every setting of the clip. Any reduction in \emph{tokens to a fixed generalisation target} arises from \cref{prop:ml-expressivity,prop:admissible-grades} and is visible as a gap in held-out rather than training loss. These are different measurements and the program reports them separately, in \cref{subsubsec:sample-efficiency} and \cref{subsubsec:intrinsic} respectively.

\subsection{Remaining Theoretical Work}
\label{subsec:settling}

The framework's central quantity is settled in the regime the program occupies. \Cref{thm:stratified-separation} establishes the separation of \cref{con:minimax-separation} for level-stratified geometric profiles at squared loss: a minimax lower bound over the uniform target class, with the estimator unrestricted, met by the matching upper bound of \cref{lem:envelope} throughout the window of \cref{lem:window}, so that on that window the ratio $\Lambda(g)$ of \cref{prop:ml-expressivity} is a genuine separation between the graded prior and its absence rather than a comparison of bounds. What remains is the general case: \cref{con:minimax-separation} for arbitrary profiles and Lipschitz losses, whose proof requires extending the packing argument of \cref{lem:envelope} beyond level-homogeneous designs. It is independent of the experiments above and could be completed before any of them.

Three further directions follow from limitations recorded above. The first limitation of \cref{rem:sample-complexity-scope} calls for the layerwise extension of \cref{prop:ml-expressivity} through the graded stack. \Cref{subsec:graded-forward} records that the level-decomposed feed-forward block performs no cross-level mixing, and that directed transport between levels requires block-triangular maps rather than block-diagonal ones; the morphic structure of \cite{sh-111} supplies them. And \cref{subsubsec:context-length} makes an extrapolable graded positional encoding a precondition for testing discourse-level grading at all, which is the setting in which \cref{cor:stratified} predicts the largest effect at open clip.

\bibliography{sh-109}

\end{document}